\documentclass{article}

\usepackage{microtype}
\usepackage{graphicx}
\usepackage{subfigure}
\usepackage{booktabs} 

\usepackage{hyperref}

\usepackage[accepted]{icml2024}

\usepackage{smile}

\usepackage{bm}
\usepackage{amsmath}
\usepackage{amssymb}
\usepackage{mathtools}
\usepackage{amsthm}
\usepackage{wrapfig}
\usepackage{color}
\usepackage{caption}
\usepackage{diagbox}
\usepackage{tabularx}
\usepackage{natbib}

\usepackage[capitalize,noabbrev]{cleveref}

\theoremstyle{plain}
\newtheorem{proposition}[theorem]{Proposition}
\theoremstyle{definition}
\theoremstyle{remark}

\usepackage[textsize=tiny]{todonotes}

\icmltitlerunning{Bayesian Design Principles for Offline-to-Online RL}

\usepackage{pgfplots}
\usepackage{mathtools}
\usepackage{color,xcolor}

\usepackage{tikz}
\definecolor{myred}{RGB}{216, 27, 96}
\definecolor{myblue}{RGB}{30, 136, 229}
\definecolor{myorange}{RGB}{255, 193, 7}
\definecolor{mygreen}{RGB}{0,177,64}

\usetikzlibrary{decorations.pathreplacing, automata, positioning, arrows}
\pgfplotsset{compat=1.17}

\def\##1\#{\begin{align}#1\end{align}}
\def\$#1\${\begin{align*}#1\end{align*}}

\def\##1\#{\begin{align}#1\end{align}}
\def\$#1\${\begin{align*}#1\end{align*}}

\def\oper{\mathop{\text{op}}}

\newcommand{\la}{\langle}
\newcommand{\ra}{\rangle}

\usepackage[textsize=tiny]{todonotes}
\usepackage{bbm}

\tikzset{
    underbrace style/.style={
        decorate,
        decoration={
            brace,
            mirror,
            amplitude=0.4em
        }
    }
}

\begin{document}

\twocolumn[
\icmltitle{Bayesian Design Principles for Offline-to-Online Reinforcement Learning }



\icmlsetsymbol{equal}{*}

\begin{icmlauthorlist}
\icmlauthor{Hao Hu}{equal,tsinghua1}
\icmlauthor{Yiqin Yang}{equal,tsinghua3}
\icmlauthor{Jianing Ye}{tsinghua1}
\icmlauthor{Chengjie Wu}{tsinghua1}
\icmlauthor{Ziqing Mai}{tsinghua1}
\icmlauthor{Yujing Hu}{netease}
\icmlauthor{Tangjie Lv}{netease}
\icmlauthor{Changjie Fan}{netease}
\icmlauthor{Qianchuan Zhao}{tsinghua2}
\icmlauthor{Chongjie Zhang}{washu}
\end{icmlauthorlist}

\icmlaffiliation{tsinghua1}{Institute for Interdisciplinary Sciences, Tsinghua University, China}
\icmlaffiliation{tsinghua2}{Department of Automation, Tsinghua University, China}
\icmlaffiliation{tsinghua3}{Institute of Automation, Chinese Academy of Sciences, China}
\icmlaffiliation{washu}{Washington University in St. Louis, USA}
\icmlaffiliation{netease}{Fuxi AI Lab, Netease, Inc., Hangzhou, China}
\icmlcorrespondingauthor{Hao Hu}{huh22@mails.tsinghua.edu.cn}

\icmlkeywords{Machine Learning, ICML}

\vskip 0.3in
]



\printAffiliationsAndNotice{\icmlEqualContribution} 

\begin{abstract}
    Offline reinforcement learning (RL) is crucial for real-world applications where exploration can be costly or unsafe. However, offline learned policies are often suboptimal, and further online fine-tuning is required. In this paper, we tackle the fundamental dilemma of offline-to-online fine-tuning: if the agent remains pessimistic, it may fail to learn a better policy, while if it becomes optimistic directly, performance may suffer from a sudden drop. We show that Bayesian design principles are crucial in solving such a dilemma. Instead of adopting optimistic or pessimistic policies, the agent should act in a way that matches its belief in optimal policies.    
    Such a probability-matching agent can avoid a sudden performance drop while still being guaranteed to find the optimal policy. Based on our theoretical findings, we introduce a novel algorithm that outperforms existing methods on various benchmarks, demonstrating the efficacy of our approach. Overall, the proposed approach provides a new perspective on offline-to-online RL that has the potential to enable more effective learning from offline data.
    Our code is public online at \href{https://github.com/YiqinYang/BOORL}{https://github.com/YiqinYang/BOORL}.
\end{abstract}


\section{Introduction}

Reinforcement learning (RL) has shown impressive success in solving complex decision-making problems such as board games \citep{silver2016mastering} and video games \citep{mnih2013playing}, and has been applied to many real-world problems like plasma control \citep{degrave2022magnetic}, and human  

\begin{figure}[H]
    \includegraphics[width=0.8\columnwidth]{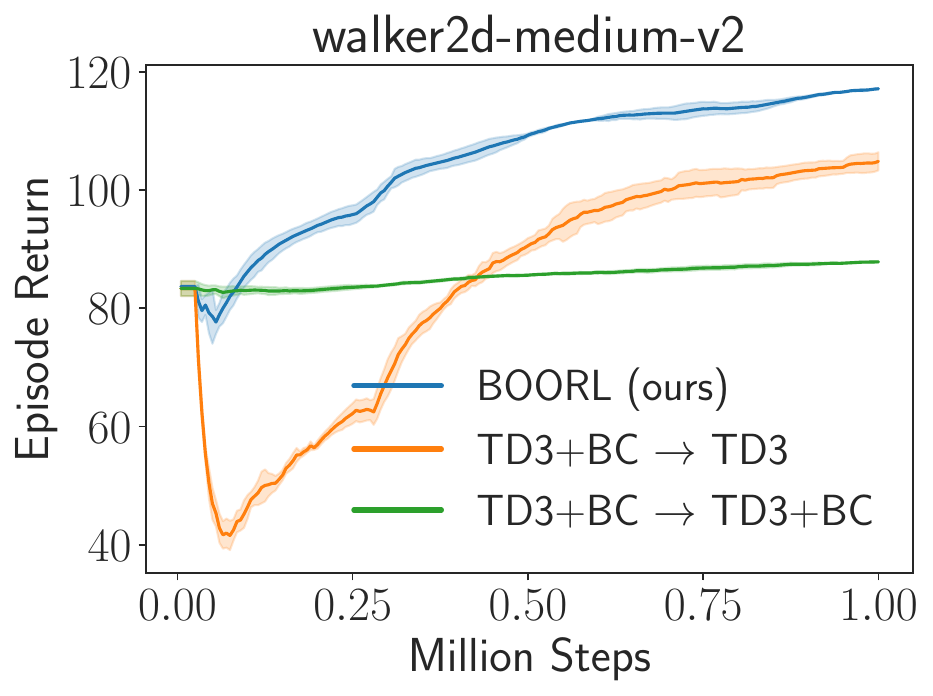}
    \caption{Fine-tuning dilemma in offline-to-online setting.
    if the algorithm remains pessimistic as it does in offline algorithms, the agent learns slowly due to a lack of exploration~(green). Conversely, when the algorithm is optimistic, the agent's performance may suffer from a sudden drop due to inefficient use of offline knowledge and radical exploration~(orange).
    We adopt a probability-matching approach to attain a fast and robust performance improvement~(blue).
    }
    \label{fig: demo}
\end{figure}
preference alignment \citep{ouyang2022training}. However, online RL algorithms rely on a significant amount of exploration, which can be time-consuming and expensive. Offline RL~\citep{levine2020offline} tackles such a problem by utilizing previously collected data and has gained increasing attention 
in recent years, with the potential to leverage large-scale and diverse datasets~\citep{kumar2022offline}. However, offline learned policies can be sub-optimal and generalize poorly due to insufficient data and overoptimization~\cite{gao2023scaling}, necessitating further online fine-tuning.

To address this challenge, a hybrid approach \citep{nair2020awac, lee2022offline, song2022hybrid} has been proposed, enabling sample-efficient learning utilizing both previously collected data and online environments. Pure online or offline strategies are known to fail on the offline-to-online (off-to-on) transition, and various solutions have been proposed, including a balanced replay buffer~\cite{lee2022offline, ball2023efficient}, ensembling~\cite{lee2022offline, ball2023efficient}, and value calibration~\cite{nakamoto2023cal}. However, there is a lack of a theoretical understanding of the optimistic-pessimistic dilemma, as depicted in Figure~\ref{fig: demo}, and previous methods do not fully address such a dilemma in a principled way. This naturally leads to the question:


\begin{center}
    \textit{
    What are the design principles to balance offline data reuse and online exploration in offline-to-online reinforcement learning?
}
\end{center}

To answer this question, we integrate information-theoretic concepts into the design and analysis of RL algorithms. Our results show that Bayesian principles are crucial in solving the dilemma and lead to methods superior to pure online and offline methods. Intuitively, by taking a probability matching approach, which samples from the posterior rather than the most optimistic or the most pessimistic policy, it balances between reusing known experiences and exploring the unknowns. We derive a concrete bound in linear MDPs and conduct experiments in didactic bandits to further demonstrate the superiority of such a Bayesian approach in off-to-on settings.

Based on the theoretical results, we design a simple yet efficient offline-to-online algorithm using approximated posterior sampling with bootstrapped datasets~\citep{osband2016deep}. Experiments show that our algorithm effectively resolves the dilemma, which effectively explores while avoiding a sudden drop in performance. Also, our algorithm is generally compatible with off-the-shelf offline RL methods for offline-to-online transition.




Our contribution is threefold: (1) we provide an information-theoretic characterization of RL algorithms' performance that links online and offline performance with the agent's gained information about the environment, (2) we demonstrate the superiority of the Bayesian approach in offline-to-online RL theoretically, and (3) we develop a practical approach with bootstrapping for offline-to-online RL and achieve superior performance on various tasks.
Overall, our proposed approach provides a new perspective on offline-to-online fine-tuning that has the potential to enable more effective learning from offline data.


\subsection{Related Works}
\paragraph{Offline RL.}
 Offline RL methods mainly address the extrapolation error issue in value and policy estimation, and can be roughly divided into policy constraint~\cite{yang2021believe,ma2021offline, yang2023flow}, pessimistic value estimation~\cite{kumar2020conservative}, and model-based methods~\cite{yu2021combo, hu2022role}.
Policy constraint methods keep the policy close to the behavior under a probabilistic distance~\cite{yang2021believe,ma2021offline, yang2023flow} while pessimistic value estimation methods enforce a regularization constraint on the critic loss to penalize over-generalization~\cite{kumar2020conservative}.
Model-based methods attempt to learn a model from offline data with minimal modification to the policy learning~\cite{yu2020mopo}.

\paragraph{Offline-to-Online RL.}
On the empirical side, \citet{nair2020awac} is among the first to propose a direct solution to offline-to-online RL. \citet{lee2022offline} proposes a balanced replay buffer and pessimistic Q-ensembles. \citet{zhang2023policy} proposes a policy expansion scheme to ensure a smooth offline-to-online transition. \citet{ball2023efficient} leverages ensemble methods and a high update-to-data ratio to enable faster online learning with offline data.
\citet{nakamoto2023cal}~observes an optimistic-pessimistic dilemma similar to ours and proposes calibrating the offline and online learned value function. However, they do not formally point out such a dilemma nor analyze it in a principled way. \citet{wagenmaker2023leveraging} proposes a novel alignment step in actor-critic RL to bridge the gap in online and offline value learning, which also help mitigate the performance drop during the finetuning stage.

Theoretically, \citet{xie2021policy} shows the importance of online exploration when the offline dataset only has partial coverage. \citet{song2022hybrid} demonstrates cases where a purely offline dataset can fail while a hybrid approach succeeds, and \citet{xie2022role} shows an interesting connection between offline concentration coefficient and online learning efficiency.
 \citet{yu2023actor} analyzes online learning with offline data in linear MDPs. Their algorithm allows for the incorporation of offline data without degraded performance compared to pure online learning.

\paragraph{Bayesian RL and Information-Theoretic Analysis.}
\citet{osband2017posterior,russo2014learning} theoretically justify the effectiveness of Bayesian methods like Thompson sampling. 
\citet{russo2016information,lu2019information} discuss the performance of TS in linear bandits. \citet{lu2019information,ouyang2017learning} provide the performance guarantee of TS on tabular MDPs. \citet{dann2021provably,zhong2022gec} analyze the frequentist regret of optimistic TS in general function approximation. 
\citet{uehara2021pessimistic} analyzes the performance of Bayesian methods in the offline setting. \citet{lu2019information} derives an information-theoretical formulation to analyze the regret bound of online learning algorithms like UCB and TS. \citet{xu2023bayesian} proposes an adaptive approach to backup Bayesian algorithms with a frequentist regret. Our work extends these works to offline and off-to-on settings. 
On the empirical side, \citep{osband2016deep} first adopts a Bayesian view on the exploration in deep RL. \citep{chua2018deep} proposes a model-based approach for Bayesian exploration. \citet{ghosh2022offline} adopts the Bayesian principle in the offline setting. Our work extends these works to the off-to-on setting.


\section{Preliminaries}
\subsection{Episodic Reinforcement Learning}
We consider finite-horizon episodic Markov Decision Processes (MDPs), defined by the tuple $(\cS,\cA, H,\cP,r),$ where $\cS$ is a state space, $\cA$ is an action space, $H$ is the horizon and $\cP=\{\cP_h\}_{h=1}^H, r=\{r_h\}^H_{h=1}$ are the transition function and reward function, respectively.\par

A policy $\pi =\{\pi_h\}_{h=1}^H$ specifies a decision-making strategy in which the agent chooses its actions based on the current state, i.e., $a_h \sim \pi_h(\cdot \given s_h)$. The value function $V^\pi_{h}: \cS \rightarrow \RR$ is defined as the sum of future rewards starting at state $s$ and step $h\in [H]$, and similarly, the Q-value function, i.e.
\begin{align}
V^\pi_{h}(s) &= \EE_{\pi}\Big[ \sum_{t=h}^{H} r_t(s_t, a_t)\Biggiven s_h=s  \Big],\notag \\
 Q^\pi_{h}(s,a) &= \EE_\pi\Big[\sum_{t=h}^{H} r_h(s_t, a_t)\Biggiven s_h=s, a_h=a  \Big].
\label{eq:def_value_fct}
\end{align}
where the expectation is w.r.t. the trajectory $\tau$ induced by $\pi$. 
We define the Bellman operator as
\begin{align}
(\BB_h f)(s,a) &= \EE\bigl[r_h(s, a) + f(s')\bigr],
\label{eq:def_bellman_op}
\end{align}
for any $f:\mathcal{S}\rightarrow \mathbb{R}$ and $h\in [H]$.
The optimal Q-function~$Q^*$, optimal value function~$V^*$ and are related by the Bellman optimality equation
\begin{align}
 V^*_{h}(s) & = \max_{a\in \cA}Q^*_{h}(s,a),\notag\\
  Q^*_{h}(s,a) & = (\BB_h V^*_{h+1}) (s,a), 
\label{eq:dp_optimal_values}
\end{align}
while the optimal policy is defined as
\begin{align*}
   \pi^*_{h} (\cdot \given s) & =\argmax_{\pi}\EE_{a\sim \pi}{Q^*_{h}(s, a)}.
\end{align*}

We define the suboptimality, or the per-episode regret as the performance difference of the optimal policy $\pi^*$ and the current policy $\pi_k$ given the initial state $s_1=s$. That is  
\begin{equation}
    \Delta_k= \text{SubOpt}(\pi_k;s) = V^{\pi^*}_{1}(s) - V^{\pi_k}_{1}(s). \notag
    \label{eq:def_regret_2}
\end{equation}

\subsection{Linear Function Approximation}
To derive a concrete bound for Bayesian offline-to-online learning, we consider the \textit{linear MDP}~\citep{jin2020provably,jin2021pessimism} as follows, where the transition kernel and expected reward function are linear with respect to a feature map, which indicate that the value function is also linear.

\begin{definition}[Linear MDP] \label{assumption:linear}
    $\rm{MDP}(\cS, \cA, H, \PP, r)$ is a \emph{linear MDP} with a feature map $\phi: \cS \times \cA \rightarrow \RR^d$, if for any $h\in [H]$, there exist $d$ \emph{unknown} (signed) measures $\mu_h = (\mu_h^{(1)}, \ldots, \mu_h^{(d)})$ over $\cS$ and an \emph{unknown} vector $\theta_h \in \RR^d$, such that 
     for any $(s, a) \in \cS \times \cA$, we have 
    \begin{align}\label{eq:linear_transition}  
    \PP_h(\cdot\given s, a) = \la\phi(s, a), \mu_h(\cdot)\ra, \quad 
    r_h(s, a) = \la\phi(s, a), \theta_h\ra.  
    \end{align}
    Without loss of generality, we assume $\norm{\phi(s, a)} \le 1$ for all $(s,a ) \in \cS \times \cA$, and $\max\{\norm{\mu_h(\cS)}, \norm{\theta_h}\} \le \sqrt{d}$ for all $h \in [H]$.
\end{definition}

\subsection{Information Gain and Bayesian Learning}
Let $\cH_{k,h} = (s_{1,1},a_{1,1},r_{1,1},\ldots, s_{k,h-1},a_{k,h-1},r_{k,h-1},s_{k,h})$ be all the history up to step $h$ of episode $k$. We use subscript $k,h$ to indicate quantities conditioned on $\cH_{k,h}$, i.e. $\PP_{k,h} =  \PP(\cdot|\cH_{k,h}), \EE_{k,h}[\cdot]= \EE[\cdot|\cH_{k,h}]$. The filtered mutual information is defined as 
\$
I_{k,h}(X;Y) = D_{\text{KL}}(\mathbb{P}_{k,h}(X,Y)||\mathbb{P}_{k,h}(X)\mathbb{P}_{k,h}(Y)),
\$
which is a random variable of $\cH_{k,h}$. For a horizon dependent quantity $f_{k,h}$, we define $\EE_{k}{[f_k]}=\sum_{h=1}^H{\EE_{k,h} [f_{k,h}]}$ and similarly for $\PP_k$. We use $t$ instead of $k,h$ for simplicity when it does not lead to confusion, e.g., $I_t \overset{\Delta}{=} I_{k,h}$.

We assume that the MDP can be described by an unknown model parameter $w=\{w_h\}_{h=1}^H$, which governs the outcome distribution. This is true in linear MDP since $Q$-function is  always linear, i.e., $Q_{w_h}(s,a) = \langle w_h, \phi(s,a) \rangle$. The agent's belief over the environment at the $k$-th episode is represented as a distribution $\beta_k$ over $w$.

We also define the information ratio \citep{russo2016information} as the ratio between the expected single step regret and the expected reduction in entropy of the unknown parameter as follows
\begin{definition}[Information Ratio]
    The information ratio $\Gamma_{k,h}$ given history $\cH_{k,h}$ is the minimum value $\Gamma$ such that
    the following event 
    \begin{align}
    \left|Q_{w_h}(s,a) - \EE Q_{w_h}(s,a) \right|\leq \frac{\Gamma}{2} \sqrt{I_{k,h}(w_h;r_{h},s_{h+1}\given s,a)}, \label{eq:info_ratio}
    \end{align}
    holds for all $h\in[H], s\in \cS, a\in\cA$ with probability $1-\delta/2$. 
\end{definition}


We consider the Bayesian regret of
an algorithm $\pi$ over $T$ periods
\$
\text{BayesRegret}(T,\pi) = \EE[\text{Regret}(T,\pi)] = \EE[\sum_{k=1}^{K} \EE_{\beta_k}[\Delta_k]],
\$
where $T=HK$ and the expectation is taken over the randomness in outcomes, algorithm  $\pi$, as well as the posterior
distribution $\beta_{k}$ over $w$. We also use $\text{BayesRegret}(N, T,\pi)$ to denote the offline-to-online regret of an algorithm $\pi$ that uses an offline dataset of size $N$ and interacts online for $T$ steps.

Similar to the definition of the coverage coefficient in offline RL literature~\citep{jin2021pessimism, uehara2021pessimistic}, we can generalize such a concept by taking the expectation over the belief $\beta$. Specifically, we have the following definition from \citet{uehara2021pessimistic}.
\begin{definition}
    The Bayesian coverage coefficient with respect to the feature map $\phi(s,a)$ and posterior $\beta$ is defined as
    \begin{equation}
        C^{\dagger}_{\beta} = \max_{h\in[H]} \EE_{w \sim\beta} \sup_{\|x\|=1}\frac{x^{\top}\Sigma_{\pi^*_w,h} x}{x^{\top}\Sigma_{\rho_h} x},
    \end{equation}
    where 
    $
    \Sigma_{\pi^*_w,h} = \EE_{(s,a)\sim d_{\pi^*_w,h}(s,a)} [\phi(s,a)\phi(s,a)^{\top}], \Sigma_{\rho_h} = \EE_{\rho_h} [\phi(s,a)\phi(s,a)^{\top}].
    $
    
\end{definition}    
Bayesian coverage coefficient is a natural generalization of normal coverage coefficient \citep{uehara2021pessimistic,jin2021pessimism,rashidinejad2021bridging} 
 in Baysian settings.
 

\section{Theoretical Analysis}
\label{Sec: theory}
It is known that we should adopt optimistic algorithms (e.g., UCB~\citep{auer2002using}) in online settings to avoid missing optimal strategies, and we should adopt pessimistic algorithms (e.g., LCB~\citep{rashidinejad2021bridging}) to avoid overconfidence in unknown regions. However, it is unclear what the principled way for offline-to-online settings is where both an offline dataset and an online environment are available. As Figure~\ref{fig: demo} demonstrates, optimistic online algorithms~(e.g., TD3~\citep{fujimoto2018addressing}) can mismanage prior knowledge in the dataset, leading to a sudden drop in performance. On the other hand, pessimistic offline algorithms~(e.g., TD3+BC~\citep{fujimoto2021minimalist}) can be too conservative in exploration, which leads to slow learning.

Such a dilemma is fundamentally intractable under the frequentist (i.e., worst-case) scenario. For instance, consider a multi-arm bandit problem where one arm has limited historical pulls in the dataset. Opting for exploration, we risk significant suboptimality in scenarios where this arm proves less favorable. Conversely, avoiding exploration might incur substantial regret in cases where the arm is optimal. This predicament highlights the limitations of a frequentist analysis for offline-to-online RL. Therefore, we resort to the Bayesian point of view and conduct an information-theoretic analysis in Section~\ref{sec: info} to understand how we can use both the dataset and the online environment properly under the average regret criteria. Interestingly, we show that a probability matching agent (e.g., Thompson Sampling; TS) can make a good balance and outperform pure online and offline agents. Intuitively, uniformly sampling from the posterior rather than acting according to the most optimistic or pessimistic estimation strikes a proper balance between efficient exploration and safe exploitation.

In Section~\ref{sec:linearmdp}, we derive a concrete performance bound for such agents in linear MDPs. Such theoretical performance bound matches well with empirical observations on the didactic bandit problem, as shown in Figure~\ref{fig: theory}.
The probability-matching approach contrasts drastically with pure optimism, pure pessimism, and a naive approach that gradually switches from pessimism to optimism.

\subsection{Information-Theoretic Analysis}
\label{sec: info}
A good exploration strategy should avoid careless trials and incur regret only when it can learn enough from such action. Therefore, the suboptimality should be proportional to the possible information gain from the current policy. Similarly, after learning from the offline dataset, a good exploitation strategy should incur a suboptimality only due to its uncertainty about the environment. The following theorem formalizes such an intuition under the lens of information theory.

\begin{theorem}
    \label{theorem:information}
    
    Then the per-episode regret of Thompson Sampling and UCB agents satisfies
    \#
    \EE_k[\Delta_k]\leq \sum_{h=1}^{H}\Gamma_{k,h} \sqrt{I_{k,h}(w_h;a_{k,h},r_{k,h},s_{k,h+1})}+2\delta H^2, \label{online_bound}
    \#
    where $a_{k,h}\sim \pi_{k,h}$. Similarly, the per-episode regret of Thompson Sampling and LCB agents satisfies
    \#
    \EE_k[\Delta_k]\leq \sum_{h=1}^{H}\Gamma_{k,h} \sqrt{I_{k,h}(w_h;a^*_{h},r_{k,h},s_{k,h+1})}+ 2\delta H^2,
    \label{offline_bound}
    \#
    where $a_{h}^*\sim \pi_{h}^*$. 
\end{theorem}
\begin{proof}
    Please refer to Appendix~\ref{proof:information} for detailed proof.
\end{proof}


Here UCB, LCB and Thompson Sampling strategies are defined in an abstract and information-theoretic manner, with details shown in Appendix~\ref{appendix: alg_abs}.




Theorem~\ref{theorem:information} leads to an information-theoretic performance bound for both online exploration and offline exploitation. Equation~\eqref{online_bound} indicates an online $\tilde{\cO}(\sqrt{T})$-regret bound using the chain rule of mutual information, as depicted in Proposition~\ref{online_guarantee}. With additional assumption on the coverage of the dataset, Equation~\eqref{offline_bound} implies an $\tilde{\cO}(\sqrt{C/N})$ offline performance bound where $C$ is the coverage coefficient. Note that we have no assumptions on the structures of the MDPs (e.g., linear MDP, etc.). Instead, we only assume that the uncertainty in the Q-value function can be reduced at a certain rate as we gain more information about the environment and the total information available does not grow too fast. This assumption generally holds in various settings, including linear MDPs~\citep{jin2020provably}, factored MDPs~\citep{lu2019information}, and kernel MDPs~\citep{yang2020reinforcement}. Please refer to  \citet{lu2019information} for a detailed derivation.

Note that Thompson sampling enjoys both regret bounds in Equation~\eqref{online_bound} and Equation~\eqref{offline_bound}, which indicates that Thompson sampling is suitable for both offline and online settings. Moreover, it indicates that a Bayesian approach enjoys better guarantees in offline-to-online settings since it can avoid sudden performance drop (due to Equation~\eqref{offline_bound}) and explore efficiently (due to Equation~\eqref{online_bound}).
This is summarized in Table~\ref{tab:1}, where we classify existing settings and corresponding doctrines. Table~\ref{tab:1} suggests that the Bayesian approach is consistent across different settings and recommends a realist approach in offline-to-online settings instead of optimism or pessimism.

\begin{table}[H]
    
    \centering
    \begin{tabular}{lcl}\toprule
    Setting           & Doctrine  & Algorithm \\ \midrule
    Online Learning   & Optimism  & TS, UCB    \\ \hline
    Offline Learning  & Pessimism & TS, LCB    \\ \hline
    Offline-to-online & Realism   & TS        \\ \bottomrule
    
    \end{tabular}

    \caption{A taxonomy of the doctrines in different settings of reinforcement learning. a Bayesian approach like TS is generally suitable for online, offline and offline-to-online settings, and is the only one that works in the offline-to-online setting.}
    \label{tab:1}
\end{table}

\begin{figure*}[t]
    \centering
    \subfigure{\includegraphics[scale=0.4]{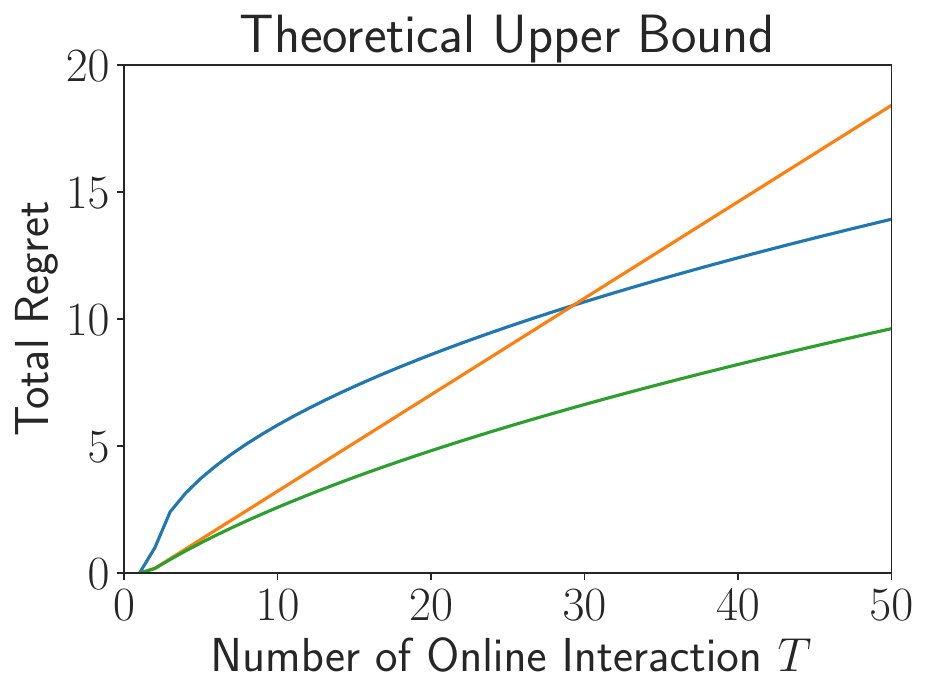}}
    \subfigure{\includegraphics[scale=0.4]{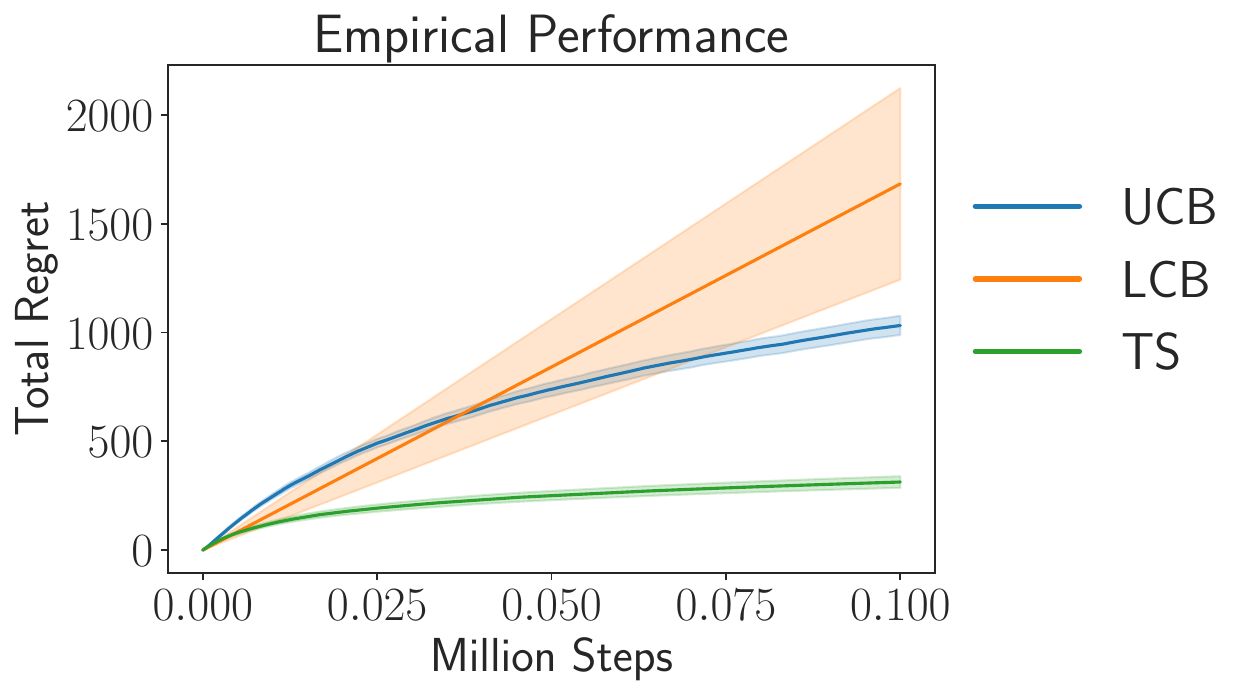}}
    \caption{Theoretical upper bound in Theorem~\ref{theorem:1} and experiment results on Bernoulli bandits. The performance of a Bayesian approach matches the performance of LCB at an early stage by using prior knowledge in the dataset properly and matches the performance of UCB in the run by allowing efficient exploration. Therefore, a realistic Bayesian agent performs better than both optimistic UCB and pessimistic LCB agents.
    }
    \label{fig: theory}
\end{figure*}




\subsection{Specification in Linear MDPs}
\label{sec:linearmdp}

To provide more insights about Theorem~\ref{theorem:information}, we consider a concrete regret bound for Bayesian methods in linear MDPs. Applying Theorem~\ref{theorem:information} to linear MDPs as defined in Definition~\ref{assumption:linear}, we have the following theorem.

\begin{theorem}[Regret of Bayesian Agents in Linear MDPs, informal]
    \label{theorem:1}
    Given an offline dataset $\cD$ of size $N$, 
    the regret of Thompson sampling during online interaction satisfies the following bound:
    \#
        \label{eq:main_result}
        \text{\rm BayesRegret}(N,T,\pi)  \leq c\sqrt{d^3H^3\iota} \left(\sqrt{\frac{N}{C^\dagger_\beta}+ T}-\sqrt{\frac{N}{C^\dagger_\beta}}\right),
    \#
 where $\iota$ is a logarithmic factor and $c$ is an absolute constant.
\end{theorem}
\begin{proof}
    Please refer to Appendix~\ref{proof:theorem1} for detailed proof.
\end{proof}


Theorem~\ref{theorem:1} demonstrates that the Bayesian approach provides a robust regret guarantee. From simple algebraic observations that $\sqrt{a+b}-\sqrt{a} \leq \sqrt{b}$ and $\sqrt{a+b}-\sqrt{a}\leq b/(2\sqrt{a})$, Theorem~\ref{theorem:1} indicates that Bayesian agent can converge to the optimal policy at an $\tilde{\cO}(\sqrt{T})$ rate while having a bounded single-step suboptimality (i.e., $\tilde{\cO}(\sqrt{C^\dagger_\beta/N})$), a feat neither naive online nor offline approaches can accomplish alone. This is further formalized in Propositions~\ref{prop:1} and ~\ref{prop:2}. 

\begin{proposition}
    \label{prop:1}
    Under the same assumption of Theorem~\ref{theorem:1}, the expected first-step suboptimality of UCB can be unbounded (i.e. $\tilde{O}(1)$), while the expected suboptimality of Thompson sampling satisfies
    \$
        \text{\rm SubOpt}(N,T, \pi) \leq c \sqrt{\frac{C^\dagger_\beta d^3 H^3 \iota}{ N}} = \tilde{\cO}\left(\sqrt{\frac{C^\dagger_\beta d^3 H^3}{N}}\right),
    \$
    where $\iota$ is a logarithmic factor and $c$ is an absolute constant.
\end{proposition}
\begin{proof}
    Please refer to Appendix~\ref{proof:prop1} for detailed proof.
\end{proof}
\begin{table*}[t]
    \centering
    \begin{tabular}{c|c|cccc|c}
        \toprule
         Task & Type & ODT & PEX & Cal-QL & BOORL & RLPD\\
         \midrule
        \multirow{5}{*}{Hopper} & random & 10.1$\rightarrow$30.8 & 7.6$\rightarrow$10.1 & 9.3$\rightarrow$11.9 & 8.8$\rightarrow$75.7 & \textbf{84.1}\\
          & medium & 66.9$\rightarrow$97.5 & 63.8$\rightarrow$78.6 & 75.8$\rightarrow$100.6 & 61.9$\rightarrow$\textbf{109.8} & 107.3 \\
          & medium-replay & 86.6$\rightarrow$88.8 & 89.8$\rightarrow$103.3 & 95.4$\rightarrow$106.1 & 75.5$\rightarrow$\textbf{111.1} & 58.9\\
          & medium-expert & 107.6$\rightarrow$\textbf{111.1} & 91.5$\rightarrow$107.8 & 85.0$\rightarrow$\textbf{111.6} & 89.0$\rightarrow$103.4 & 95.2\\
          & expert & 108.1$\rightarrow$\textbf{110.7} & 102.4$\rightarrow$96.6 & 94.8$\rightarrow$\textbf{110.3} & 111.5$\rightarrow$\textbf{109.2} & 100.4\\
          \midrule
        \multirow{5}{*}{Walker2d} & random & 4.6$\rightarrow$8.8 & 2.9$\rightarrow$40.7 & 14.8$\rightarrow$17.3 & 4.8$\rightarrow$\textbf{93.6} & 76.4\\
          & medium & 72.1$\rightarrow$76.7 & 79.8$\rightarrow$94.8 & 80.8$\rightarrow$89.6 & 83.6$\rightarrow$\textbf{107.7} & \textbf{108.6}\\
          & medium-replay & 68.9$\rightarrow$76.8 & 73.6$\rightarrow$89.3 & 83.8$\rightarrow$94.5 & 69.1$\rightarrow$\textbf{114.4} & \textbf{115.0}\\
          & medium-expert & 108.1$\rightarrow$108.7 & 109.6$\rightarrow$\textbf{117.9} & 106.8$\rightarrow$111.0 & 110.8$\rightarrow$\textbf{116.2} & \textbf{115.1}\\
          & expert & 108.2$\rightarrow$107.6 & 108.6$\rightarrow$111.9 & 108.8$\rightarrow$109.2 & 110.0$\rightarrow$118.6 & \textbf{119.8}\\
          \midrule
        \multirow{5}{*}{Halfcheetah} &    random & 1.1$\rightarrow$2.2 & 9.6$\rightarrow$61.2 &  22.0$\rightarrow$45.1 & 10.7$\rightarrow$\textbf{97.7} & 63.0\\
          & medium & 42.7$\rightarrow$42.1 & 47.3$\rightarrow$67.8 & 48.0$\rightarrow$72.3 & 47.9$\rightarrow$\textbf{98.7} & 90.5 \\
          & medium-replay & 39.9$\rightarrow$40.4 & 44.1$\rightarrow$55.2 & 46.5$\rightarrow$59.5 & 44.5$\rightarrow$\textbf{91.5} & 87.6 \\
          & medium-expert & 86.8$\rightarrow$94.1 & 86.7$\rightarrow$91.0 & 48.0$\rightarrow$90.2 & 77.7$\rightarrow$\textbf{97.9} & 94.3 \\
          & expert & 87.3$\rightarrow$94.3 & 90.5$\rightarrow$95.5 & 64.5$\rightarrow$92.1 & 97.5$\rightarrow$\textbf{98.4} & 93.2 \\
        \midrule
        \multirow{6}{*}{Antmaze} & umaze & 56.6$\rightarrow$83.5 & 81.6$\rightarrow$\textbf{100.0} & 78.5$\rightarrow$\textbf{100.0} & 81.7$\rightarrow$\textbf{100.0} & 95.6 \\
        & umaze-diverse & 51.4$\rightarrow$78.9 & 74.1$\rightarrow$78.7 & 75.1$\rightarrow$\textbf{96.0} & 78.3$\rightarrow$\textbf{99.4} & 93.3\\
        & medium-play & 0.0$\rightarrow$0.0 & 68.6$\rightarrow$90.8 & 59.4$\rightarrow$91.9 & 50.6$\rightarrow$\textbf{100.0} & 96.3\\
        & medium-diverse & 0.0$\rightarrow$0.0 & 61.3$\rightarrow$86.4 & 69.5$\rightarrow$84.1 & 61.7$\rightarrow$86.6 & \textbf{94.6}\\
        & large-play & 0.0$\rightarrow$0.0 & 49.9$\rightarrow$68.2 & 24.2$\rightarrow$55.9 & 61.0$\rightarrow$75.8 & \textbf{81.6}\\
        & large-diverse & 0.0$\rightarrow$0.0 & 45.7$\rightarrow$70.0 & 35.2$\rightarrow$69.1 & 50.9$\rightarrow$79.9 & \textbf{86.8}\\
        \midrule
        $\delta_{\rm sum}$~(0.2M) & & 146.0 & 326.8 & 392.0 & \textbf{698.1} & -\\
        \bottomrule
    \end{tabular}
    \caption{Normalized score before and after the online fine-tuning with five random seeds. Each method is pre-trained with 1M steps and then fine-tuned with 0.2M online steps.
    Since offline algorithms' performance differs, we focus on performance improvement within a limited time, $\delta_{\rm sum}$~(0.2M), which denotes the sum of performance improvement on all tasks within 0.2M steps.
    RLPD starts from a random initialization rather than the offline pretrain policy, so we only compare the final performance with RLPD rather than the relative performance improvement.
    }
    \label{tab: d4rl}
\end{table*}

\begin{proposition}
    \label{prop:2}
    Under the same assumption of Theorem~\ref{theorem:1}, the regret of LCB can be unbounded (i.e. $\tilde{O}(T)$), while the regret of Thompson sampling satisfies
    \$
        \text{\rm BayesRegret}(N, T,\pi) \leq  2c\sqrt{d^3H^3T\iota} = \tilde{\cO}(\sqrt{d^3H^3T}),
    \$
    where $\iota$ is a logarithmic factor and $c$ is an absolute constant.
\end{proposition}
\begin{proof}
    Please refer to Appendix~\ref{proof:prop2} for detailed proof.
\end{proof}

The performance bound in Theorem~\ref{theorem:1} incorporates the number of online interactions $T$ and the offline dataset size $N$, demonstrating that both elements play a key role in minimizing the average regret. This contrasts significantly with previous theoretical findings. \citet{xie2021policy} analyzes the benefit of online exploration when offline data only has partial coverage. At the same time, our result shows that we can benefit from online exploration in Bayesian settings, even if offline data has full coverage. \citet{song2022hybrid} proposes Hy-Q, which allows a weaker notion of coverage by combining a certain proportion of online and offline data, while our result shows that Bayesian strategies achieve the best of both worlds regardless of data proportions.


To further verify our theoretical findings, we conducted experiments on didactic Bernoulli bandit, and the result is shown in Figure~\ref{fig: theory}. The performance of the TS agent aligns well with our predictions in 
Equation~\eqref{eq:main_result} outperforms both pure online and offline agents. We also compare TS with a naive approach where we gradually switch from pessimism to optimism by interpolating the weight on the bonus term. This naive agent performs badly regardless of the choice of the interpolation scheme. This result highlights the importance of adopting a realistic probability match strategy rather than committing to optimism or pessimism. Please refer to Appendix~\ref{appendix: multi-arm bandit} for more results and details on the didactic bandits.


\begin{figure*}[t]
    \centering
    \subfigure{\includegraphics[width=0.245\linewidth]{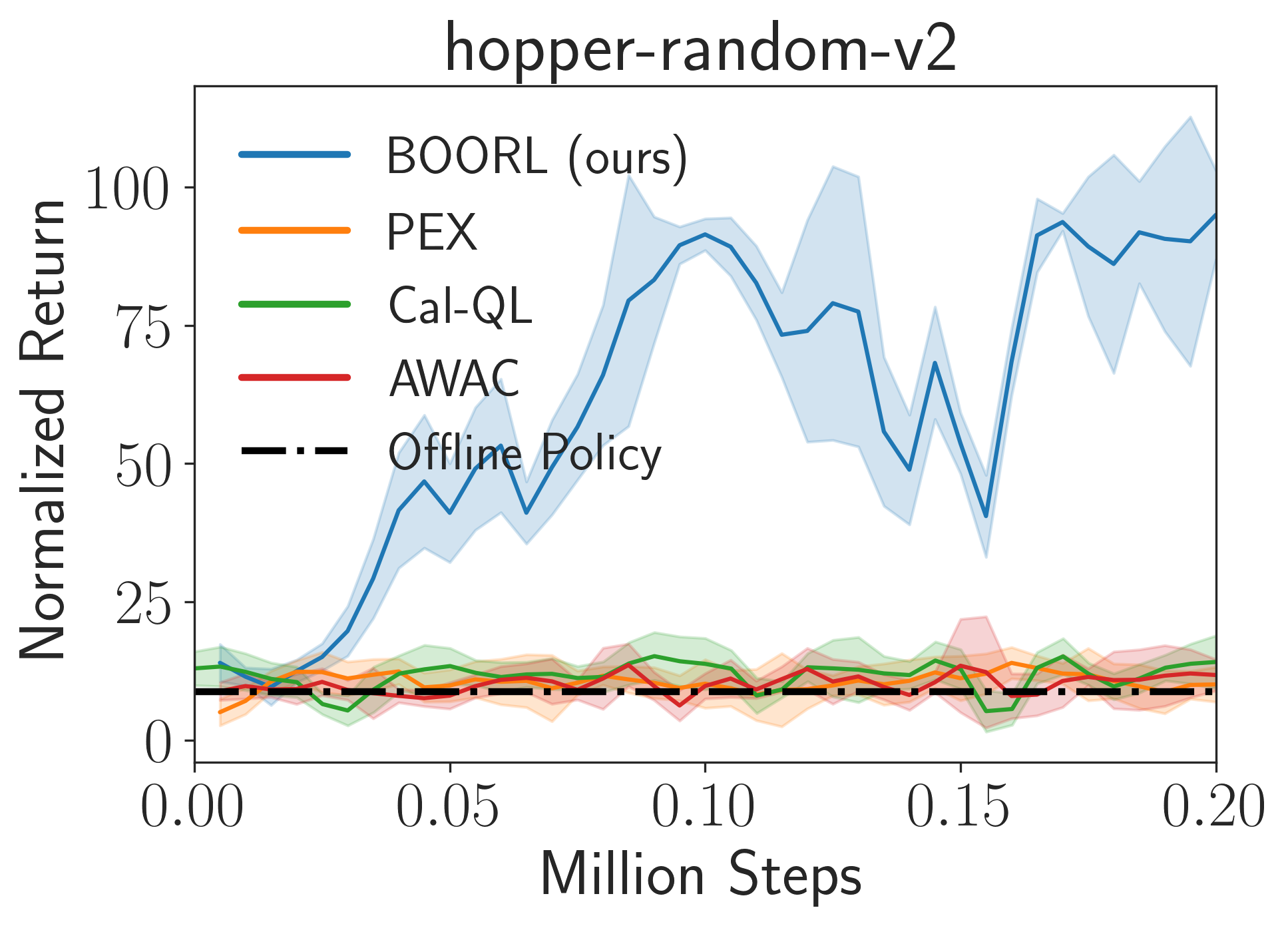}}
    \hfill
    \subfigure{\includegraphics[width=0.245\linewidth]{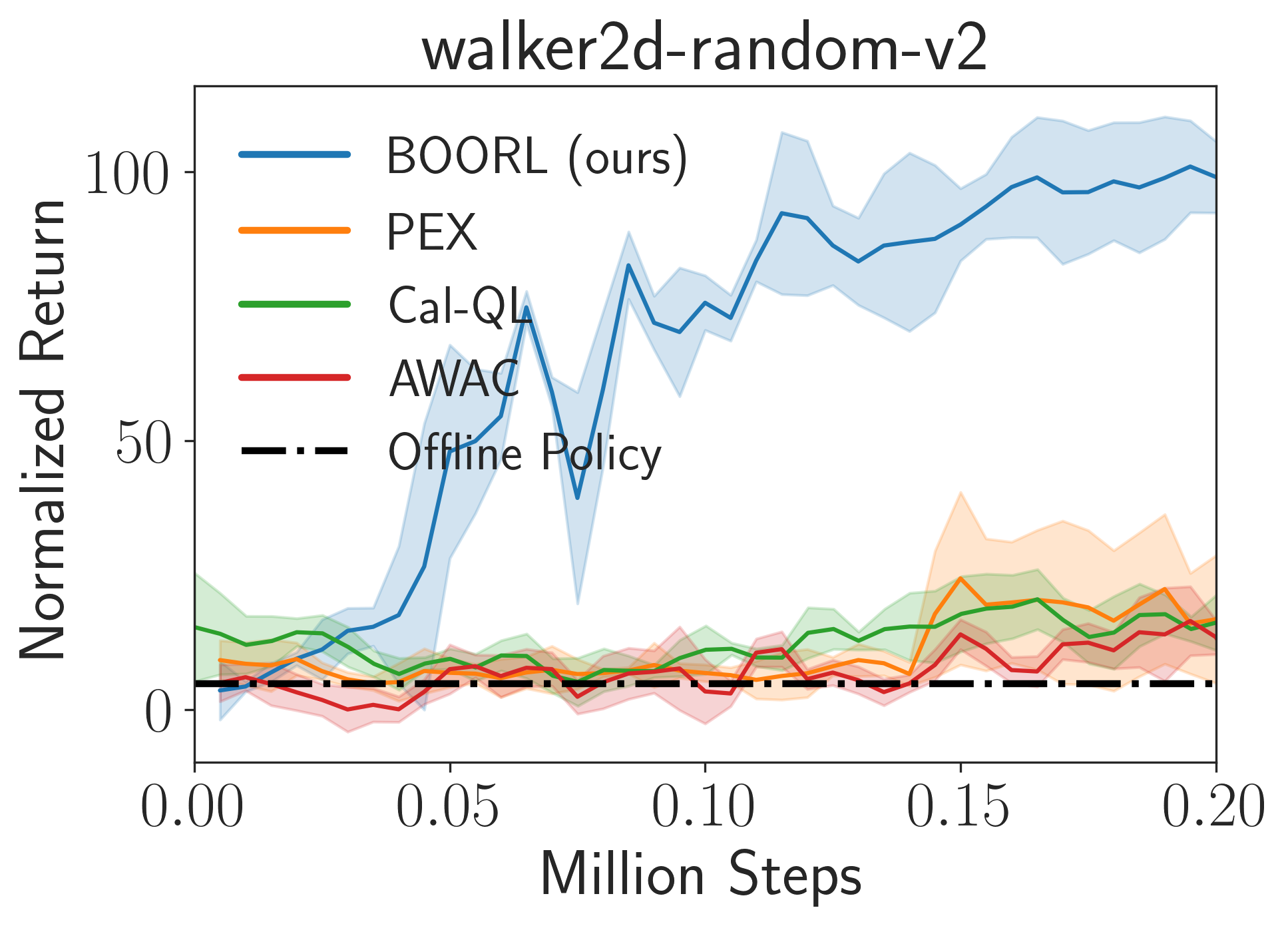}}
    \hfill
    \subfigure{\includegraphics[width=0.245\linewidth]{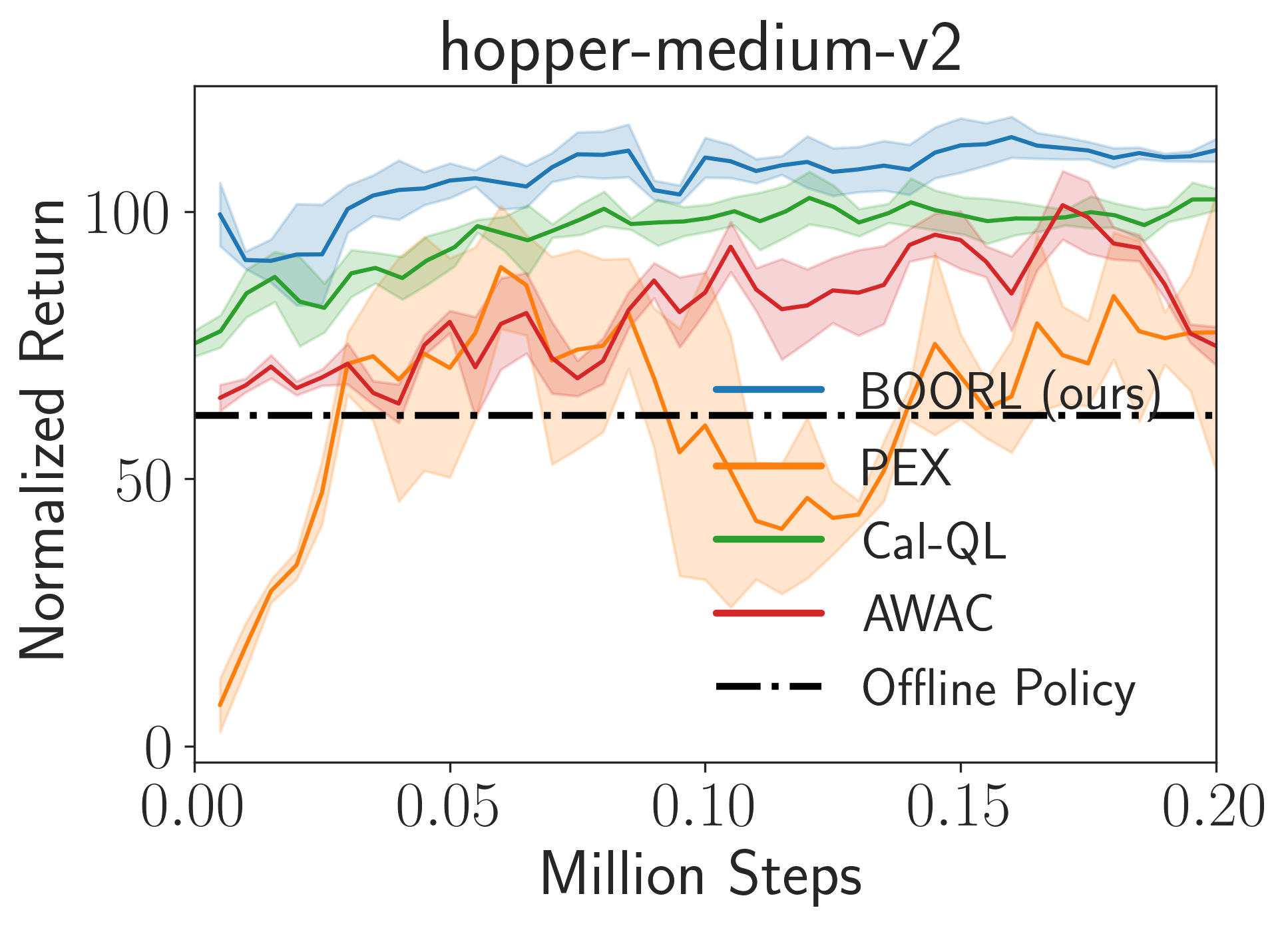}}
    \hfill
    \subfigure{\includegraphics[width=0.245\linewidth]{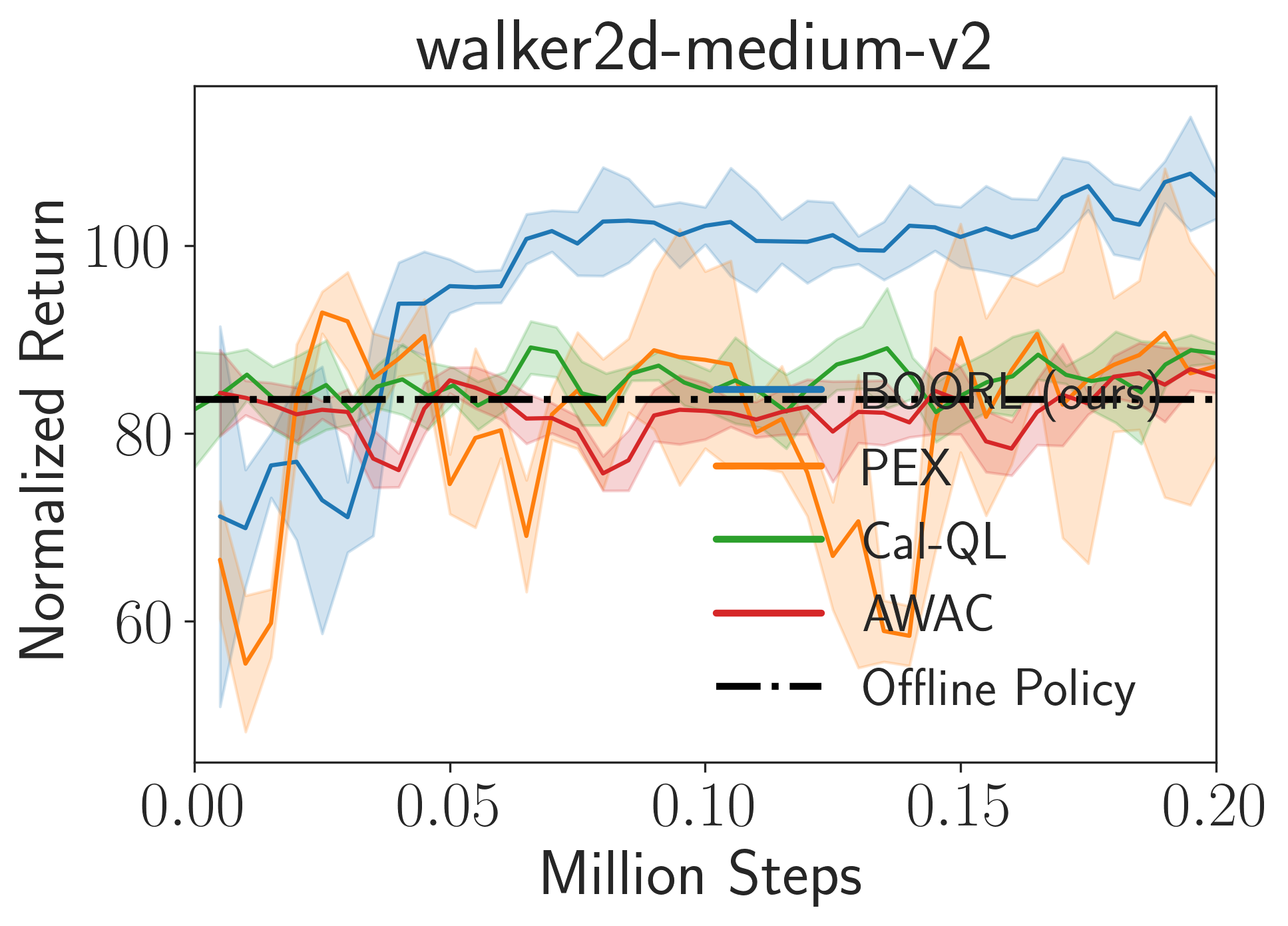}}
    
    \subfigure{\includegraphics[width=0.245\linewidth]{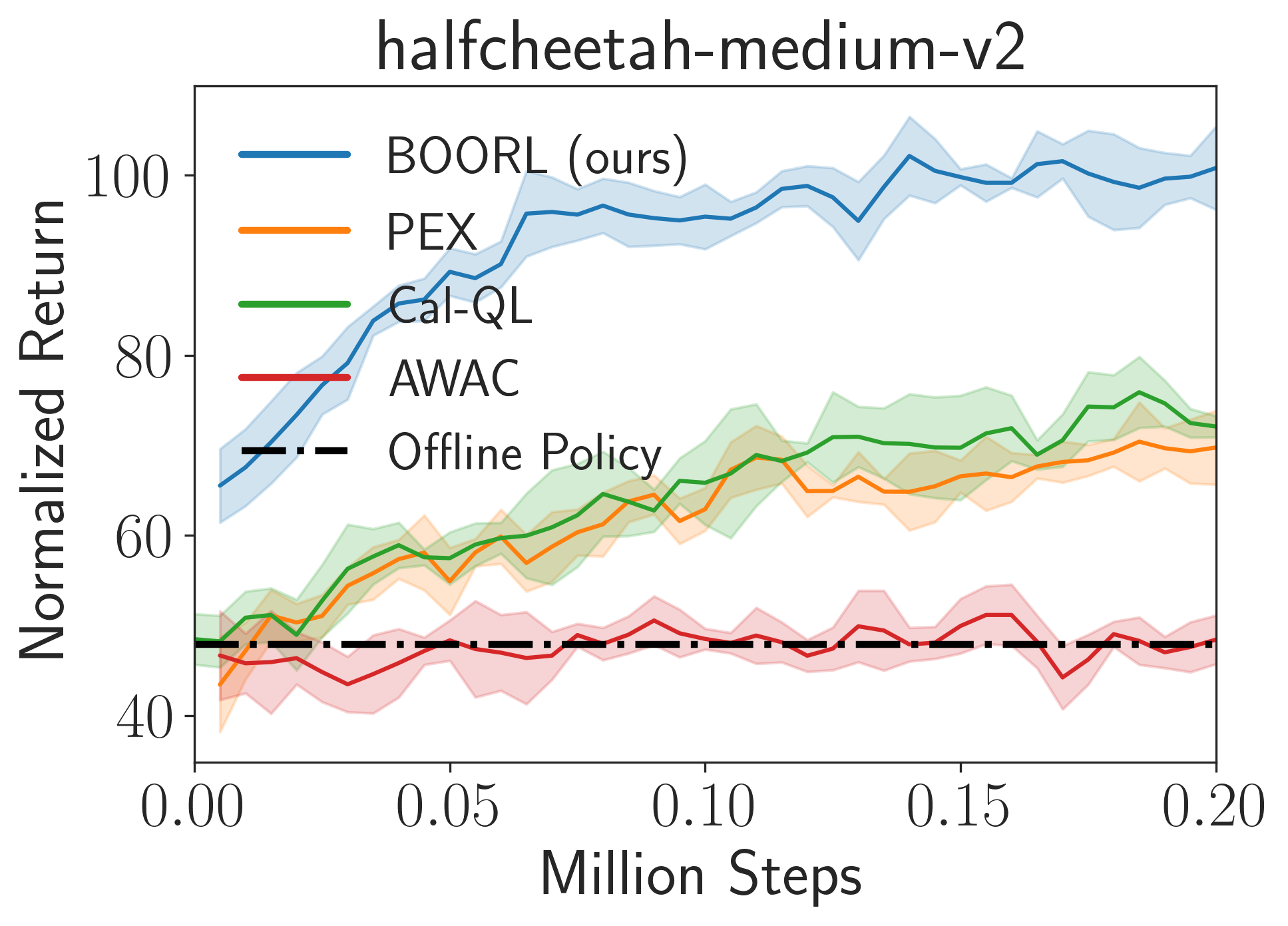}}
    \hfill
    \subfigure{\includegraphics[width=0.245\linewidth]{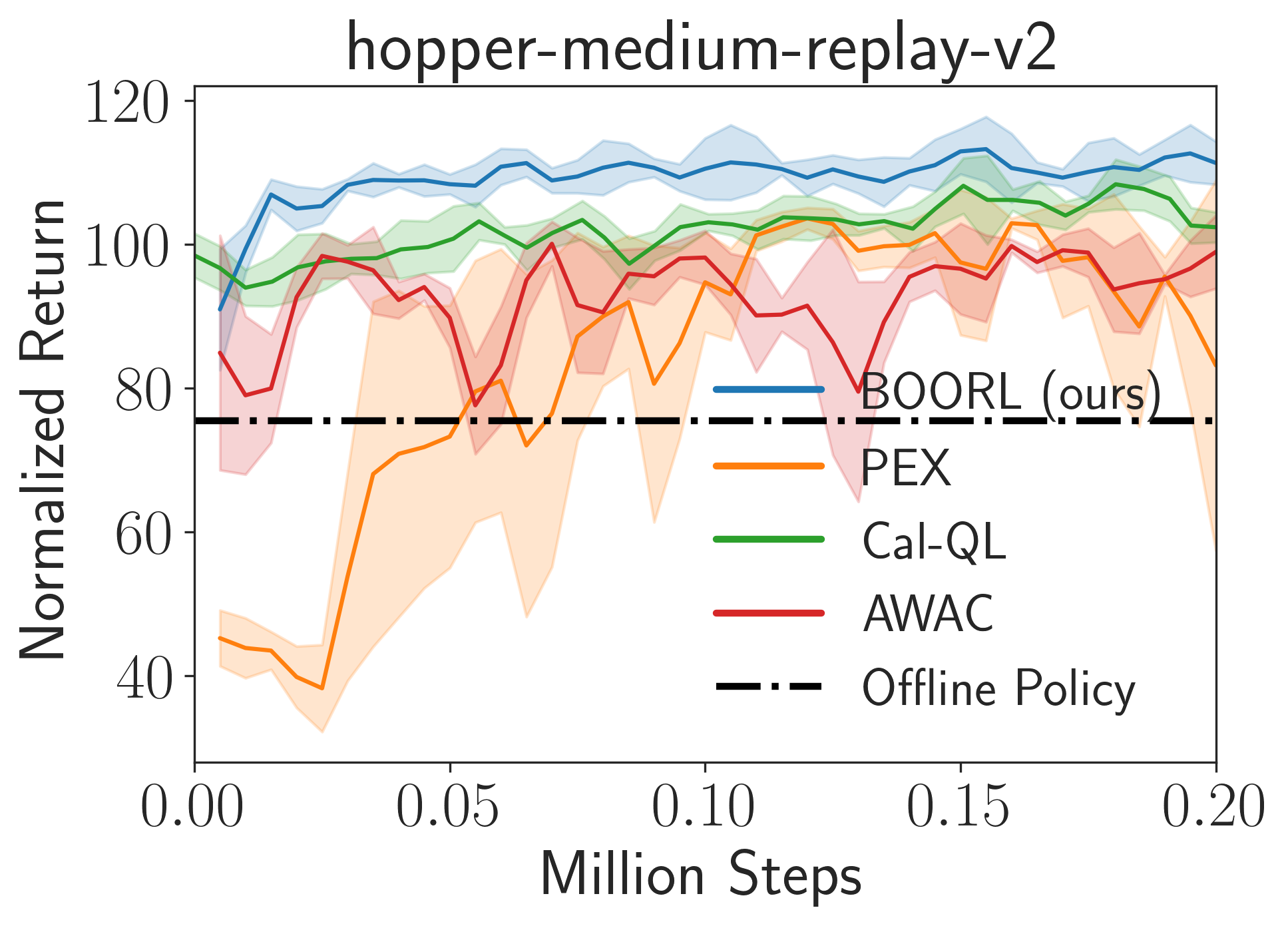}}
    \hfill
    \subfigure{\includegraphics[width=0.245\linewidth]{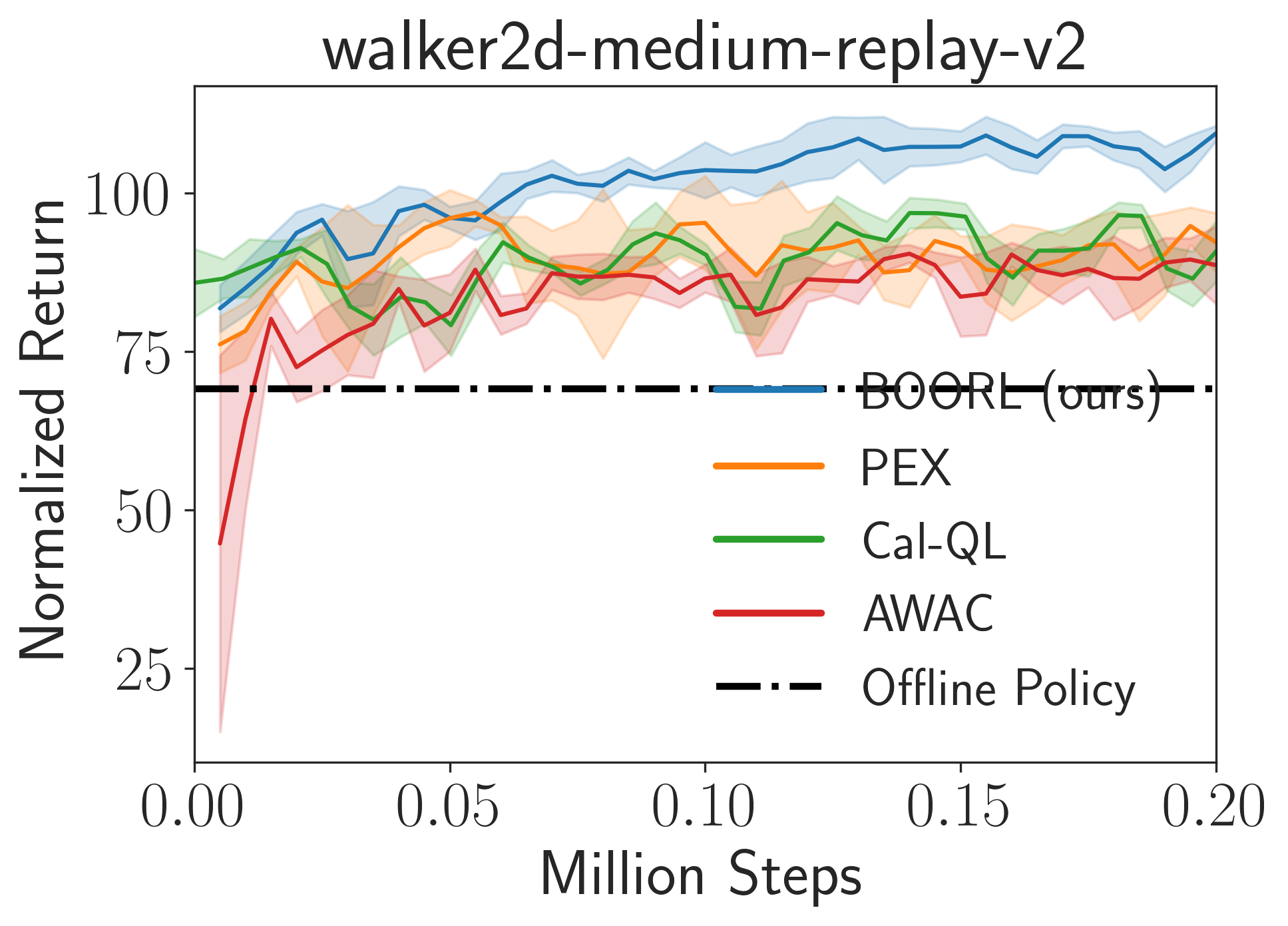}}
    \hfill
    \subfigure{\includegraphics[width=0.245\linewidth]{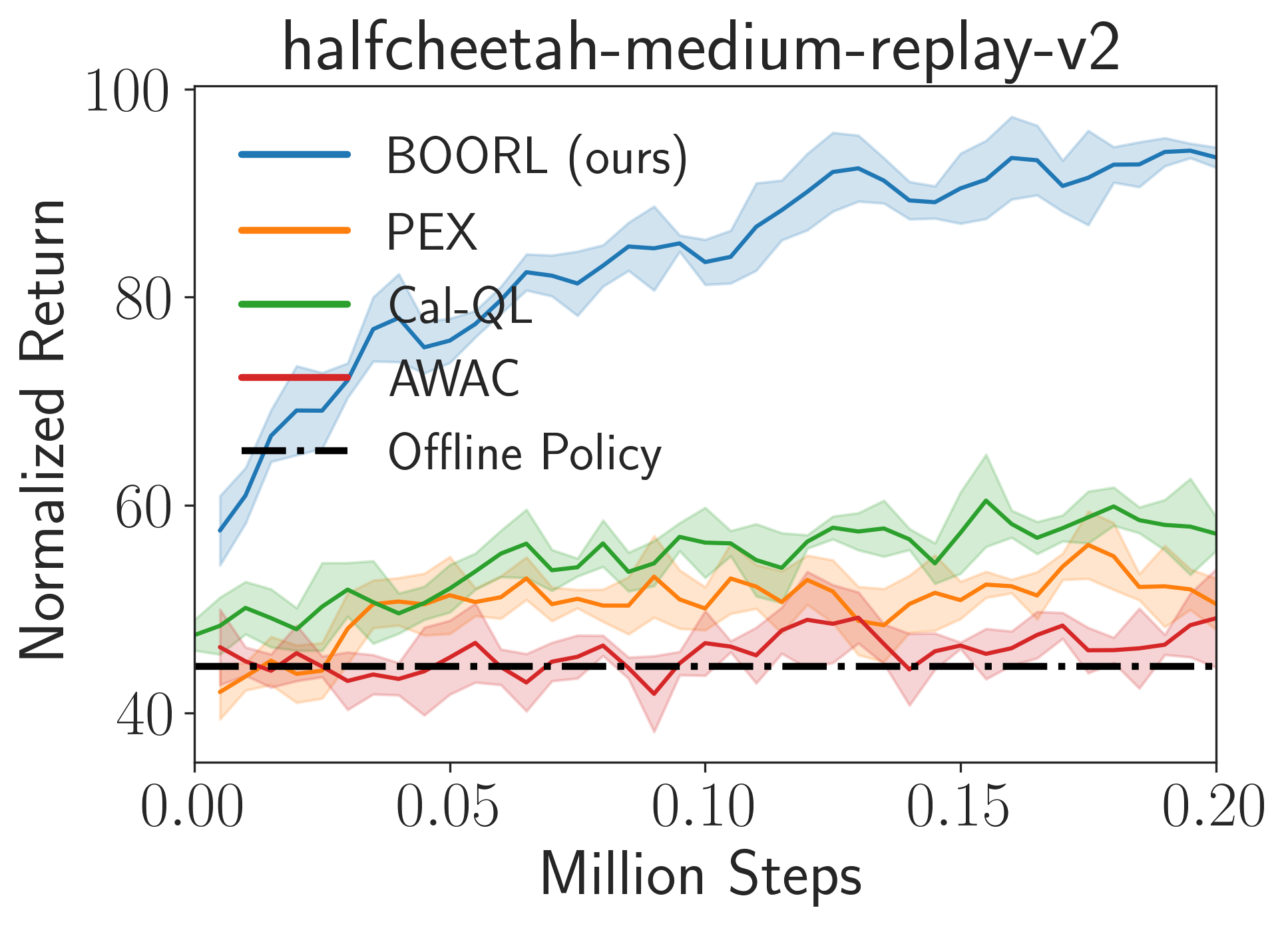}}
    
    \subfigure{\includegraphics[width=0.245\linewidth]{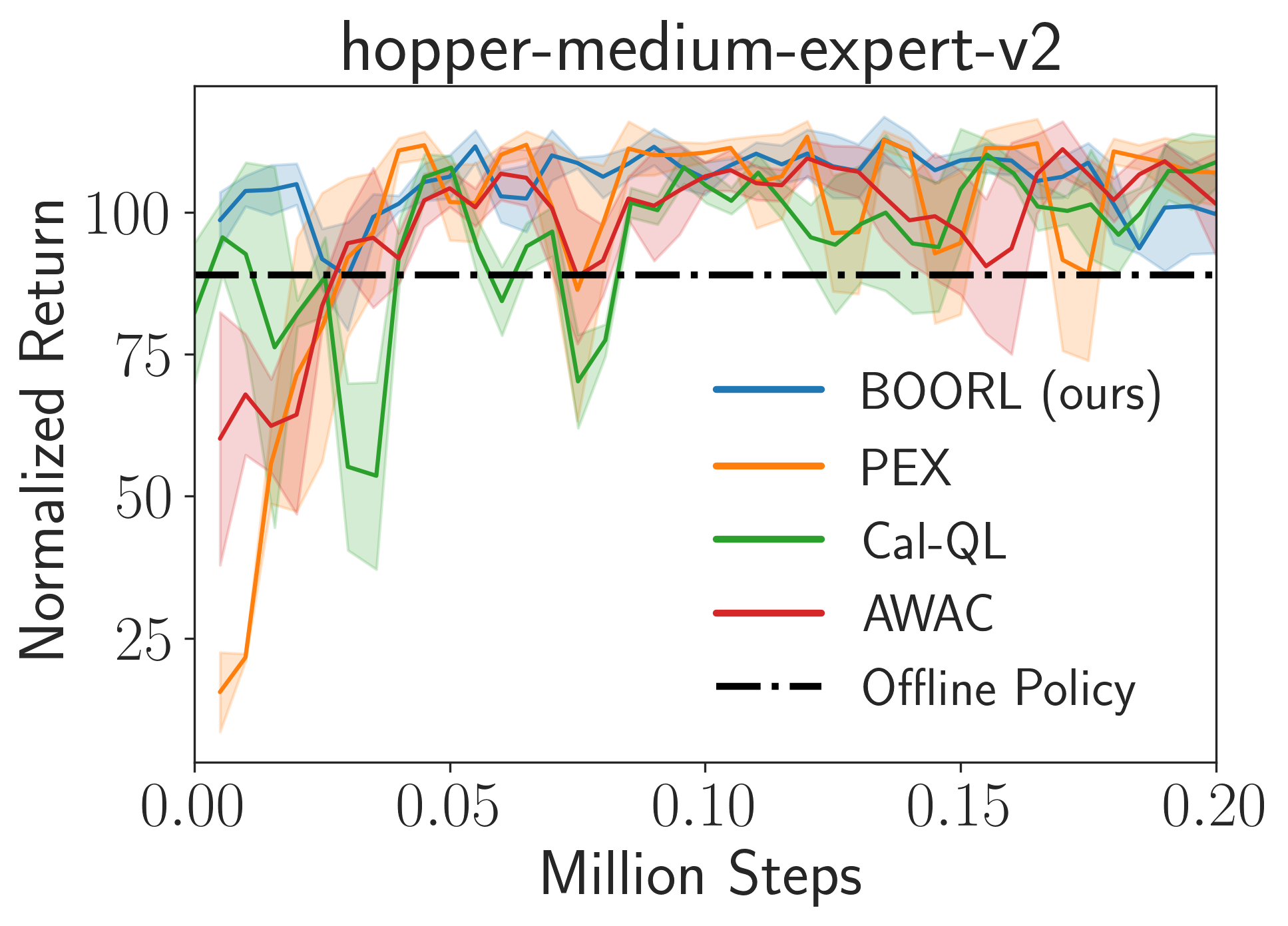}}
    \hfill
    \subfigure{\includegraphics[width=0.245\linewidth]{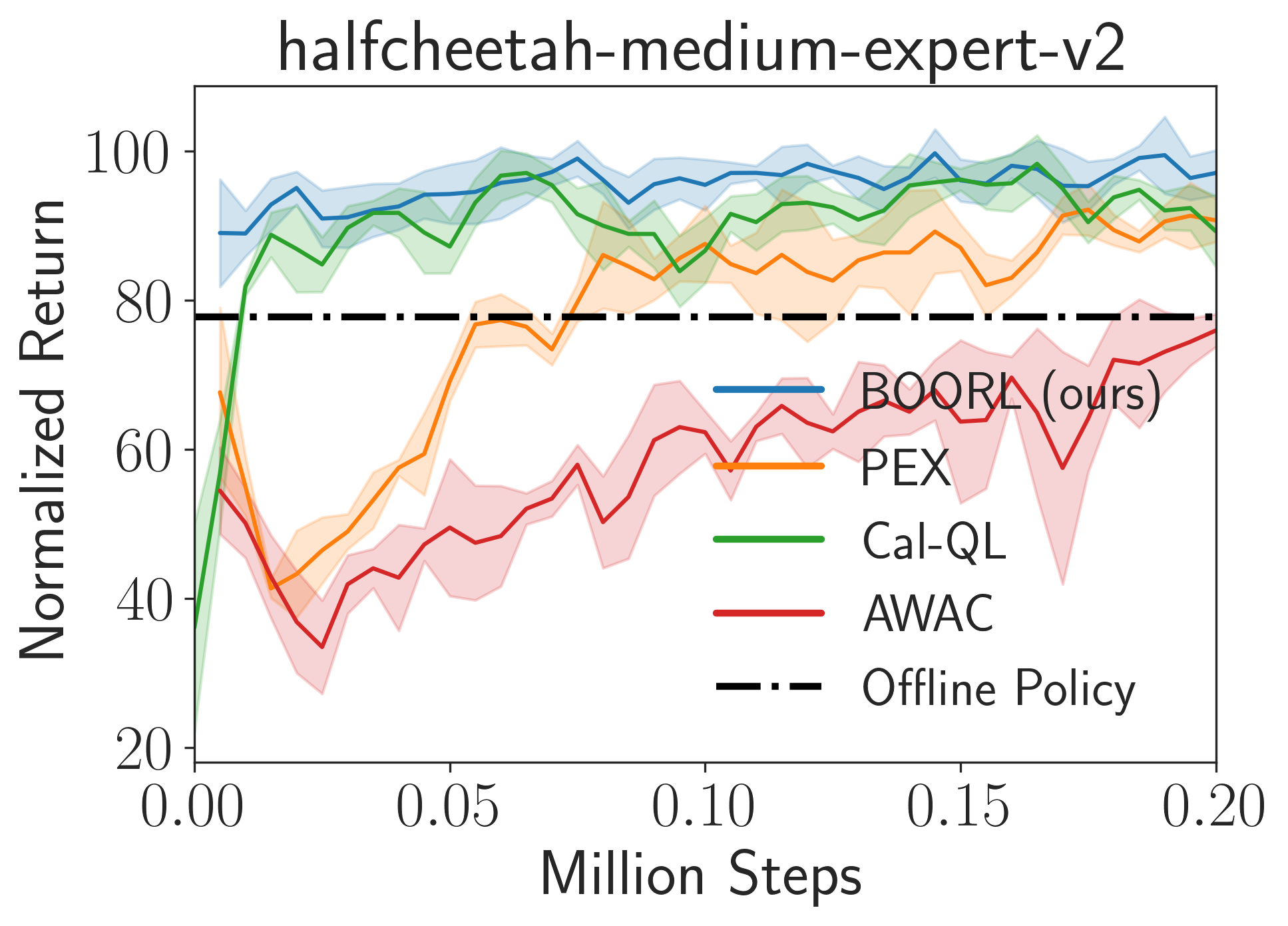}}
    \hfill
    \subfigure{\includegraphics[width=0.245\linewidth]{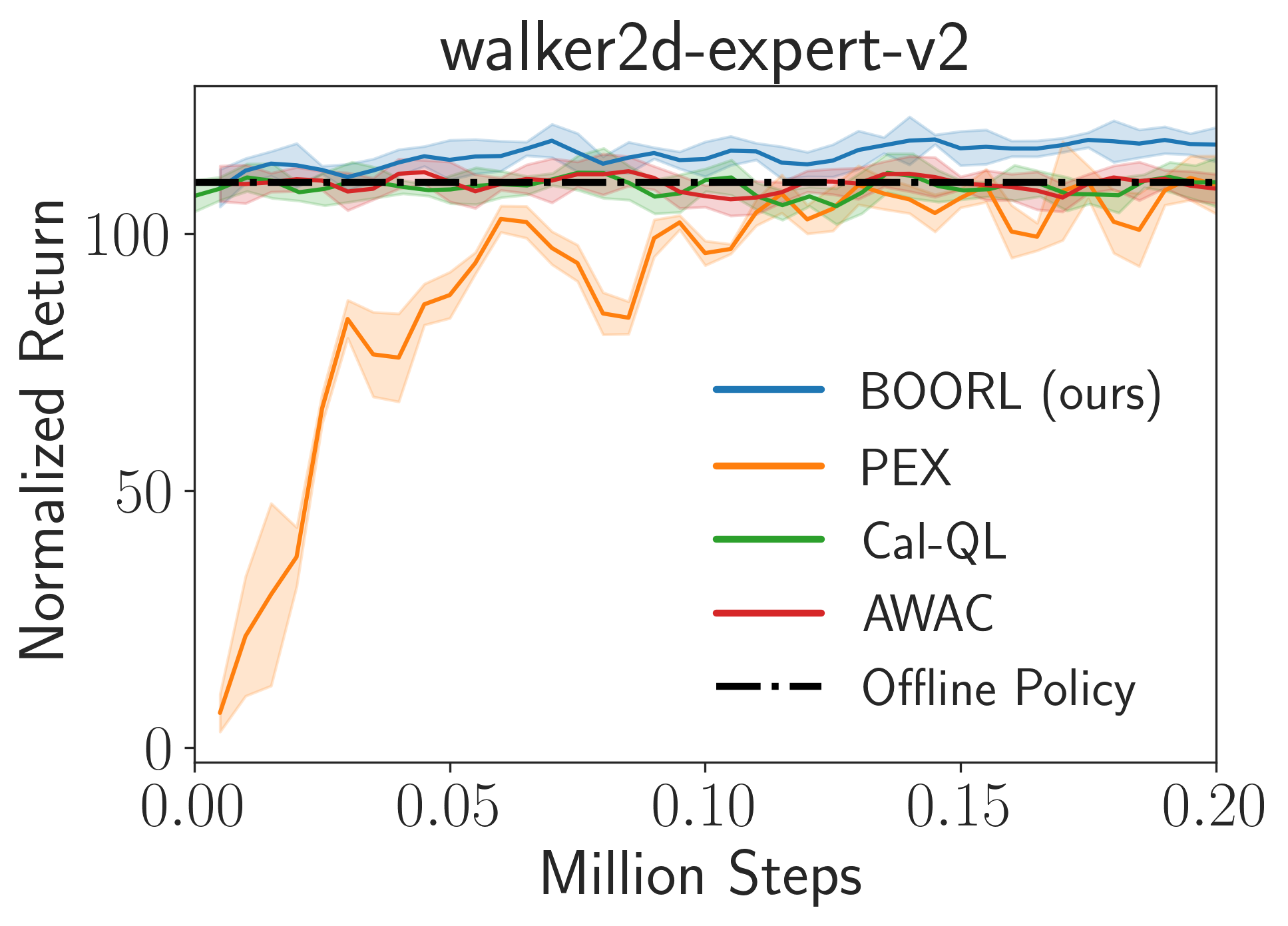}}
    \hfill
    \subfigure{\includegraphics[width=0.245\linewidth]{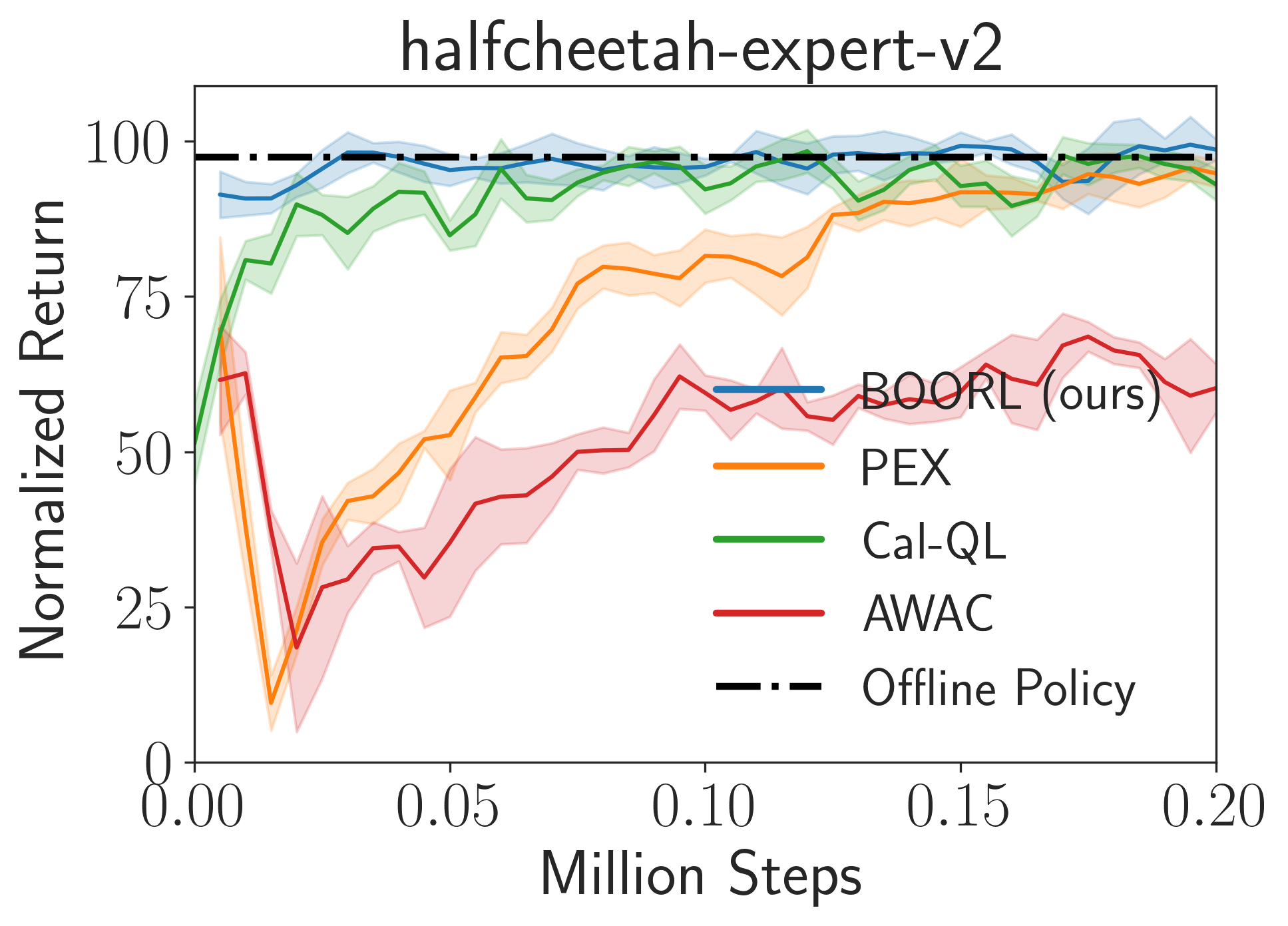}}
    \caption{Experiments between several baselines and BOORL within 0.2M time steps. The reference line is the performance of TD3+BC. The experimental results are averaged with five random seeds.
    Please refer to Appendix~\ref{appendix: pex} for more results.
    }
    \label{fig: new baseline}
\end{figure*}

\begin{algorithm}[h]
\caption{BOORL, Offline Phase}
\label{alg: offline rl}
\begin{algorithmic}[1]
    \STATE {\bf Require}: Ensemble size $N$, offline dataset $\cD^{\rm off}$, masking distribution $M$
    \STATE Initialize parameters of $N$ independent TD3+BC agents $\{Q_{\theta_i}, \pi_{\phi_i}\}_{i=1}^{N}$
    \FOR{$i=1, \cdots, N$}
    \STATE \textcolor{purple}{Sample bootstrap mask $m \sim M$~(e.g., Bernoulli distribution)}
    \STATE \textcolor{purple}{Add $m$ to $\cD^{\rm off}$ as $\cD_{i}^{\rm off}$}
    \FOR{each training iteration}
    \STATE Sample a random minibatch $\{\tau_j\}_{j=1}^{B}\sim \cD_{i}^{\rm off}$
    \STATE Calculate $L_{\rm critic}^{\rm offline}(\theta_i)$ and update $\theta_i$
    \STATE Calculate $L_{\rm actor}^{\rm offline}(\phi_i)$ and update $\phi_i$
    \ENDFOR
    \ENDFOR\\
    \STATE \textbf{Return} $ \{Q_{\theta_i}, \pi_{\phi_{i}}\}_{i=1}^{N}$
\end{algorithmic}
\end{algorithm}

\begin{algorithm}[h]
    \caption{BOORL, Online Phase}
    \label{alg: online rl}
    \begin{algorithmic}[1]
        \STATE {\bf Require}: $ \{Q_{\theta_i}, \pi_{\phi_{i}}\}_{i=1}^{N}$, offline dataset $\cD^{\rm off}$
        \STATE Initialize empty online replay buffer $\cD^{\rm on}$
        \FOR{each iteration}
        \STATE Obtain initial state from environment $s_0$
        \FOR{step $t=1, \cdots, T$}
        \STATE \textcolor{purple}{Construct a categorical distribution based on $$p_i= \frac{\exp(Q_{\theta_i}(s_t,\pi_{\phi_i}(s_t)))}{\sum_j\exp(Q_{\theta_j}(s_t,\pi_{\phi_j}(s_t)))}$$}
        \STATE \textcolor{purple}{Pick an policy to act $a_t \sim \pi_{\phi_n}(\cdot\mid s_t)$ by sampling index $n$ based on $p=(p_1, p_2, ..., p_N)$}
        \STATE Store transition $(s_t, a_t, r_t, s_{t+1})$ in $\cD^{\rm on}$
        \STATE Sample minibatch $b$ from $\cD^{\rm off}$ and $\cD^{\rm on}$
        \FOR{$i=1,\cdots,N$}
        \STATE With $b$, calculate $L_{\rm critic}^{\rm online}(\theta_{i})$ and update $\theta_{i}$
        \ENDFOR
        \FOR{$i=1,\cdots,N$}
        \STATE With $b$, calculate $L_{\rm actor}^{\rm online}(\phi_{i})$ and update $\phi_{i}$
        \ENDFOR
        \ENDFOR
        \ENDFOR
    \end{algorithmic}
\end{algorithm}

\section{Method}
\label{sec: algorithm}


Designing a posterior sampling-based offline-to-online agent faces two main challenges. The first is to ensure pessimism while obtaining a posterior distribution during offline pretraining. Naively applying posterior sampling during the offline phase loses frequentist (i.e., worst-case) guarantees~\citep{anonymous2024posterior} and can perform poorly in practice. The second challenge is to update the posterior smoothly despite the distributional shift between offline and online data. Otherwise, it may lead to unstable performance improvements. 

 To address the above challenges, we propose a simple yet effective Bayesian Offline-to-Online Reinforcement Learning~(BOORL) method to balance offline policy reuse and online exploration.

Our method consists of two phases. We combine the bootstrapping mechanism and pessimistic offline algorithms during the offline phase to obtain a posterior belief over optimal policies while ensuring offline performance. During the online phase, we generate the posterior belief over policies based on softmax $Q$-values, and we act by sampling from such posterior. The observation from the environment is then added on the fly to each bootstrapped dataset from the offline phase to update the posterior. We reweight the newly added online data to guarantee a smooth posterior update. 

\paragraph{Offline Phase.}
In the offline phase, we resample the dataset by generating a set of masks $\{M_\ell\}_{\ell=1}^L$. Each mask $M_\ell \in \{0,1\}^N$ is a binary vector representing if each data point in $\cD^{\rm off}$ is selected in the dataset $\cD_\ell^{\rm off}$. We adopt a standard bootstrapping method, and each $M_\ell$ is sampled independently from the Bernoulli distribution $\text{Ber}(p)$. These masks are stored in the memory replay buffer to represent $\mathcal{D}_{\ell}^{\rm off}$ and we train $L$ policy networks and corresponding $Q$-value networks $\{\pi_{\phi_\ell}, Q_{\theta_\ell}\}_{\ell=1}^{L}$ concerning each resampled dataset.
Next, each one of these offline policies is trained against its own pessimistic $Q$-value network and bootstrapped dataset with the offline RL loss~(e.g., TD3+BC~\citep{fujimoto2021minimalist}):

\begin{align}
    \cL_{\rm critic}(\theta_\ell) = \mathbb{E}_{(s,a,r,s',m_\ell)\sim\mathcal{D}_{\ell}^{\rm off}}\left[(r+\gamma Q_{\theta'_{\ell}}(s', \tilde{a})\right. \notag \\
    \left. -Q_{\theta_{\ell}}(s, a))^2*M_\ell\right], \\
    \cL_{\rm actor}(\phi_\ell) = -\mathbb{E}_{(s,a,m_\ell)\sim\mathcal{D}_{\ell}^{\rm off}}\left[(\lambda Q_{\theta_{\ell}}(s,\pi_{\phi_\ell}(s)) 
    \right.\notag\\
    \left. - (\pi_{\phi_\ell}(s)-a)^2 )*M_\ell\right],
\end{align}
where $\tilde{a}=\pi_{\phi'_\ell}(s')+\epsilon$ is the target policy smoothing regularization and $\lambda$ is the hyper-parameter for behavior cloning. $\theta'_{\ell}$ and $\phi'_{\ell}$ are the parameters for the Q-value and policy target network, respectively. 

\paragraph{Online Phase.}
When acting in the online environment, we first compute posterior beliefs over candidate policies $\{\phi_\ell\}_{\ell=1}^{L}$ according to 
\begin{align}
    p_\ell= \frac{\exp(Q_{\theta_\ell}(s,a_\ell))}{\sum_j\exp(Q_{\theta_j}(s,a_j))},
\end{align}
where $a_\ell\sim\pi_{\phi_\ell}(s)$.  
Then, we sample index $\ell\in\{1, \cdots, L\}$ based on $p=(p_1, p_2, ..., p_L)$ at each time step and following $\pi_{n}$ to collect online data. Each loaded policy and $Q$-value network is continued to be trained with the online RL loss~(e.g., TD3~\citep{fujimoto2018addressing}).
Additionally, to incorporate prior data better, we reweight the online data so that for the training batch, $50\%$ of the sampled data are from the online phase, and the remaining $50\%$ from the offline replay buffer $\cD^{\rm off}$~\citep{ross2012agnostic}. 


Intuitively, the Bayesian-based agent allows a ``robust'' policy improvement since it adopts a mild exploration strategy by sampling from a soft distribution, which reduces ``catastrophic'' attempts from optimistic exploration. The overall algorithm is summarized in Algorithm~\ref{alg: offline rl} and Algorithm~\ref{alg: online rl}.
We highlight elements important to our approach in \textcolor{purple}{Purple}.
Our approach only requires minor adjustments to existing methods, making its effectiveness easily implementable and following.

\section{Experiments}

\begin{table*}[htb]
    \centering
    \begin{tabular}{c|c|ccccc}
        \toprule
         Task & Type & BOORL & Bayesian & $\delta$ & Hybrid RL & $\delta$ \\
         \midrule
        \multirow{3}{*}{Hopper} & random & 75.7$\pm$1.3 & 85.4$\pm$3.3 & -9.7 & 75.2$\pm$3.9 & 0.5 \\
          & medium & 109.8$\pm$1.6 & 109.6$\pm$1.5 & 0.2 & 91.4$\pm$1.2 & 18.4 \\
          & medium-replay & 111.1$\pm$0.3 & 110.6$\pm$0.6 & 0.5 & 103.5$\pm$2.7 & 7.6\\
          \midrule
        \multirow{3}{*}{Walker2d} & random & 93.6$\pm$4.4 & 92.4$\pm$4.7 & 1.2 & 15.4$\pm$0.8 & 78.2\\
          & medium & 107.7$\pm$0.5 & 96.5$\pm$3.5 & 11.2 & 86.4$\pm$0.4 & 21.3  \\
          & medium-replay & 114.4$\pm$0.9 & 103.7$\pm$2.1 & 10.7 & 99.7$\pm$2.4 & 14.7\\
          \midrule
        \multirow{3}{*}{Halfcheetah} & random & 97.7$\pm$1.1 & 94.5$\pm$4.2 & 3.2 & 85.2$\pm$0.5 & 12.5 \\
          & medium & 98.7$\pm$0.3 & 97.7$\pm$0.5 & 1.0 & 80.3$\pm$0.2 & 18.4 \\
          & medium-replay & 91.5$\pm$0.9 & 90.5$\pm$0.5 & 1.0 & 84.8$\pm$1.0 & 6.7  \\
        \bottomrule
    \end{tabular}
    \caption{Ablation results on Mujoco tasks with the normalized score metric.}
    \label{tab: ablation ts}
\end{table*}


We design our experiments to answer the following questions:
(1) Whether BOORL can effectively solve the dilemma in offline-to-online RL?
(2) How does BOORL compare with other state-of-the-art approaches for finetuning pre-trained policies?
(3) Is BOORL general, and can it be effectively combined with other off-the-shelf offline RL algorithms?

To answer the questions above, we conduct experiments to test our proposed approach on the D4RL benchmark~\citep{fu2020d4rl}, which encompasses a variety of dataset qualities and domains.
We adopt the normalized score metric proposed by the D4RL benchmark~\citep{fu2020d4rl}, averaging over five random seeds with standard deviation.

\paragraph{Answer of Question 1:}
We compare BOORL with the online version of TD3+BC~\citep{fujimoto2021minimalist}, named $\texttt{TD3+BC~(online)}$, as well as directly using TD3 for finetuning, named $\texttt{TD3~(finetune)}$. 
For the fair comparison, these three methods are all pre-trained based on TD3+BC for 1 million time steps and adopt the TD3 algorithm for online learning.
The results in Figure~\ref{appendix fig: TD3 result} in Appendix~\ref{appendix: complete exp} show TD3+BC exhibits safe but slow performance improvement, resulting in worse asymptotic performance.
On the other hand, TD3 suffers from initial performance degradation, especially in narrow distribution datasets~(e.g., expert datasets), while BOORL attains a fast performance improvement with a smaller regret.

\paragraph{Answer of Question 2:}
We compare BOORL with several strong offline-to-online algorithms, inclduing ODT~\citep{zheng2022online}, 
AWAC~\citep{nair2020awac}, PEX~\citep{zhang2023policy}, Cal-QL~\citep{nakamoto2023cal} and RLPD~\citep{ball2023efficient}.
We re-run the official implementation to offline pre-train for 1 million steps.
Then, we report the fine-tuned performance for 200k online steps.
As for BOORL, we use TD3+BC and TD3 as the backbone of offline and online algorithms.
Table~\ref{tab: d4rl} shows that our algorithm achieves superior fine-tuning performance and notable performance improvement $\delta_{\rm sum}$ compared with other fine-tuning approaches.
The results in Figure~\ref{fig: new baseline} show that our method achieves better learning efficiency and stability than these baselines.
AWAC has limited efficiency due to a lack of online adaptability. To further test the performance of our method across various timesteps, we compare our method with PEX over 1M steps. Results in Appendix~\ref{appendix: pex} show that our method outperforms PEX across different timesteps.
Cal-QL achieves comparable stability due to calibration, but our method generally demonstrates better sample efficiency.


\paragraph{Answer of Question 3:}
We incorporate BOORL with another offline RL algorithm, IQL, and evaluate it on the sparse reward task in the D4RL benchmark, Antmaze.
Consistent with the previous experimental setup, we first offline train IQL for 1 million time steps and then load the same pre-trained weight for BOORL.
The experimental results in Table~\ref{tab: d4rl} show that BOORL achieves superior performance and higher sample efficiency than other baselines.
This demonstrates that BOORL can be easily extended to various offline RL algorithms.

\paragraph{Ablation Study.}
To delve deeper into the performance of Bayesian methods,
we enforced a strict offline → online transition. Specifically, we exclusively loaded the offline-trained policy and Q-network module, omitting the offline data during the online phase. We refer to this setup as $\texttt{Bayesian}$. Furthermore, we examined the naive offline-to-online (TD3+BC → TD3) with the Hybrid RL framework to examine the effects of integrating offline data, termed $\texttt{Hybrid RL}$. Results in Table~\ref{tab: ablation ts} reveal that Thompson Sampling perform well in most tasks. 
Please refer to Appendix~\ref{appendix: additional exp} for additional ablation studies.

\section{Conclusion}

Our work presents a novel perspective on offline-to-online RL under Bayesian principles. By utilizing a probability-matching strategy, we can effectively tackle the inherent challenge of balancing exploration efficiency and utilizing offline data. Such intuition is formalized with an information-theoretic analysis, and we provide a concrete regret bound on linear MDPs. Previous methods like a balanced replay buffer and ensembling can be rationalized as implementing an implicit probability matching strategy.
Based on the above theoretical insights, we design a two-phase algorithm that implements approximate posterior sampling with bootstrapping, which provides a smooth and robust transition from offline to online.
 Our algorithm outperforms previous methods and demonstrates superior outcomes across various tasks. Adopting more advanced Bayesian methods to enable a more efficient and robust offline-to-online transition is an interesting future direction.
\section*{Impact Statement}
This paper focuses on the offline reinforcement learning area. 
This paper offers a reliable method for effectively improving online reinforcement learning methods by offline pre-training, substantiated by sufficient proof, which provides considerable societal importance. 
In contexts where substantial data is available, our method can save significant time and cost for online reinforcement learning and avoid the performance drop.
This has the potential to expand the current boundaries of application in the field of offline reinforcement learning, making it more accessible and applicable in a broader range of societal contexts.

As for ethical aspects, to the best of our knowledge, the research presented in this paper does not directly engage with them. 
However, we acknowledge the importance of ethical considerations in machine learning research and strive to ensure our work aligns with general ethical standards.

\bibliography{reference}

\begin{thebibliography}{49}
\providecommand{\natexlab}[1]{#1}
\providecommand{\url}[1]{\texttt{#1}}
\expandafter\ifx\csname urlstyle\endcsname\relax
  \providecommand{\doi}[1]{doi: #1}\else
  \providecommand{\doi}{doi: \begingroup \urlstyle{rm}\Url}\fi

\bibitem[Anonymous(2024)]{anonymous2024posterior}
Anonymous.
\newblock Posterior sampling via langevin monte carlo for offline reinforcement
  learning, 2024.
\newblock URL \url{https://openreview.net/forum?id=WwCirclMvl}.

\bibitem[Auer(2002)]{auer2002using}
Auer, P.
\newblock Using confidence bounds for exploitation-exploration trade-offs.
\newblock \emph{Journal of Machine Learning Research}, 3\penalty0
  (Nov):\penalty0 397--422, 2002.

\bibitem[Ball et~al.(2023)Ball, Smith, Kostrikov, and
  Levine]{ball2023efficient}
Ball, P.~J., Smith, L., Kostrikov, I., and Levine, S.
\newblock Efficient online reinforcement learning with offline data.
\newblock \emph{arXiv preprint arXiv:2302.02948}, 2023.

\bibitem[Cai et~al.(2020)Cai, Yang, Jin, and Wang]{cai2020provably}
Cai, Q., Yang, Z., Jin, C., and Wang, Z.
\newblock Provably efficient exploration in policy optimization.
\newblock In \emph{International Conference on Machine Learning}, pp.\
  1283--1294. PMLR, 2020.

\bibitem[Chua et~al.(2018)Chua, Calandra, McAllister, and Levine]{chua2018deep}
Chua, K., Calandra, R., McAllister, R., and Levine, S.
\newblock Deep reinforcement learning in a handful of trials using
  probabilistic dynamics models.
\newblock \emph{Advances in neural information processing systems}, 31, 2018.

\bibitem[Dann et~al.(2021)Dann, Mohri, Zhang, and Zimmert]{dann2021provably}
Dann, C., Mohri, M., Zhang, T., and Zimmert, J.
\newblock A provably efficient model-free posterior sampling method for
  episodic reinforcement learning.
\newblock \emph{Advances in Neural Information Processing Systems},
  34:\penalty0 12040--12051, 2021.

\bibitem[Degrave et~al.(2022)Degrave, Felici, Buchli, Neunert, Tracey,
  Carpanese, Ewalds, Hafner, Abdolmaleki, de~Las~Casas,
  et~al.]{degrave2022magnetic}
Degrave, J., Felici, F., Buchli, J., Neunert, M., Tracey, B., Carpanese, F.,
  Ewalds, T., Hafner, R., Abdolmaleki, A., de~Las~Casas, D., et~al.
\newblock Magnetic control of tokamak plasmas through deep reinforcement
  learning.
\newblock \emph{Nature}, 602\penalty0 (7897):\penalty0 414--419, 2022.

\bibitem[Fu et~al.(2020)Fu, Kumar, Nachum, Tucker, and Levine]{fu2020d4rl}
Fu, J., Kumar, A., Nachum, O., Tucker, G., and Levine, S.
\newblock D4rl: Datasets for deep data-driven reinforcement learning.
\newblock \emph{arXiv preprint arXiv:2004.07219}, 2020.

\bibitem[Fujimoto \& Gu(2021)Fujimoto and Gu]{fujimoto2021minimalist}
Fujimoto, S. and Gu, S.~S.
\newblock A minimalist approach to offline reinforcement learning.
\newblock \emph{Advances in neural information processing systems},
  34:\penalty0 20132--20145, 2021.

\bibitem[Fujimoto et~al.(2018)Fujimoto, Hoof, and
  Meger]{fujimoto2018addressing}
Fujimoto, S., Hoof, H., and Meger, D.
\newblock Addressing function approximation error in actor-critic methods.
\newblock In \emph{International conference on machine learning}, pp.\
  1587--1596. PMLR, 2018.

\bibitem[Gao et~al.(2023)Gao, Schulman, and Hilton]{gao2023scaling}
Gao, L., Schulman, J., and Hilton, J.
\newblock Scaling laws for reward model overoptimization.
\newblock In \emph{International Conference on Machine Learning}, pp.\
  10835--10866. PMLR, 2023.

\bibitem[Ghosh et~al.(2022)Ghosh, Ajay, Agrawal, and Levine]{ghosh2022offline}
Ghosh, D., Ajay, A., Agrawal, P., and Levine, S.
\newblock Offline rl policies should be trained to be adaptive.
\newblock In \emph{International Conference on Machine Learning}, pp.\
  7513--7530. PMLR, 2022.

\bibitem[Gittens \& Tropp(2011)Gittens and Tropp]{gittens2011tail}
Gittens, A. and Tropp, J.~A.
\newblock Tail bounds for all eigenvalues of a sum of random matrices.
\newblock \emph{arXiv preprint arXiv:1104.4513}, 2011.

\bibitem[Hu et~al.(2022)Hu, Yang, Zhao, and Zhang]{hu2022role}
Hu, H., Yang, Y., Zhao, Q., and Zhang, C.
\newblock On the role of discount factor in offline reinforcement learning.
\newblock In \emph{International Conference on Machine Learning}, pp.\
  9072--9098. PMLR, 2022.

\bibitem[Jin et~al.(2020)Jin, Yang, Wang, and Jordan]{jin2020provably}
Jin, C., Yang, Z., Wang, Z., and Jordan, M.~I.
\newblock Provably efficient reinforcement learning with linear function
  approximation.
\newblock In \emph{Conference on Learning Theory}, pp.\  2137--2143. PMLR,
  2020.

\bibitem[Jin et~al.(2021)Jin, Yang, and Wang]{jin2021pessimism}
Jin, Y., Yang, Z., and Wang, Z.
\newblock Is pessimism provably efficient for offline rl?
\newblock In \emph{International Conference on Machine Learning}, pp.\
  5084--5096. PMLR, 2021.

\bibitem[Kumar et~al.(2020)Kumar, Zhou, Tucker, and
  Levine]{kumar2020conservative}
Kumar, A., Zhou, A., Tucker, G., and Levine, S.
\newblock Conservative q-learning for offline reinforcement learning.
\newblock \emph{Advances in Neural Information Processing Systems},
  33:\penalty0 1179--1191, 2020.

\bibitem[Kumar et~al.(2022)Kumar, Agarwal, Geng, Tucker, and
  Levine]{kumar2022offline}
Kumar, A., Agarwal, R., Geng, X., Tucker, G., and Levine, S.
\newblock Offline q-learning on diverse multi-task data both scales and
  generalizes.
\newblock \emph{arXiv preprint arXiv:2211.15144}, 2022.

\bibitem[Lee et~al.(2022)Lee, Seo, Lee, Abbeel, and Shin]{lee2022offline}
Lee, S., Seo, Y., Lee, K., Abbeel, P., and Shin, J.
\newblock Offline-to-online reinforcement learning via balanced replay and
  pessimistic q-ensemble.
\newblock In \emph{Conference on Robot Learning}, pp.\  1702--1712. PMLR, 2022.

\bibitem[Levine et~al.(2020)Levine, Kumar, Tucker, and Fu]{levine2020offline}
Levine, S., Kumar, A., Tucker, G., and Fu, J.
\newblock Offline reinforcement learning: Tutorial, review, and perspectives on
  open problems.
\newblock \emph{arXiv preprint arXiv:2005.01643}, 2020.

\bibitem[Lu \& Van~Roy(2019)Lu and Van~Roy]{lu2019information}
Lu, X. and Van~Roy, B.
\newblock Information-theoretic confidence bounds for reinforcement learning.
\newblock \emph{Advances in Neural Information Processing Systems}, 32, 2019.

\bibitem[Ma et~al.(2021)Ma, Yang, Hu, Liu, Yang, Zhang, Zhao, and
  Liang]{ma2021offline}
Ma, X., Yang, Y., Hu, H., Liu, Q., Yang, J., Zhang, C., Zhao, Q., and Liang, B.
\newblock Offline reinforcement learning with value-based episodic memory.
\newblock \emph{arXiv preprint arXiv:2110.09796}, 2021.

\bibitem[Mnih et~al.(2013)Mnih, Kavukcuoglu, Silver, Graves, Antonoglou,
  Wierstra, and Riedmiller]{mnih2013playing}
Mnih, V., Kavukcuoglu, K., Silver, D., Graves, A., Antonoglou, I., Wierstra,
  D., and Riedmiller, M.
\newblock Playing atari with deep reinforcement learning.
\newblock \emph{arXiv preprint arXiv:1312.5602}, 2013.

\bibitem[Nair et~al.(2020)Nair, Gupta, Dalal, and Levine]{nair2020awac}
Nair, A., Gupta, A., Dalal, M., and Levine, S.
\newblock Awac: Accelerating online reinforcement learning with offline
  datasets.
\newblock \emph{arXiv preprint arXiv:2006.09359}, 2020.

\bibitem[Nakamoto et~al.(2023)Nakamoto, Zhai, Singh, Mark, Ma, Finn, Kumar, and
  Levine]{nakamoto2023cal}
Nakamoto, M., Zhai, Y., Singh, A., Mark, M.~S., Ma, Y., Finn, C., Kumar, A.,
  and Levine, S.
\newblock Cal-ql: Calibrated offline rl pre-training for efficient online
  fine-tuning.
\newblock \emph{arXiv preprint arXiv:2303.05479}, 2023.

\bibitem[Osband \& Van~Roy(2017)Osband and Van~Roy]{osband2017posterior}
Osband, I. and Van~Roy, B.
\newblock Why is posterior sampling better than optimism for reinforcement
  learning?
\newblock In \emph{International conference on machine learning}, pp.\
  2701--2710. PMLR, 2017.

\bibitem[Osband et~al.(2016)Osband, Blundell, Pritzel, and
  Van~Roy]{osband2016deep}
Osband, I., Blundell, C., Pritzel, A., and Van~Roy, B.
\newblock Deep exploration via bootstrapped dqn.
\newblock \emph{Advances in neural information processing systems}, 29, 2016.

\bibitem[Ouyang et~al.(2022)Ouyang, Wu, Jiang, Almeida, Wainwright, Mishkin,
  Zhang, Agarwal, Slama, Ray, et~al.]{ouyang2022training}
Ouyang, L., Wu, J., Jiang, X., Almeida, D., Wainwright, C.~L., Mishkin, P.,
  Zhang, C., Agarwal, S., Slama, K., Ray, A., et~al.
\newblock Training language models to follow instructions with human feedback.
\newblock \emph{arXiv preprint arXiv:2203.02155}, 2022.

\bibitem[Ouyang et~al.(2017)Ouyang, Gagrani, Nayyar, and
  Jain]{ouyang2017learning}
Ouyang, Y., Gagrani, M., Nayyar, A., and Jain, R.
\newblock Learning unknown markov decision processes: A thompson sampling
  approach.
\newblock \emph{Advances in neural information processing systems}, 30, 2017.

\bibitem[Rashidinejad et~al.(2021)Rashidinejad, Zhu, Ma, Jiao, and
  Russell]{rashidinejad2021bridging}
Rashidinejad, P., Zhu, B., Ma, C., Jiao, J., and Russell, S.
\newblock Bridging offline reinforcement learning and imitation learning: A
  tale of pessimism.
\newblock \emph{Advances in Neural Information Processing Systems},
  34:\penalty0 11702--11716, 2021.

\bibitem[Ross \& Bagnell(2012)Ross and Bagnell]{ross2012agnostic}
Ross, S. and Bagnell, J.~A.
\newblock Agnostic system identification for model-based reinforcement
  learning.
\newblock \emph{arXiv preprint arXiv:1203.1007}, 2012.

\bibitem[Russo \& Van~Roy(2014)Russo and Van~Roy]{russo2014learning}
Russo, D. and Van~Roy, B.
\newblock Learning to optimize via posterior sampling.
\newblock \emph{Mathematics of Operations Research}, 39\penalty0 (4):\penalty0
  1221--1243, 2014.

\bibitem[Russo \& Van~Roy(2016)Russo and Van~Roy]{russo2016information}
Russo, D. and Van~Roy, B.
\newblock An information-theoretic analysis of thompson sampling.
\newblock \emph{The Journal of Machine Learning Research}, 17\penalty0
  (1):\penalty0 2442--2471, 2016.

\bibitem[Silver et~al.(2016)Silver, Huang, Maddison, Guez, Sifre, Van
  Den~Driessche, Schrittwieser, Antonoglou, Panneershelvam, Lanctot,
  et~al.]{silver2016mastering}
Silver, D., Huang, A., Maddison, C.~J., Guez, A., Sifre, L., Van Den~Driessche,
  G., Schrittwieser, J., Antonoglou, I., Panneershelvam, V., Lanctot, M.,
  et~al.
\newblock Mastering the game of go with deep neural networks and tree search.
\newblock \emph{nature}, 529\penalty0 (7587):\penalty0 484--489, 2016.

\bibitem[Song et~al.(2022)Song, Zhou, Sekhari, Bagnell, Krishnamurthy, and
  Sun]{song2022hybrid}
Song, Y., Zhou, Y., Sekhari, A., Bagnell, J.~A., Krishnamurthy, A., and Sun, W.
\newblock Hybrid rl: Using both offline and online data can make rl efficient.
\newblock \emph{arXiv preprint arXiv:2210.06718}, 2022.

\bibitem[Uehara \& Sun(2021)Uehara and Sun]{uehara2021pessimistic}
Uehara, M. and Sun, W.
\newblock Pessimistic model-based offline reinforcement learning under partial
  coverage.
\newblock \emph{arXiv preprint arXiv:2107.06226}, 2021.

\bibitem[Wagenmaker \& Pacchiano(2023)Wagenmaker and
  Pacchiano]{wagenmaker2023leveraging}
Wagenmaker, A. and Pacchiano, A.
\newblock Leveraging offline data in online reinforcement learning.
\newblock In \emph{International Conference on Machine Learning}, pp.\
  35300--35338. PMLR, 2023.

\bibitem[Xie et~al.(2021)Xie, Jiang, Wang, Xiong, and Bai]{xie2021policy}
Xie, T., Jiang, N., Wang, H., Xiong, C., and Bai, Y.
\newblock Policy finetuning: Bridging sample-efficient offline and online
  reinforcement learning.
\newblock \emph{Advances in neural information processing systems},
  34:\penalty0 27395--27407, 2021.

\bibitem[Xie et~al.(2022)Xie, Foster, Bai, Jiang, and Kakade]{xie2022role}
Xie, T., Foster, D.~J., Bai, Y., Jiang, N., and Kakade, S.~M.
\newblock The role of coverage in online reinforcement learning.
\newblock \emph{arXiv preprint arXiv:2210.04157}, 2022.

\bibitem[Xu \& Zeevi(2023)Xu and Zeevi]{xu2023bayesian}
Xu, Y. and Zeevi, A.
\newblock Bayesian design principles for frequentist sequential learning.
\newblock In \emph{International Conference on Machine Learning}, pp.\
  38768--38800. PMLR, 2023.

\bibitem[Yang \& Wang(2020)Yang and Wang]{yang2020reinforcement}
Yang, L. and Wang, M.
\newblock Reinforcement learning in feature space: Matrix bandit, kernels, and
  regret bound.
\newblock In \emph{International Conference on Machine Learning}, pp.\
  10746--10756. PMLR, 2020.

\bibitem[Yang et~al.(2021)Yang, Ma, Li, Zheng, Zhang, Huang, Yang, and
  Zhao]{yang2021believe}
Yang, Y., Ma, X., Li, C., Zheng, Z., Zhang, Q., Huang, G., Yang, J., and Zhao,
  Q.
\newblock Believe what you see: Implicit constraint approach for offline
  multi-agent reinforcement learning.
\newblock \emph{Advances in Neural Information Processing Systems},
  34:\penalty0 10299--10312, 2021.

\bibitem[Yang et~al.(2023)Yang, Hu, Li, Li, Yang, Zhao, and
  Zhang]{yang2023flow}
Yang, Y., Hu, H., Li, W., Li, S., Yang, J., Zhao, Q., and Zhang, C.
\newblock Flow to control: Offline reinforcement learning with lossless
  primitive discovery.
\newblock In \emph{Proceedings of the AAAI Conference on Artificial
  Intelligence}, volume~37, pp.\  10843--10851, 2023.

\bibitem[Yu et~al.(2020)Yu, Thomas, Yu, Ermon, Zou, Levine, Finn, and
  Ma]{yu2020mopo}
Yu, T., Thomas, G., Yu, L., Ermon, S., Zou, J.~Y., Levine, S., Finn, C., and
  Ma, T.
\newblock Mopo: Model-based offline policy optimization.
\newblock \emph{Advances in Neural Information Processing Systems},
  33:\penalty0 14129--14142, 2020.

\bibitem[Yu et~al.(2021)Yu, Kumar, Rafailov, Rajeswaran, Levine, and
  Finn]{yu2021combo}
Yu, T., Kumar, A., Rafailov, R., Rajeswaran, A., Levine, S., and Finn, C.
\newblock Combo: Conservative offline model-based policy optimization.
\newblock \emph{Advances in neural information processing systems},
  34:\penalty0 28954--28967, 2021.

\bibitem[Yu \& Zhang(2023)Yu and Zhang]{yu2023actor}
Yu, Z. and Zhang, X.
\newblock Actor-critic alignment for offline-to-online reinforcement learning.
\newblock In \emph{International Conference on Machine Learning}, pp.\
  40452--40474. PMLR, 2023.

\bibitem[Zhang et~al.(2023)Zhang, Xu, and Yu]{zhang2023policy}
Zhang, H., Xu, W., and Yu, H.
\newblock Policy expansion for bridging offline-to-online reinforcement
  learning.
\newblock \emph{arXiv preprint arXiv:2302.00935}, 2023.

\bibitem[Zheng et~al.(2022)Zheng, Zhang, and Grover]{zheng2022online}
Zheng, Q., Zhang, A., and Grover, A.
\newblock Online decision transformer.
\newblock In \emph{International Conference on Machine Learning}, pp.\
  27042--27059. PMLR, 2022.

\bibitem[Zhong et~al.(2022)Zhong, Xiong, Zheng, Wang, Wang, Yang, and
  Zhang]{zhong2022gec}
Zhong, H., Xiong, W., Zheng, S., Wang, L., Wang, Z., Yang, Z., and Zhang, T.
\newblock Gec: A unified framework for interactive decision making in mdp,
  pomdp, and beyond.
\newblock \emph{arXiv preprint arXiv:2211.01962}, 2022.

\end{thebibliography}
\bibliographystyle{icml2024}

\newpage
\appendix
\onecolumn

\section{Algorithm Details}
\subsection{Details of abstract algorithms}
\label{appendix: alg_abs}

In this section, we provide an information-theoretic abstraction of the UCB, LCB and Thompson Sampling algorithm when both offline dataset and online environment are present. Here we use $t$ to represent $(k,h)$ for simplicity.

\begin{algorithm}[h]
\caption{Upper Confidence Bound}
\label{alg: abs_ucb}
\begin{algorithmic}[1]
    \STATE {\bf Require}: offline dataset $\cD^{\rm off}$, online interaction episodes $K$, exploration coefficient $\Gamma_t$,
    \STATE Initialize prior $\beta_0$ with the offline dataset $\cD^{\rm off}$.
    \FOR{$k=1, \cdots, K$}
        \FOR{$h=1,\cdots, H$}
        \STATE Calculate posterior mean $\bar{Q}_{w,t}(\cdot,\cdot) = \EE_{w\sim \beta_t} [Q_{t,w}(\cdot,\cdot)]$.
        \STATE Calculate optimistic value function $\widehat{Q}(\cdot,\cdot) = \bar{Q}_{w,t}(\cdot,\cdot) + \frac{\Gamma_t}{2} \sqrt{I_t(w;\cdot,\cdot)}$.
        \STATE Execute $a_t = \argmax_{a_t} \widehat{Q}(s_t,a)$ and receive feedback $s_{t+1}, r_t$.
        \STATE Update posterior $\beta_t$ with evidence $(a_t,r_t,s_{t+1})$.
        \ENDFOR
    \ENDFOR
\end{algorithmic}
\end{algorithm}

\begin{algorithm}[h]
\caption{Lower Confidence Bound}
\label{alg: abs_lcb}
\begin{algorithmic}[1]
    \STATE {\bf Require}: offline dataset $\cD^{\rm off}$, online interaction episodes $K$, exploration coefficient $\Gamma_t$,
    \STATE Initialize prior $\beta_0$ with the offline dataset $\cD^{\rm off}$.
    \FOR{$k=1, \cdots, K$}
        \FOR{$h=1,\cdots, H$}
        \STATE Calculate posterior mean $\bar{Q}_{w,t}(\cdot,\cdot) = \EE_{w\sim \beta_t} [Q_{t,w}(\cdot,\cdot)]$.
        \STATE Calculate pessimistic value function $\widehat{Q}(\cdot,\cdot) = \bar{Q}_{w,t}(\cdot,\cdot) - \frac{\Gamma_t}{2} \sqrt{I_t(w;\cdot,\cdot)}$.
        \STATE Execute $a_t = \argmax_{a_t} \widehat{Q}(s_t,a)$ and receive feedback $s_{t+1}, r_t$.
        \STATE Update posterior $\beta_t$ with evidence $(a_t,r_t,s_{t+1})$.
        \ENDFOR
    \ENDFOR
\end{algorithmic}
\end{algorithm}

\begin{algorithm}[h]
\caption{Thompson Sampling}
\label{alg: abs_ts}
\begin{algorithmic}[1]
    \STATE {\bf Require}: offline dataset $\cD^{\rm off}$, online interaction episodes $K$, exploration coefficient $\Gamma_t$,
    \STATE Initialize prior $\beta_0$ with the offline dataset $\cD^{\rm off}$.
    \FOR{$k=1, \cdots, K$}
        \FOR{$h=1,\cdots, H$}
        \STATE Sample parameter $w_t$ from posterior $\beta_t$. 
        \STATE Calculate corresponding value function $\widehat{Q}(\cdot,\cdot) = {Q}_{w_t,t}(\cdot,\cdot)$.
        \STATE Execute $a_t = \argmax_{a_t} \widehat{Q}(s_t,a)$ and receive feedback $s_{t+1}, r_t$.
        \STATE Update posterior $\beta_t$ with evidence $(a_t,r_t,s_{t+1})$.
        \ENDFOR
    \ENDFOR
\end{algorithmic}
\end{algorithm}

\newpage

\clearpage

\section{Additional Propositions}
\label{sec: proof}

\begin{proposition}
    \label{online_guarantee}
    Suppose the following equation for each episode holds with $\Gamma_{t} \leq \Gamma$ for all $k \in [K], h \in [H]$, 
    \#
    \EE_k[\Delta_k] \leq \sum_{h=1}^H{\Gamma_{t} \sqrt{I_{t}(w_h;a_{t},r_{t},s_{t+1})}+ \epsilon_{t}}. \label{regret_info_bound}
    \#
    Then
    \$
    \EE[\text{Regret}(T,\pi)] \leq \Gamma \sqrt{T I(w;\cH_{T})}+ \EE{\sum_{k=1}^K \sum_{h=1}^H \epsilon_{t}} .
    \$
\end{proposition}
\begin{proof}
    By leveraging the chain rule of mutual information, i.e. 
    \$
    I(X;Y_1,\ldots,Y_N) = \sum_{i=1}^N I(X;Y_i|Y_1,\ldots,Y_{i-1}),
    \$
    and the Cauchy-Schwartz inequality, we have the result immediately.
\end{proof}

\section{Missing Proofs}

\subsection{Proof of Theorem~\ref{theorem:information}}
\label{proof:information}
\begin{proof}
    By regret decomposition in Lemma~\ref{regret_decompose}, we have 
    \$
    \EE_k[\Delta_k] &=\sum_{h=1}^H\EE_{\pi^*}\left[\langle Q_{t}(s_h,\cdot),\pi^*_h(\cdot\given s_h)- \pi_{h,k}(\cdot \given s_h)\rangle\right]+ \sum_{h=1}^{H}\left(\EE_{\pi^*}[\iota_{t}(s,a)] - \EE_k{[\iota_{t}(s,a)]}\right)\\
    &\leq \sum_{h=1}^{H}\left(\EE_{\pi^*}[\iota_{t}(s,a)] - \EE_k{[\iota_{t}(s,a)]}\right),
    \$
    where $\iota_{t}(s,a) = r_{t}(s,a) + (\BB_h V_{t+1})(s,a) - Q_{t+1}(s,a)$. The inequality is due to the fact that $\pi_{t}$ is the greedy policy with respect to $Q_{t}$.

    Let $\cW_k$ be the confidence set of $w$ such that Equation~\eqref{eq:info_ratio} holds, we have 
    \$
    \PP_{k}(w\in\cW_k)\geq 1-\frac{\delta}{2}.
    \$
    For a UCB algorithm with upper confidence estimation $Q_{t}(s,a)=\bar{Q}_{t,w}(s,a)+ \frac{\Gamma_{t}}{2}\sqrt{I_{t}(w_h;r_{t,a},s_{t+1,a})}$, we have
    \$-\Gamma_{t}\sqrt{I_{t}(w_h;r_{t,a},s_{t+1,a})} \leq \iota_{t}(s,a) \leq 0\$ 
    when $w\in\cW_k$. Here $r_ {t,a}$ and $s_ {t+1,a}$ are the random variables for the reward and next state given action $a$ at time $t$.
    
    Then we have
    \$
    &\sum_{h=1}^{H}\left(\EE_{\pi^*}[\iota_{t}(s,a)] - \EE_k{[\iota_{t}(s,a)]}\right) \\
    \leq& \sum_{h=1}^{H}\left\{\mathbbm{1}_{w \in \cW_k} \left\{ \EE_{\pi^*}[\iota_{t}(s,a)] - \EE_k{[\iota_{t}(s,a)]}\right\}+\frac{1}{2}\delta \cdot 2H\right\} \\
    \leq& \sum_{h=1}^{H}{\EE_k\left[\sum_{a\in\cA}\PP{(a_{t}=a)}\Gamma_{t}\sqrt{I_{t}(w_h;r_{t,a},s_{t+1,a})}\right]} + \delta H^2\\
    \leq& \sum_{h=1}^{H}{\EE_k\left[\Gamma_{t}\sqrt{\sum_{a\in\cA}\PP{(a_{t}=a)}I_{t}(w_h;r_{t,a},s_{t+1,a})}\right]} + \delta H^2\\
    =&\sum_{h=1}^{H}{\EE_k\left[\Gamma_{t}\sqrt{\sum_{a\in\cA}\PP{(a_{t}=a)}I_{t}(w_h;r_{t,a_{t}},s_{t+1,a_{t}}\given a_{t}=a)}\right]} + \delta H^2\\
    =&\sum_{h=1}^{H}{\EE_k\left[\Gamma_{t}\sqrt{I_{t}(w_h;a_{t}, r_{t},s_{t+1})}\right]} + \delta H^2.\\
    \$

    Similarly, for a LCB algorithm with lower confidence functions $Q_{t}(s,a)=\bar{Q}_{t,w}(s,a) - \frac{\Gamma_{t}}{2}\sqrt{I_{t}(w_h;r_{t,a},s_{t+1,a})}$, we have
    \$0\leq \iota_{t}(s,a) \leq \Gamma_{t}\sqrt{I_{t}(w_h;r_{t,a},s_{t+1,a})}\$ 
    when $w\in\cW_k$.Then we have
    \$
    &\sum_{h=1}^{H}\left(\EE_{\pi^*}[\iota_{t}(s,a)] - \EE_k{[\iota_{t}(s,a)]}\right) \\
    \leq& \sum_{h=1}^{H}\left\{\mathbbm{1}_{w \in \cW_k} \left\{ \EE_{\pi^*}[\iota_{t}(s,a)] - \EE_k{[\iota_{t}(s,a)]}\right\}+\frac{1}{2}\delta \cdot 2H\right\} \\
    \leq& \sum_{h=1}^{H}{\EE_{\pi^*}\left[\sum_{a\in\cA}\PP{(a^*_{t}=a)}\Gamma_{t}\sqrt{I_{t}(w_h;r_{t,a},s_{t+1,a})}\right]} + \delta H^2\\
    \leq& \sum_{h=1}^{H}{\EE_{\pi^*}\left[\Gamma_{t}\sqrt{\sum_{a\in\cA}\PP{(a^*_{t}=a)}I_{t}(w_h;r_{t,a},s_{t+1,a})}\right]} + \delta H^2\\
    =&\sum_{h=1}^{H}{\EE_{\pi^*}\left[\Gamma_{t}\sqrt{\sum_{a\in\cA}\PP{(a^*_{t}=a)}I_{t}(w_h;r_{t,a^*_{t}},s_{t+1,a^*_{t}}\given a^*_{t}=a)}\right]} + \delta H^2\\
    =&\sum_{h=1}^{H}{\EE_{\pi^*}\left[\Gamma_{t}\sqrt{I_{t}(w_h;a^*_{t}, r_{t},s_{t+1})}\right]} + \delta H^2.\\
    \$

    For Thompson Sampling, note that the probability matching property implies that $\PP_k(\widehat{w}_k\in\cW_k) = \PP_k(w\in\cW_k) \geq 1-\frac{\delta}{2}$, we have 

    \$
    &\sum_{h=1}^{H}\left(\EE_{\pi^*}[\iota_{t}(s,a)] - \EE_k{[\iota_{t}(s,a)]}\right) \\
    \leq& \sum_{h=1}^{H}\left\{\mathbbm{1}_{w, \widehat{w}_k \in \cW_k} \left\{ \EE_{\pi^*}[\iota_{t}(s,a)] - \EE_k{[\iota_{t}(s,a)]}\right\}+\delta \cdot 2H\right\} \\
    \leq& \sum_{h=1}^{H}{\EE_{\pi^*}\left[\sum_{a\in\cA}\PP{(a^*_{t}=a)}\Gamma_{t}\sqrt{I_{t}(w_h;r_{t,a},s_{t+1,a})}\right]} + 2\delta H^2\\
    =& \sum_{h=1}^{H}{\EE_{k}\left[\sum_{a\in\cA}\PP{(a_{t}=a)}\Gamma_{t}\sqrt{I_{t}(w_h;r_{t,a},s_{t+1,a})}\right]} + 2\delta H^2.\\
    \$
    The last equality and the last equality is due to the fact that $\pi_{k,h}$ is the optimal policy under $\widehat{w}_k$, and $w$ and $\widehat{w}_k$ have the same distribution. 
    The rest of the proof is similar to the case of UCB and LCB and is omitted for simplicity.
\end{proof}

\subsection{Proof of Theorem~\ref{theorem:1}}
\label{proof:theorem1}

\begin{theorem}[Regret of Bayesian Agents in Linear MDPs, restatment]
    \label{theorem:1_restate}
    Given an offline dataset $\cD$ of size $N$, and a fixed posterior $\beta$ during the online interaction phase, 
    the regret of Thompson sampling during online interaction satisfies the following bound:
    \#
        \text{\rm BayesRegret}(N,T,\pi)  \leq 4c\sqrt{d^3H^3\iota} \left(\sqrt{\frac{N}{C^\dagger_\beta}+ T}-\sqrt{\frac{N}{C^\dagger_\beta}}\right),
    \#
    for sufficiently large $N$ or $T$. Here $\iota$ is a logarithmic factor and $c$ is an absolute constant.
\end{theorem}
\begin{proof}

    At each online episode $k$, we have
    \begin{align}
        \label{eq:bound_eigen}
        &\sum_{h=1}^{H}\Gamma_{t} \EE_{\pi^*}\left[\sqrt{I_{t}(w_h;a^*_{t},r_{t},s_{t+1})}\right]\notag\\ 
        =&\sum_{h=1}^{H}\Gamma_{t} \EE_{\pi^*}\left[\log{(1+\phi(s_h,a_h)^\top\Lambda_h^{-1}\phi(s_h,a_h))^{1/2}}\right]\notag\\ 
        \leq &\Gamma_t \EE_{\pi^*}\Bigl[ \sum_{h=1}^H  \bigl(\phi(s_t,a_t)^\top \Lambda_k^{-1}\phi(s_t,a_t)\bigr)^{1/2}\Bigr]\notag \\
        =& \Gamma_t\EE_{\pi^*}\Bigl[ \sum_{h=1}^H \sqrt{\Tr\big(\phi(s,a)\phi(s,a)^\top \Lambda_k^{-1}\big)} \Bigr]\notag \\
        \leq&   \Gamma_t \sum_{h=1}^H \sqrt{\Tr\Big(\EE_{d^{\pi^*}_h}\big[\phi(s,a)\phi(s,a)^\top \big]\Lambda^{-1}_{t}\Big)} \notag \\
        =& \Gamma_t \sum_{h=1}^H \sqrt{\Tr\Big(\Sigma_{\pi^*,h}^\top \Lambda^{-1}_{t}\Big)},
    \end{align}
    where $\Sigma_{\pi^*,h} \overset{\Delta}{=} \EE_{d^{\pi^*}_h}\big[\phi(s,a)\phi(s,a)^\top \big]$. The first equality uses Lemma~\ref{linear_mi}, The first inequality uses the fact that $\log(1+x)\leq x \forall x\geq 0 $. The second equality uses the trace trick and the last inequality due to Jensen inequality and the linearity of the trace function.

    By the definition of Bayesian coverage coefficient, we have 
    \$
    \EE{\left[\sum_{\ell=1}^{L}\phi(s_{\ell,h},a_{\ell,h})\phi(s_{\ell,h},a_{\ell,h})^{\top}\right]}  \succeq \frac{L}{C^\dagger_\beta} \Sigma_{\pi^*_\beta,h},
    \$
    where $\Sigma_{\pi^*_\beta,h} \overset{\Delta}{=} \EE_{w\sim\beta}\EE_{d^{\pi^*_w}_h}\big[\phi(s,a)\phi(s,a)^\top \big]$.

    From the probability matching property of Thompson sampling method, we have 
    \$
    \EE{\left[\sum_{k=1}^{K}\phi(s_{t},a_{t})\phi(s_{t},a_{t})^{\top}\right]} = K \Sigma_{\pi^*_\beta,h}.
    \$
    From matrix concentration inequalities~\citep{gittens2011tail}, with a probability $1-\xi$ where $\xi=\frac{d}{H} e^{-\frac{4(L+KC^\dagger_\beta)}{C^\dagger_\beta\kappa^2_{\beta}}}$, we have 
    \#
    \label{matrix_concentration}
    \sum_{\ell=1}^{L}\phi(s_{\ell,h},a_{\ell,h})\phi(s_{\ell,h},a_{\ell,h})^{\top}+\sum_{k=1}^{K}\phi(s_{t},a_{t})\phi(s_{t},a_{t})^{\top} \succeq \frac{1}{2}(\frac{L}{C^\dagger_\beta}+K) \Sigma_{\pi^*_\beta}.
    \#
    Here $\kappa_{\beta} = \max_{h\in[H]} \frac{\lambda_{\text{max}}(\Sigma_{\pi^*_\beta,h})}{\lambda_{\text{min}}(\Sigma_{\pi^*_\beta,h})}$ is the condition number for the feature matrix under expert policy $\pi^*$, $\lambda_{\text{max}}$ is the largest eigenvalue and $\lambda_{\text{min}}$ is the smallest \textit{non-zero} eigenvalue. 
    Let $\cE$ be the event such that Equation~\eqref{matrix_concentration} holds, then we have 
    \$
    &\EE_\beta[\Delta_k] \\
    \leq & \EE_\beta \left[\sum_{h=1}^{H}\Gamma_{t} \EE_{\pi^*}\left[\sqrt{I_{t}(w_h;a^*_{t},r_{t},s_{t+1})}\right]\right] + 2\delta H^2 \\
    \leq & \Gamma \sum_{h=1}^{H}\sqrt{\EE_\beta \Tr\Big(\Sigma_{\pi^*,h}^\top \Lambda^{-1}_{t}\Big)}  + 2\delta H^2 \\
    \leq & \mathbbm{1}_{\cE}\left\{\Gamma \sum_{h=1}^{H}\sqrt{\EE_\beta \Tr\Big(\Sigma_{\pi^*,h}^\top \Lambda^{-1}_{t}\Big)}\right\}  + \xi H^2 + 2\delta H^2 \\
    \leq & \Gamma \sum_{h=1}^{H}\sqrt{\EE_\beta \Tr\Big(\Sigma_{\pi^*,h}^\top \big(\lambda\cdot I + \frac{1}{2}(\frac{L}{C^\dagger_\beta}+K) \Sigma_{\pi^*,h} \big)^{-1}\Big)}+ \xi H^2  + 2\delta H^2 \\
    \leq & \Gamma \sum_{h=1}^{H}\sqrt{\sum_{j=1}^d \frac{\lambda_{j}(h)}{\lambda +\frac{1}{2}(\frac{L}{C^\dagger_\beta}+K) \cdot \lambda_{j}(h)}}  + (2\delta+\xi) H^2. \\
    \$
    Here $\{\lambda_{j}(h)\}_{j=1}^d$ are the eigenvalues of $\Sigma_{\pi^*,h}$ for all $h\in [H]$. The first inequality follows from Lemma~\ref{linear_uncertainty}, and the second to last inequality follows from the Jensen inequality and the definition of event $\cE$.

    Meanwhile, by definition, we have $\|\phi(s,a)\|\leq 1$ for all $(s,a)\in \cS \times \cA$. By Jensen's inequality, we have
    \begin{equation}
        \|\Sigma_{\pi^*,h}\|_{\oper} \leq \EE_{\pi^*}\big[ \|\phi(s,a)\phi(s,a)^\top \|_{\oper}  \big] \leq 1
    \end{equation}
    for all $h \in [H]$. As $\Sigma_{\pi^*,h}$ is positive semidefinite, we have $\lambda_{j}(h) \in [0,1]$ for all $h\in [H]$ and all $j\in [d]$. 

    Then we have 
    \$
    \EE_\beta[\Delta_k] \leq & \Gamma \sum_{h=1}^{H}\sqrt{\sum_{j=1}^d \frac{1}{\lambda+\frac{1}{2}(\frac{L}{C^\dagger_\beta}+K)}}  + (2\delta+\xi) H^2 \\
    \leq & H\Gamma\sqrt{ \frac{2d}{\frac{L}{C^\dagger_\beta}+K}}  + (2\delta+\xi) H^2 \\
    \leq & 2c \sqrt{ \frac{2d^3H^3\iota}{\frac{L}{C^\dagger_\beta}+K}}  + (2\delta+\xi) H^2, \\
    \$
    where $\iota=\log{\frac{4dT}{\delta}}$. For sufficiently large $L$ and $K$ such that $\xi = (\frac{L}{C^\dagger_\beta}+K)^{-1/2}K^{-1}$and let $\delta=(\frac{L}{C^\dagger_\beta}+K)^{-1/2}K^{-1}$. Using the fact that 
    \$
    \sum_{k=1}^K \sqrt{\frac{1}{a+bk}}\leq \int_{0}^{K}{\sqrt{\frac{1}{a+bx}}dx} \leq \frac{2}{b} (\sqrt{a+bK}-\sqrt{a}),
    \$ 
    we have the desired result.
\end{proof}

\subsection{Proof of Proposition~\ref{prop:1}}
\label{proof:prop1}
\begin{proof}
    Let $T=1$ in Theorem~\ref{theorem:1} and note that $\sqrt{x+1}-\sqrt{x}\leq 2/\sqrt{x}$, we have the result for Thompson sampling immediately. For an counterexample for UCB, Please refer to Lemma~\ref{lemma:ucb_failure}. 
\end{proof}

\subsection{Proof of Proposition~\ref{prop:2}}
\label{proof:prop2}
\begin{proof}
    Let $N=0$ in Theorem~\ref{theorem:1}, and we have the result for Thompson sampling immediately. For an counterexample for LCB, Please refer to Lemma~\ref{lemma:lcb_failure}.

\end{proof}

\subsection{Failure of UCB and LCB}
\begin{lemma}[Failure of UCB]
    \label{lemma:ucb_failure}
    For any $\epsilon < 0.05$, $N \geq 500$, there exists a bandit problem with two arms such that for $\widehat{a}^{\text{UCB}} = \argmax_a
\widehat{r}(a)+k \sqrt{\frac{\log{N}}{N_a}}$, with $k>0$, one has
    \begin{equation}
        \EE_{\cD}[r(a^\star)-r(\widehat{a}^{\text{UCB}})] \geq \epsilon 
    \end{equation}
\end{lemma}

\begin{proof}
    The following proof is mainly adapted from Proposition 1 in \citet{rashidinejad2021bridging}.
    Consider a two-arm bandit $\cA=\{1,2\}$, where $r(a_1)=2\epsilon$, and 
    \begin{equation}
        r(a_2) = \left\{ 
            \begin{array}{lc}
                2.1\epsilon & p =0.5 \\
                0 &p=0.5\\
            \end{array}
        \right.
    \end{equation}
    
    and the behavior policy for data collection satisfies $\mu(a_1) = (N-1)/N, \mu(a_2) = 1/N$.
    Then consider the following event $\cE = \{N(a_2)=1\}$, We have $P(\cE)= (1-1/N)^{N-1}$. As long as $N\geq 500$, we have $P(\cE)\geq 0.36$.
    Then with probability $p\geq 0.18$, we have $\widehat{r}(1)=2\epsilon, \widehat{r}(2)=2.1\epsilon$. Note that the bonus term for $a_2$ is larger than for $a_1$ sinee $a_1$ is pulled more than $a_2$, then we have $P(\widehat{a}_{\text{UCB}}=a_2)\geq p=0.18$.

    Finally, we have 
    \begin{equation}
        \EE_{\cD}[r(a^\star)-r(\widehat{a}_{\text{UCB}})] \geq 0.95\epsilon \cdot p \geq 0.1\epsilon.
    \end{equation}


\end{proof}

\begin{lemma}[Failure of LCB]
    \label{lemma:lcb_failure}
    For any $\epsilon < 0.05$, $N \geq 500$, there exists a bandit problem with two arms such that for $\widehat{a}_{\text{LCB}} = \argmax_a
\widehat{r}(a)-k \sqrt{\frac{\log{N}}{N_a}}$, with $k>0$, one has

    \begin{equation}
        \EE_{\cD}\left[\sum_{t=1}^T (r(a^\star)-r(\widehat{a}^{\text{LCB}}_t))\right] \geq 0.1\epsilon\cdot T 
    \end{equation}

\end{lemma}

\begin{proof}
    Similar to Lemma~\ref{lemma:ucb_failure}, consider a two-arm bandit $\cA=\{1,2\}$, where $r(a_1)=\epsilon$, and 
    \begin{equation}
        r(a_2) = \left\{ 
            \begin{array}{lc}
                4\epsilon & p =0.5 \\
                0 &p=0.5\\
            \end{array}
        \right.
    \end{equation}
    
    and the behavior policy for data collection satisfies $\mu(a_1) = (N-1)/N, \mu(a_2) = 1/N$.
    Then consider the following event $\cE = \{N(a_2)=1\}$, We have $P(\cE)= (1-1/N)^{N-1}$. As long as $N\geq 500$, we have $P(\cE)\geq 0.36$.
    Then with probability $p\geq 0.18$, we have $\widehat{r}(1)=\epsilon, \widehat{r}(2)=0$. Note that the bonus term for $a_2$ is larger than for $a_1$ sinee $a_1$ is pulled more than $a_2$, then for any $t$, we have $P(\widehat{a}^{\text{LCB}}=a_1)\geq p=0.18$. Note that pulling $a_1$ does not obtain any information, we can conclude that $P(\widehat{a}^{\text{LCB}}_t=a_1)\geq p=0.18$ for any $t$.

    Finally, we have 

    \begin{equation}
        \EE_{\cD}\left[\sum_{t=1}^T (r(a^\star)-r(\widehat{a}^{\text{LCB}}_t))\right] \geq \epsilon \cdot p \cdot T \geq 0.1\epsilon\cdot T
    \end{equation}


\end{proof}

\section{Auxiliary Lemmas}
\begin{lemma}[Regret Decomposition~\citep{cai2020provably}]
    \label{regret_decompose}
    We define the model prediction error as 
\#\label{eq:w11260901}
\iota_{k,h}(s,a) = r_{k,h}(s,a) + (\BB_h V_{k,h+1})(s,a) - Q_{k,h+1}(s,a),
\#
which arises from estimating $\mathbb{P}_h V^k_{h+1}$ in the Bellman equation based on only finite historical data. 
Also, we define the following filtration generated by the state-action sequence and reward functions.

\begin{definition}[Filtration]\label{def:w001}
 For any $(t)\in[K]\times[H]$, we define $\cF_{t,1}$ as the $\sigma$-algebra generated by the following state-action sequence and reward functions,
 \$
 \{(s_{\tau,i}, a_{\tau,i})\}_{(\tau, i)\in [k-1] \times [H]} \cup \{r^\tau\}_{\tau\in [k]} \cup \{(s_{k,h}, a_{k,h})\}_{i\in [h]} ,
 \$
 and $\cF_{t,2}$ as the $\sigma$-algebra generated by 
  \$
& \{(s_{\tau,i}, a_{\tau,i})\}_{(\tau, i)\in [k-1] \times [H]} \cup \{r^\tau\}_{\tau\in [k]}   \cup \{(s_{k,h}, a_{k,h})\}_{i\in [h]} \cup \{s^k_{h+1}\},
 \$
where, for the simplicity of discussion, we define $s^k_{H+1}$ as a null state for any $k\in [K]$. 
\end{definition}

    It holds that
    \# \label{1015243}
    \text{Regret}(T) &= \sum_{k=1}^K \bigl(V^{\pi^*,k}_1(s^k_1) - V^{\pi^k,k}_1(s^k_1)\bigr) \notag\\
    &=  \sum_{k=1}^K\sum_{h=1}^H \EE_{\pi^*} \bigl[ \la Q^{k}_h(s_h,\cdot), \pi^*_h(\cdot\,|\,s_h) - \pi_{k,h}(\cdot\,|\,s_h) \ra \bigr] + \cM_{K, H, 2} \notag \\
    &\qquad+\sum_{k=1}^K\sum_{h=1}^H\bigl( \EE_{\pi^*}[\iota^{k}_h(s_h,a_h)] - \iota^{k}_h(s_{k,h},a_{k,h})\bigr).
    \#
    Here $\{\cM_{t,m}\}_{(t,m)\in[K]\times[H]\times[2]}$ is a martingale adapted to the filtration $\{\cF_{t,m}\}_{(t,m)\in[K]\times[H]\times[2]}$. 
\end{lemma}

\begin{proof}
    See Lemma 4.2 in \citet{cai2020provably} for a detailed proof.
\end{proof}

\begin{lemma}[Mutual Information in Linear MDP]
    \label{linear_mi}
    It hold that 
    \$
    I_{t}(w_h;a_{t}, r_{t}, s_{t+1}|\cD) = \frac{1}{2} \log{(1+\phi(s_{t},a_{t})^\top\Lambda_t^{-1}\phi(s_{t},a_{t}))}.
    \$
\end{lemma}
\begin{proof}
    Let the prior be $w_h \sim \cN(0,\lambda\cdot I)$, then we have the following closed form posterior 
    \$
    w_h | \cD \sim  \cN(\widehat{w}_h,\Lambda_{t}^{-1}), 
    \$
    where 
    \$
    \widehat{w}_h &= \Lambda_h^{-1}(\sum_{k=1}^K \phi(s_{t},a_{t})\cdot(r_{t}+\widehat{V}_{h+1}(s_{t+1}))), \\
    \Lambda_h &= \sum_{k=1}^K \phi(s_{t},a_{t})\phi(s_{t},a_{t})^\top+\lambda\cdot I.
    \$
    Note that this is equivalent to the regularized least-square solution for linear MDPs~\citep{jin2021pessimism}.
    Then we have 
    \$
    I_{t}(w_h;a_{t},r_{t}, s_{t+1}|\cD) &= H(w_h|\cD) - H(w_h|\cD \cup \{(r_{t}, a_{t},s_{t+1})\}) \\
    &= \frac{1}{2} \log{\frac{\det(\Lambda_{t}^\dagger)}{\det(\Lambda_{t})}} \\
    &= \frac{1}{2} \log{\det(I+\Lambda_{t}^{-1/2}\phi(s_{t},a_{t})\phi(s_{t},a_{t})^\top \Lambda_{t}^{-1/2})}\\
    &= \frac{1}{2} \log{(1+\phi(s_{t},a_{t})^\top\Lambda_{t}^{-1}\phi(s_{t},a_{t}))}.\\
    \$
    where $\Lambda_{t}^\dagger = \Lambda_{t} + \phi(s_h,a_h)\phi(s_h,a_h)^\top$.
\end{proof}
\begin{lemma}
    \label{linear_uncertainty}
    Under linear MDP, we have 
    \$
    \PP_k \left(\left|Q_{t,w}(s,a)-\bar{Q}_{t,w}(s,a)\right|\leq \frac{\Gamma_t}{2}\sqrt{I_t(w_h;r_{t,a},s_{t+1,a})}, \forall h \in [H], s \in \cS, a\in \cA\right) \geq 1-\frac{\delta}{2} 
    \$
    With $\Gamma_t \equiv \Gamma = 2 c Hd \sqrt{\log{\frac{4dT}{\delta}}},$ where $c$ is an absolute constant and $\bar{Q}_{t,w}(s,a)=r_{h,w}(s,a)+\PP_h V_{h+1,w}(s,a)$. 

\end{lemma}
\begin{proof}
    
    Following a similar argument in Lemma 5.2 in \citet{jin2021pessimism}, We have with probability $1-\delta$,
     \$
     & \left|Q_{t,w}(s,a)-\bar{Q}_{t,w}(s,a)\right|\\
     \leq & \beta \sqrt{\phi(s,a)\Lambda_h^{-1}\phi(s,a)} \\
     \leq & \beta \sqrt{\log(1+\phi(s,a)\Lambda_h^{-1}\phi(s,a))\cdot \frac{\phi(s,a)\Lambda_h^{-1}\phi(s,a)}{\log(1+\phi(s,a)\Lambda_h^{-1}\phi(s,a))}} \\
     \leq & \beta \sqrt{2\log(1+\phi(s,a)\Lambda_h^{-1}\phi(s,a)) } \\
     \leq & 2\beta \sqrt{I_t(w_h;r_{t,a},s_{t+1,a}) },
    \$
    where $\beta = cHd \sqrt{\log{\frac{4dT}{\delta}}}$ and $\Lambda_h$ is defined as in Lemma~\ref{linear_mi}. The last inequality use the fact that $\phi(s,a)\Lambda_h^{-1}\phi(s,a)\leq 1$ and $2 \log(1+x) \geq x $ for $x\in[0,1]$. The last step follows from Lemma~\ref{linear_mi}.

\end{proof}


\clearpage
\section{Experiments on Bernoulli Bandits}
\label{appendix: multi-arm bandit}

\begin{figure}[h]
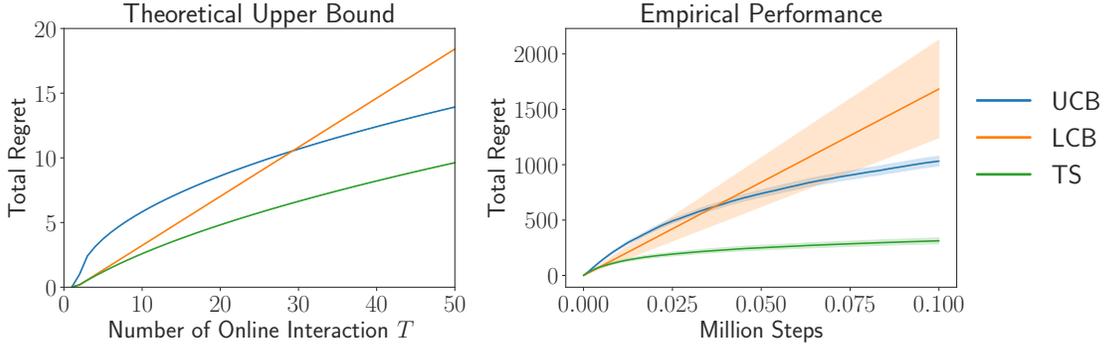

    \centering
    \subfigure{\includegraphics[scale=0.4]{theory_predict.pdf}}
    \subfigure{\includegraphics[scale=0.4]{bandit.pdf}}
    \caption{Theoretical upper bounds in Theorem~\ref{theorem:1} and experiment result on the Bernoulli bandit. 
    }
    \label{fig: appendix theory}
\end{figure}

\begin{figure}[h]
    \centering
    \subfigure[Hard switch of LCB to UCB]{\includegraphics[width=0.45\linewidth]{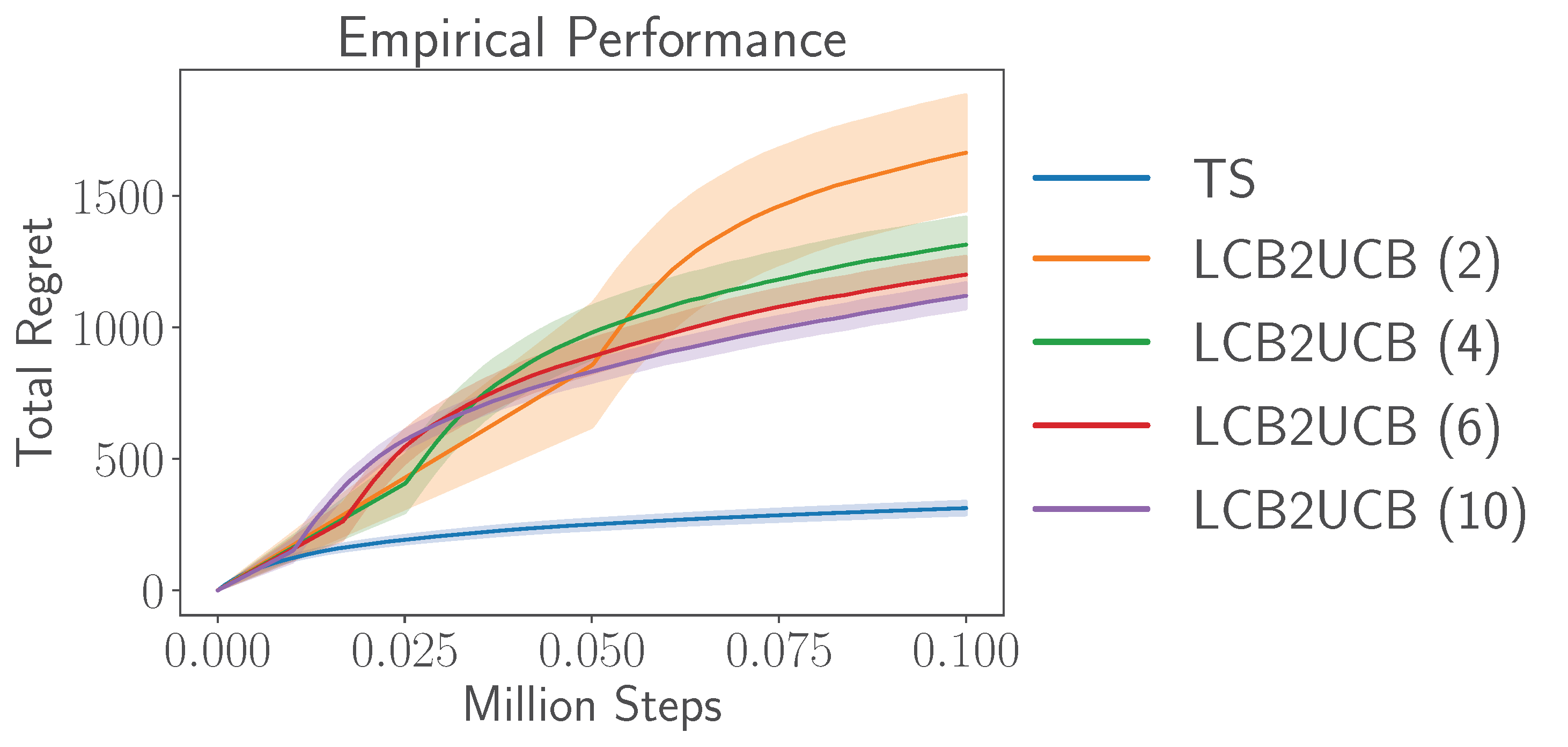}}\subfigure[Soft switch of LCB to UCB]{\includegraphics[width=0.45\linewidth]{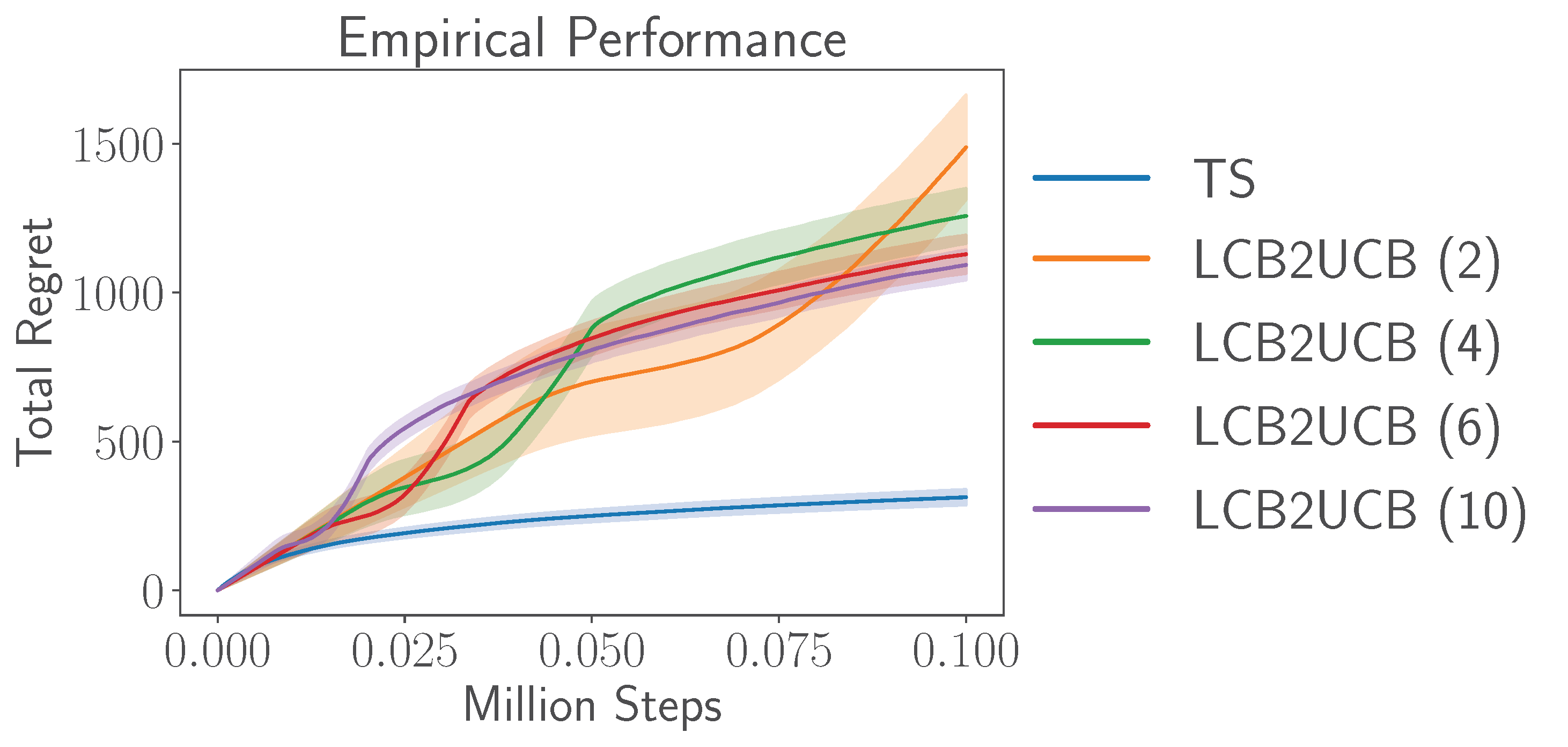}}
    \caption{Performance of different switch schemes from LCB to UCB on the Bernoulli bandit. 
    It incurs a large regret to switch from pessimism to optimism regardless of interpolation schemes.
    $x$ in \texttt{LCB2UCB}~($x$) represents the switch parameter.
    }
    \label{fig: switch}
\end{figure}


\paragraph{Experiment Setup.} We consider a didactic Bernoulli bandit, where each arm pull yielded a binary reward based on probability $p_i$. the parameters $p_n$ are i.i.d. drawn from a Beta distribution for all arms $a_i$. The objective was to identify the arm associated with the maximum probability, $p^* = \max\{p_i\}_{i=1}^{n}$.
We use a uniform policy to collect the offline data. We use 10 arms, 1000 offline data points and the online phase lasts for 100000 steps.

\paragraph{Baseline Setup.} We use both offline and online data to calculate the statistics (i.e., mean of the reward $\widehat{\mu}$ and the corresponding confidence bound $\widehat{\sigma}$). Then Upper Confidence Bound (UCB) policy, $\pi_{\rm UCB}$ select arms by $\argmax_i \widehat{\mu}(a_i)+\widehat{\sigma}(a_i)$, while the Lower Confidence Bound (LCB) policy, $\pi_{\rm LCB}$ select arms by $\argmax_i \widehat{\mu}(a_i)-\widehat{\sigma}(a_i)$.

We also consider two switching scheme that switches from LCB to UCB. The soft switch interpolates the confidence weight by $k_t= \min \{At/T-1, 1\}$ where $A$ is a parameter. Then it select the arm by $\pi^t_{\rm soft} = \argmax_i \widehat{\mu}(a_i)+k_t\cdot\widehat{\sigma}(a_i)$. The hard switch interpolates the confidence weight by $k'_t= 2*\mathbbm{1}\{t\geq T/B\}-1$ where $B$ is the parameter. Then it select the arm by $\pi^t_{\rm hard} = \argmax_i \widehat{\mu}(a_i)+k'_t\cdot\widehat{\sigma}(a_i)$.


\paragraph{Results.} The results are shown in Figure 4 and Figure 5, respectively.
We can see in Figure~\ref{fig: appendix theory} that UCB performs badly during the initial stage (i.e., the regret increases quickly), while LCB suffers from a linear regret in the long run. 
Bayesian agents, on the contrary, enjoys guarantees from the both world, which performs well at initial stage while achieving a sublinear regret. From Figure~\ref{fig: switch}, we can see that naively switches from LCB to UCB leads to suboptimal performance regardless of the interpolation scheme. 
It suffers from a sudden increase in regret the first time pessimism is switched to optimism.
On the contrary, Bayes-based methods enjoys a smooth regret curve.




\clearpage
\section{Complete Experimental Results}
\label{appendix: complete exp}

The results in Figure~\ref{appendix fig: TD3 result} show TD3+BC exhibits safe but slow performance improvement, resulting in worse asymptotic performance.
On the other hand, TD3 suffers from initial performance degradation, especially in narrow distribution datasets~(e.g., expert datasets).
Differently, BOORL attains a fast performance improvement with a smaller regret.
Due to the offline bootstrap, the initial performance in the online phase between BOORL and baselines exits a small difference, while it does not change the conclusion.

\begin{figure}[h]
    \centering\subfigure{\includegraphics[scale=0.22]{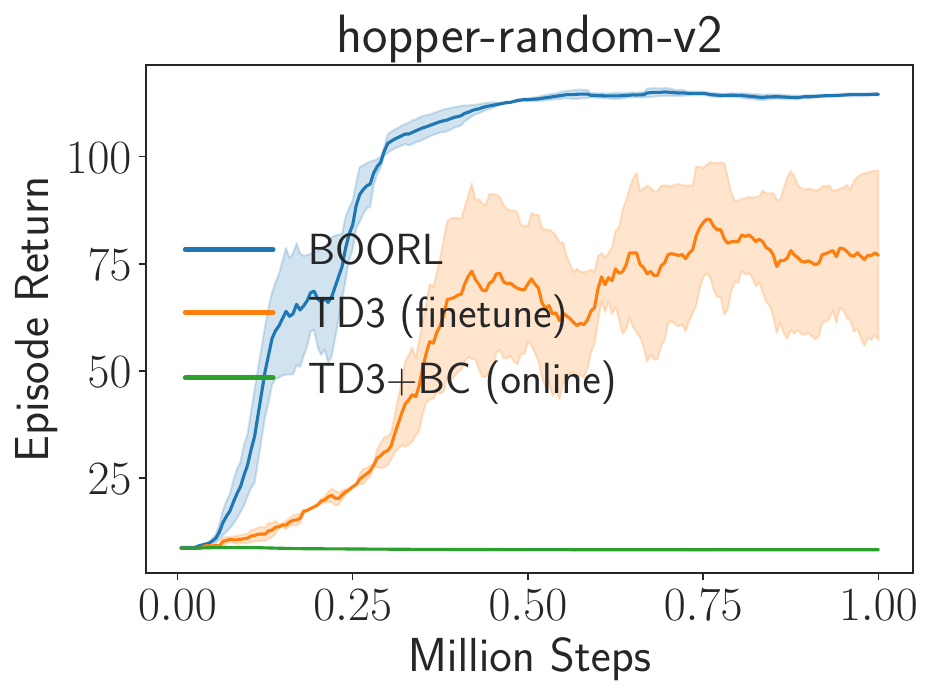}}\subfigure{\includegraphics[scale=0.22]{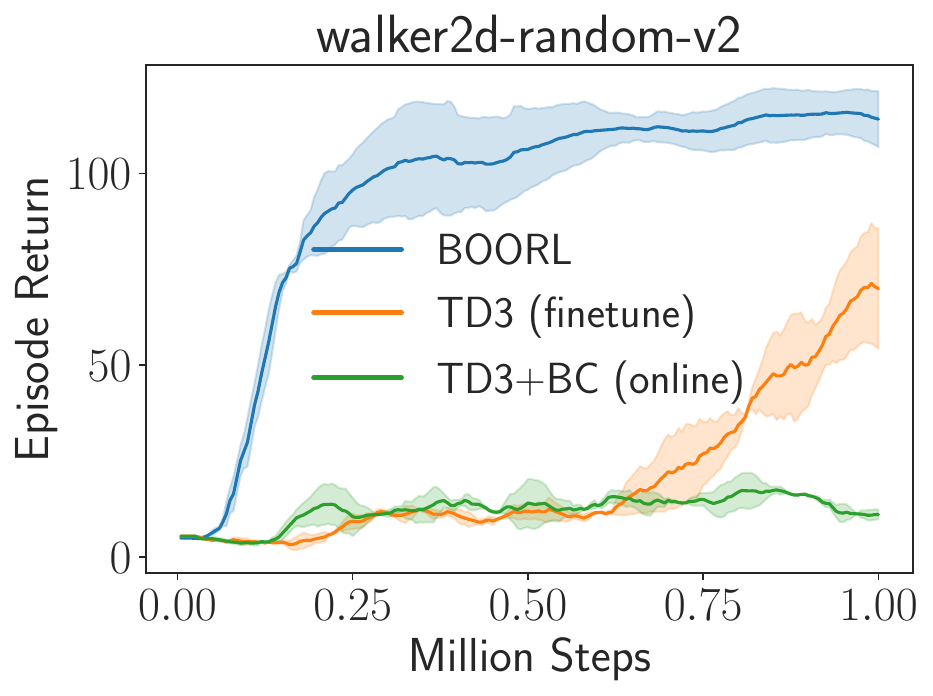}}\subfigure{\includegraphics[scale=0.22]{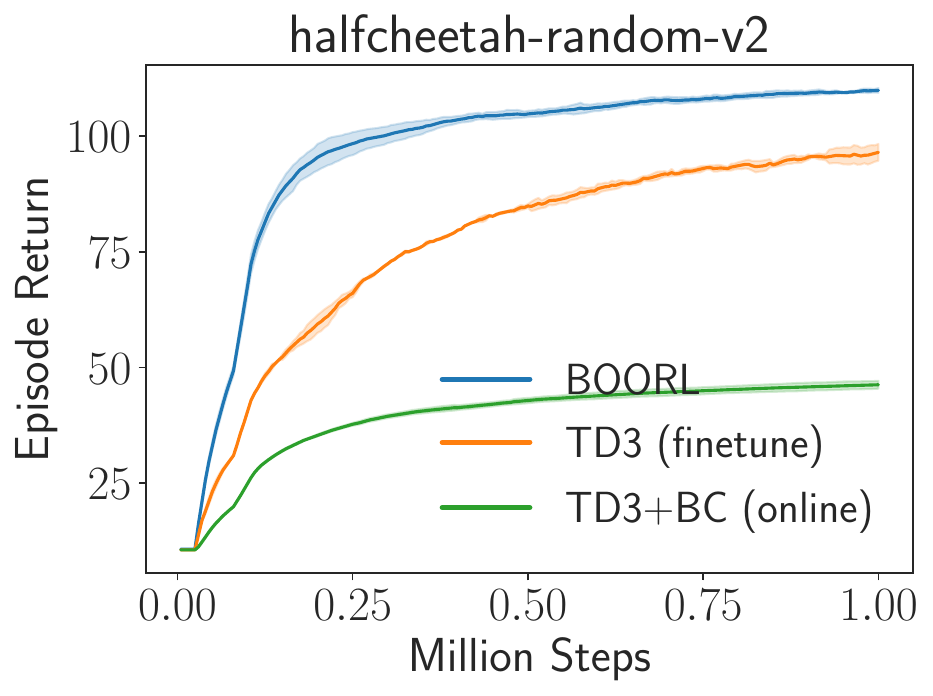}}\subfigure{\includegraphics[scale=0.22]{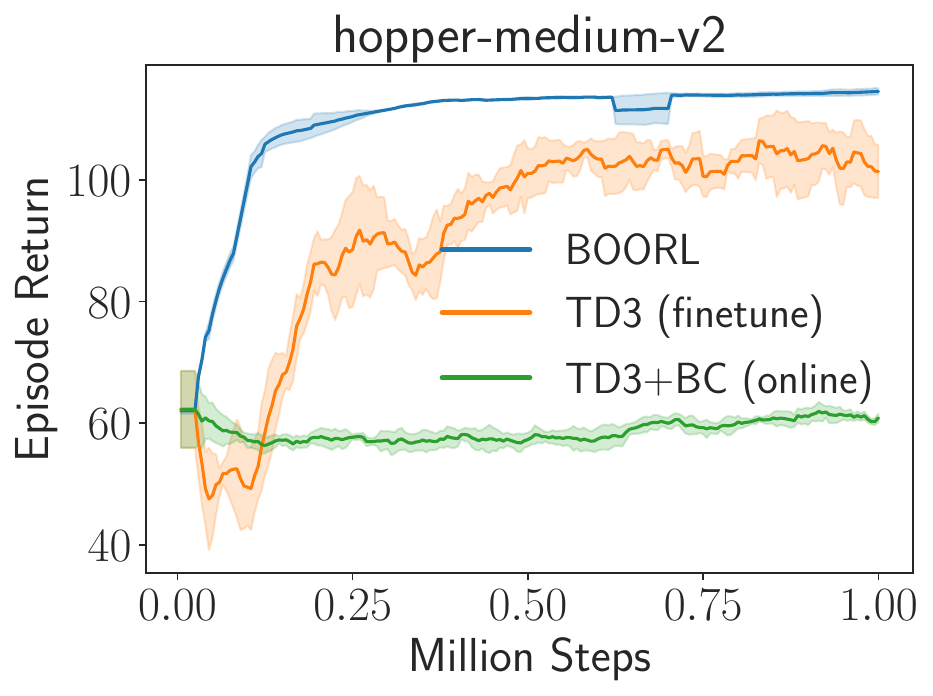}}
     \subfigure{\includegraphics[scale=0.22]{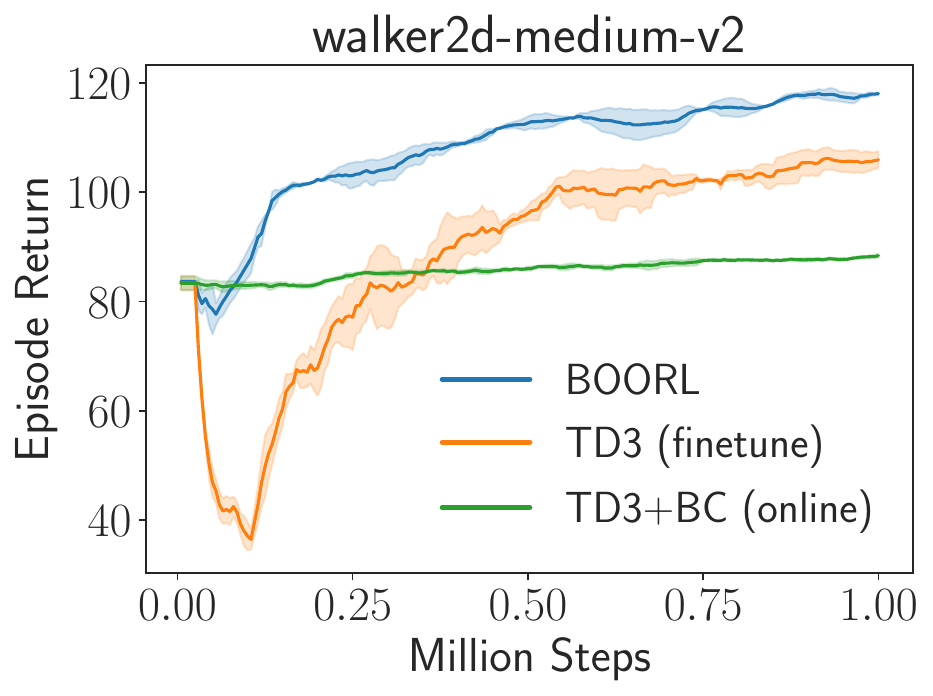}}\subfigure{\includegraphics[scale=0.22]{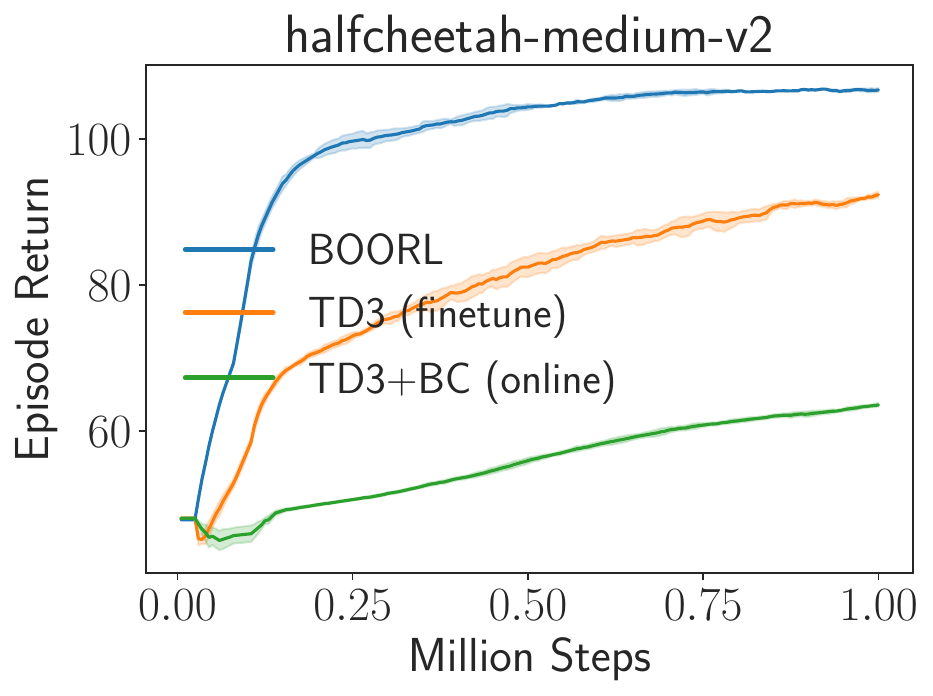}}
    \subfigure{\includegraphics[scale=0.22]{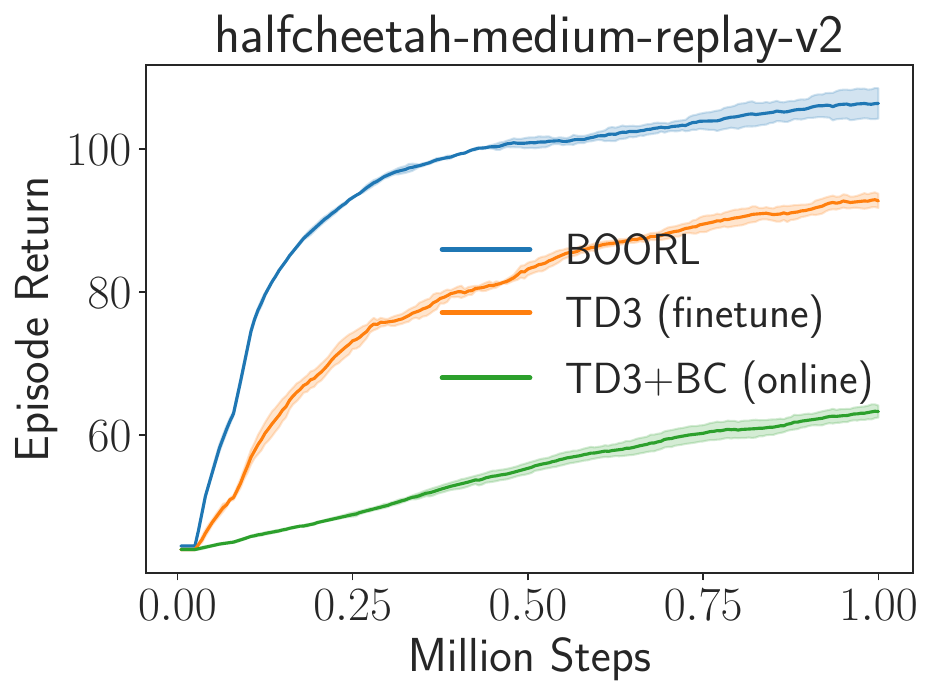}}\subfigure{\includegraphics[scale=0.22]{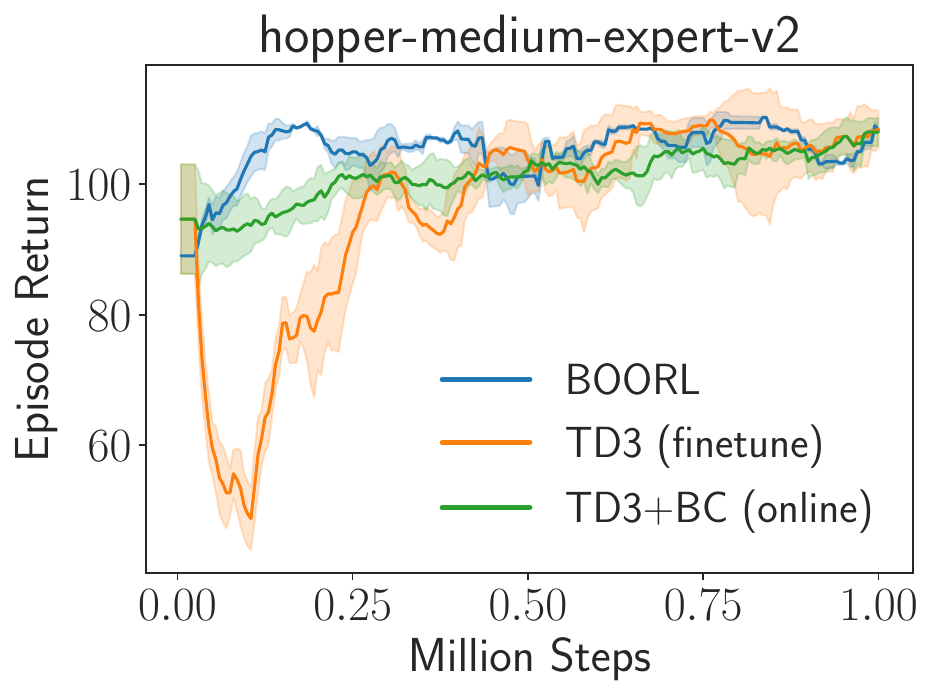}}
    \subfigure{\includegraphics[scale=0.22]{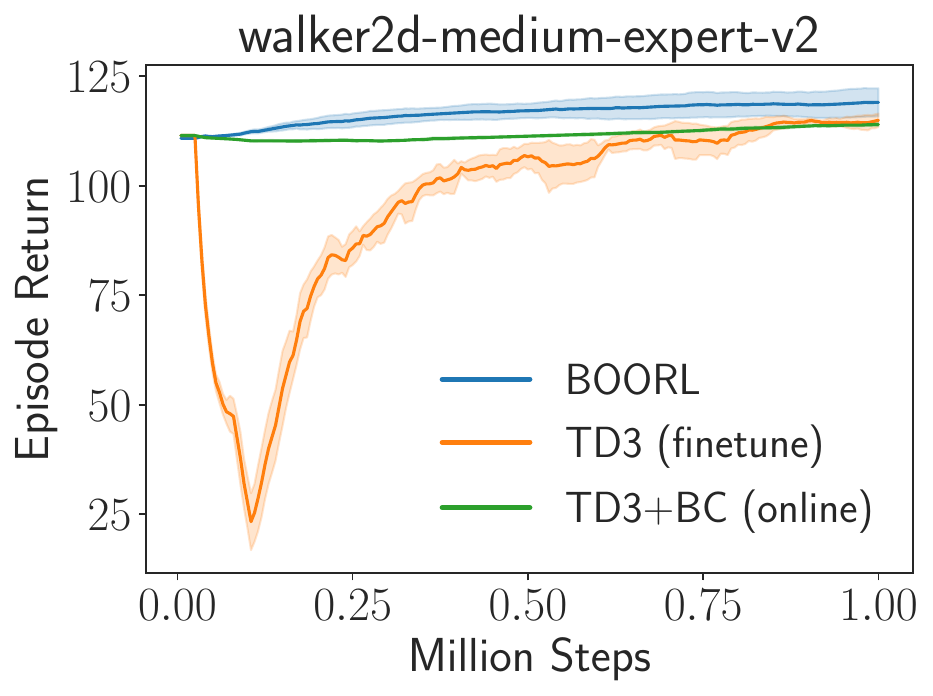}}
    \subfigure{\includegraphics[scale=0.22]{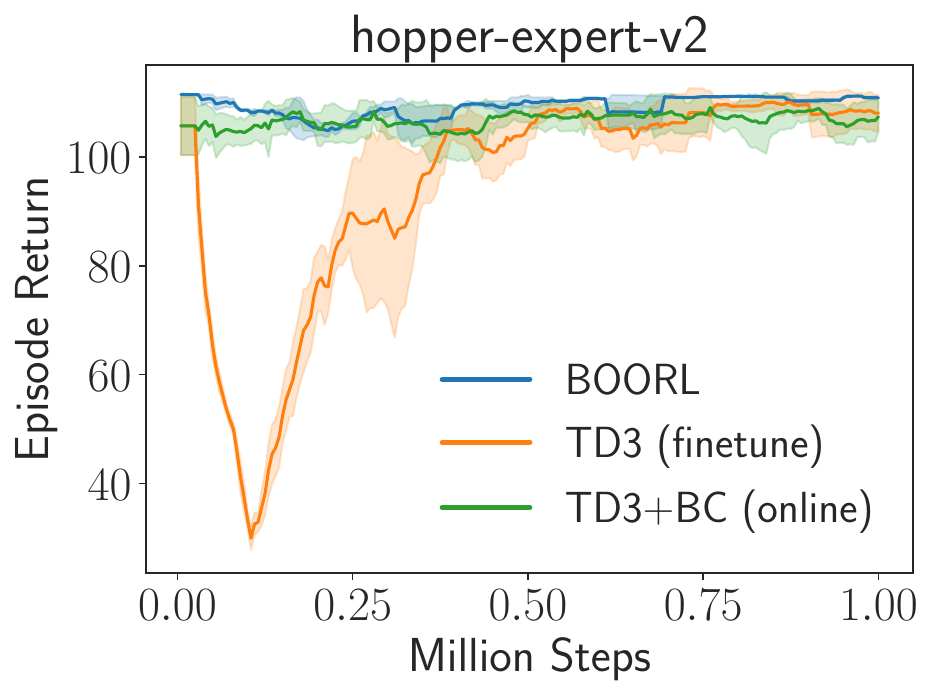}}\subfigure{\includegraphics[scale=0.22]{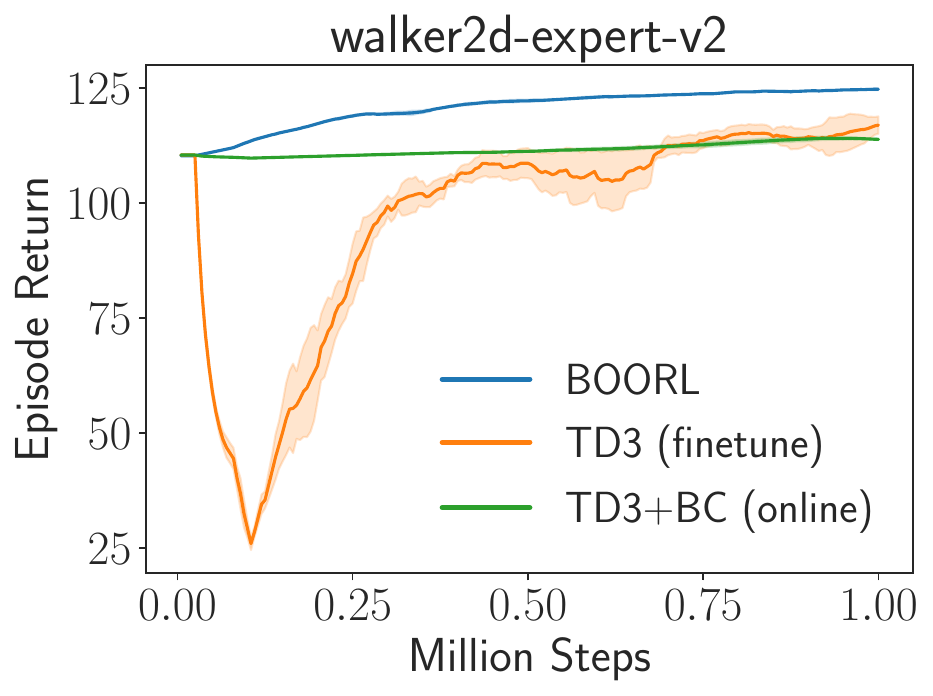}}\subfigure{\includegraphics[scale=0.22]{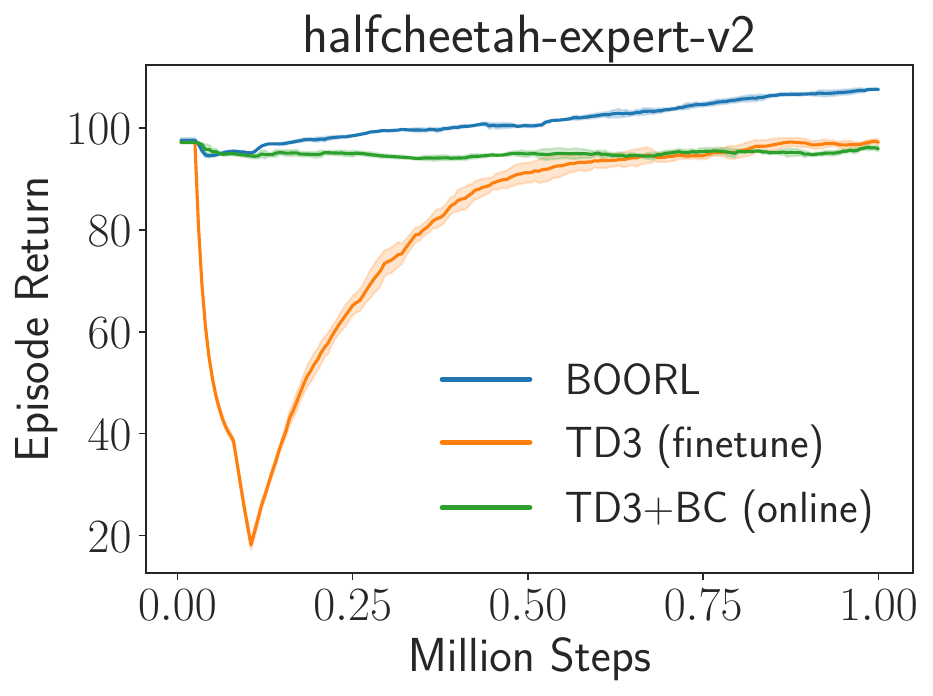}}
    \caption{Comparison between BOORL and baselines in the finetune phase.
    We adopt datasets of various quality for offline training and then load same pre-trained weight for online learning.
    We adopt normalized score metric averaged with five random seeds.    
    }
    \label{appendix fig: TD3 result}
\end{figure}

\clearpage
\section{Comparison with PEX}
\label{appendix: pex}

\begin{figure}[h]
    \centering\subfigure{\includegraphics[scale=0.22]{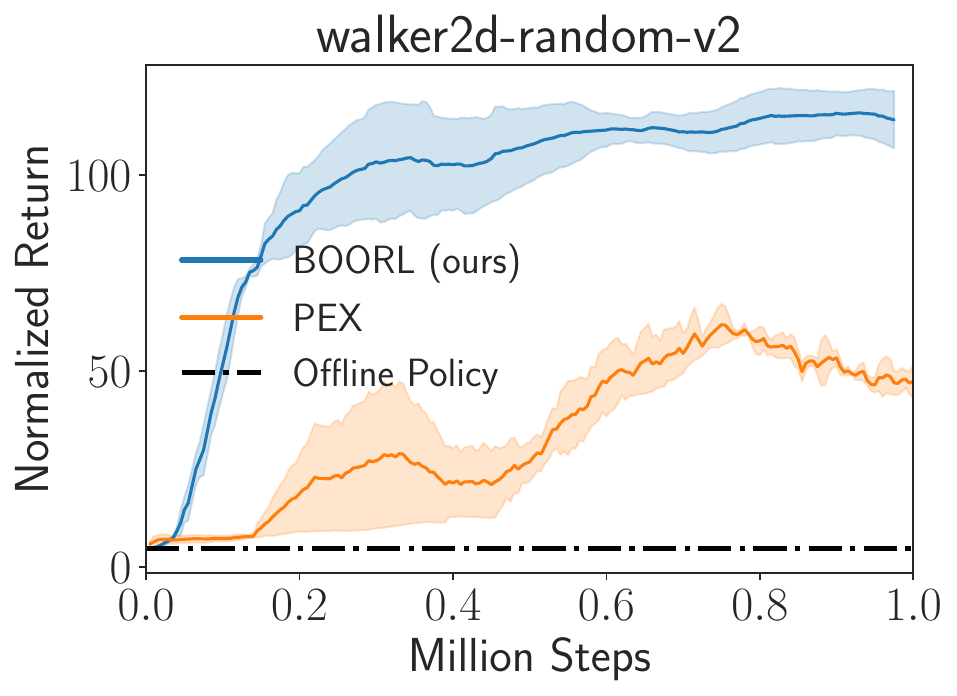}}\subfigure{\includegraphics[scale=0.22]{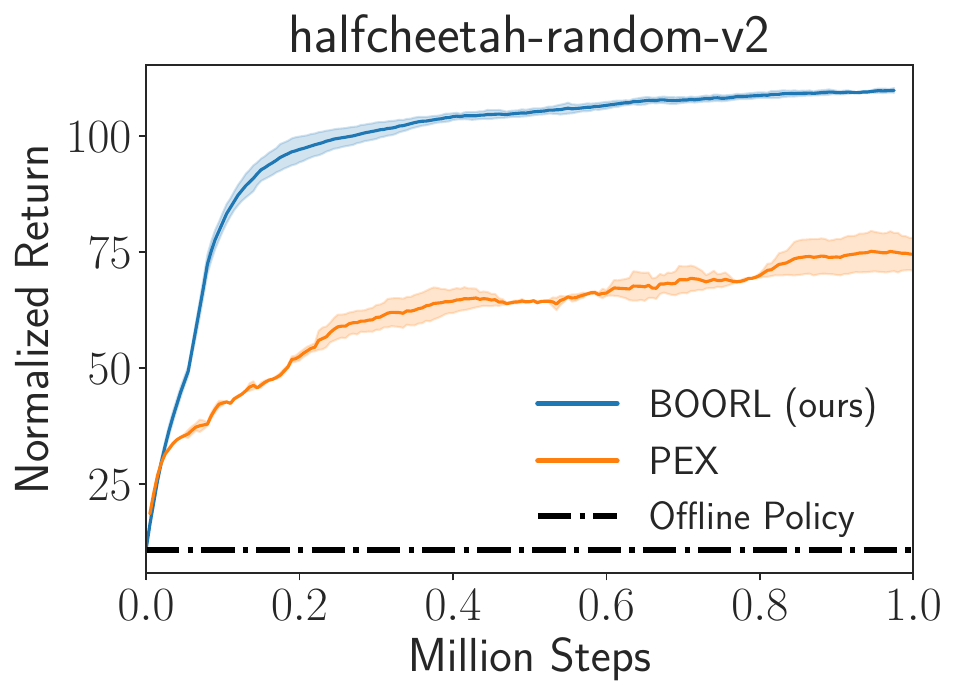}}\subfigure{\includegraphics[scale=0.22]{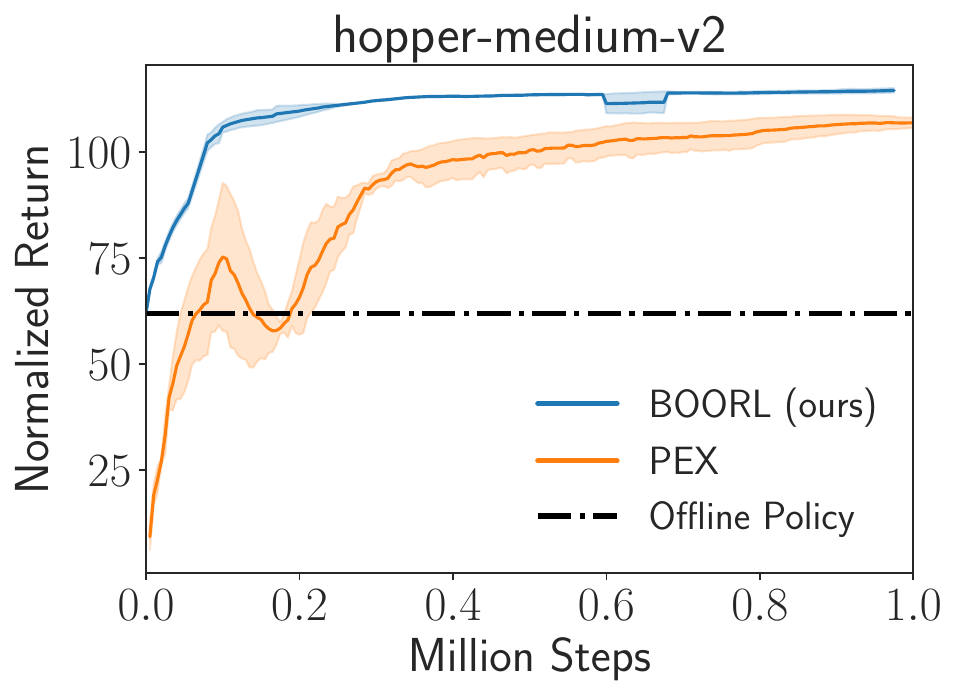}}\subfigure{\includegraphics[scale=0.22]{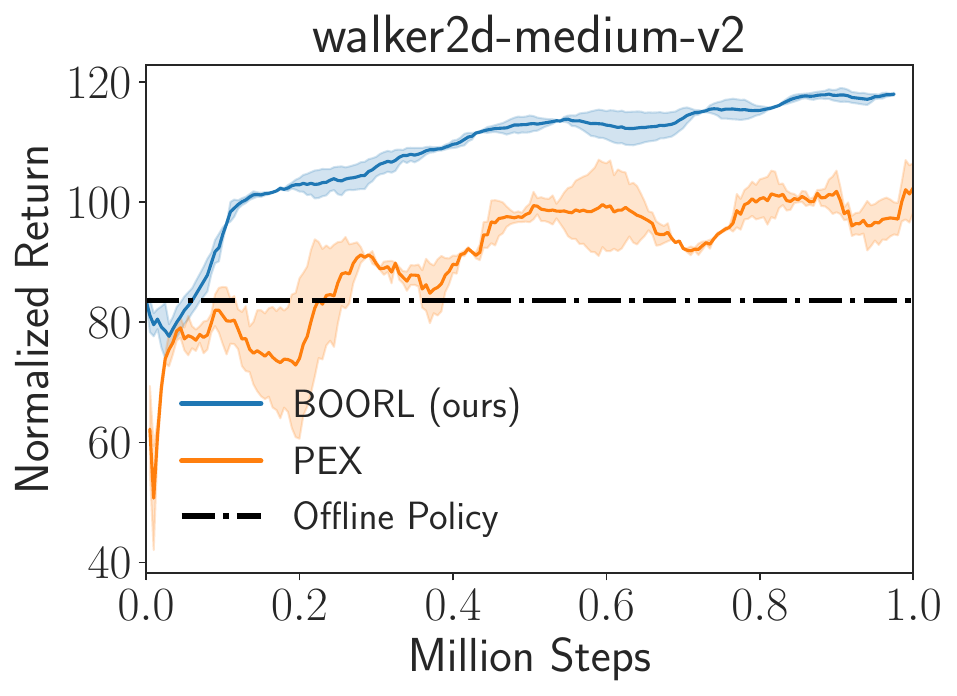}}
    \subfigure{\includegraphics[scale=0.22]{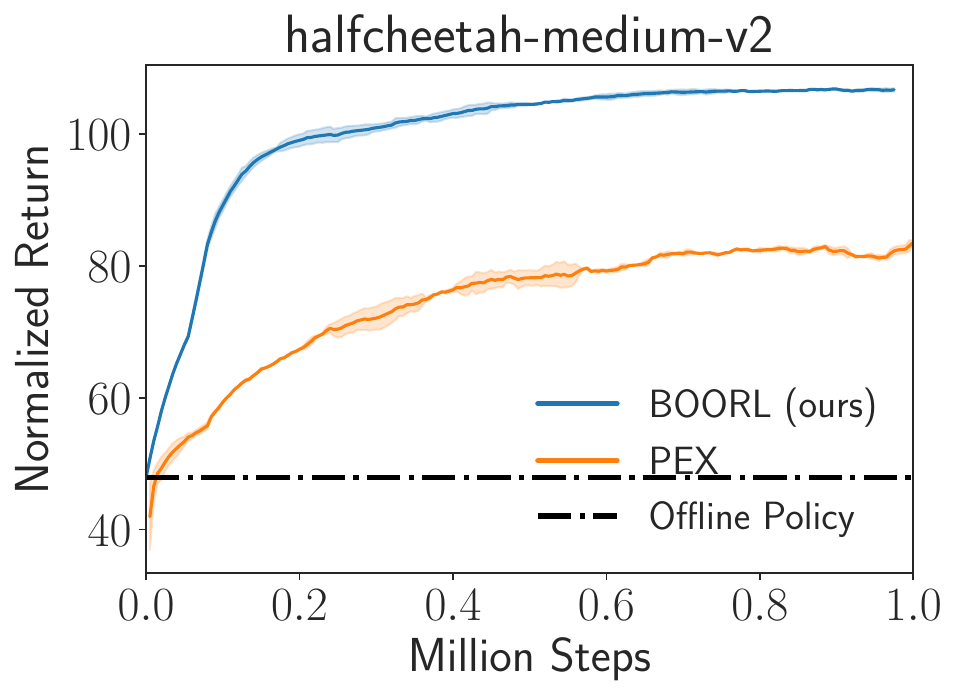}}\subfigure{\includegraphics[scale=0.22]{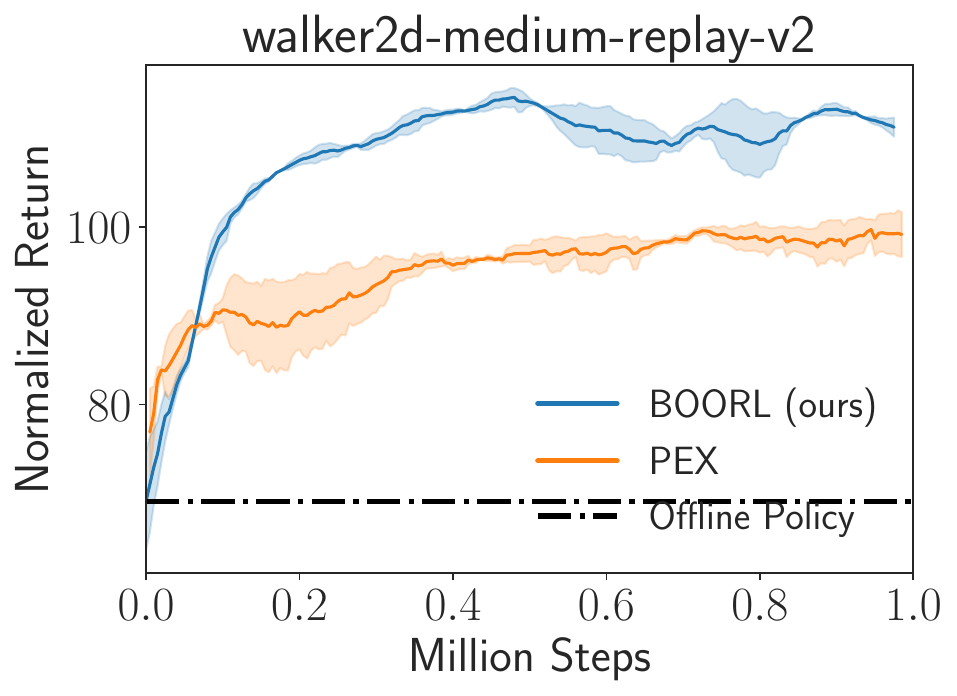}}\subfigure{\includegraphics[scale=0.22]{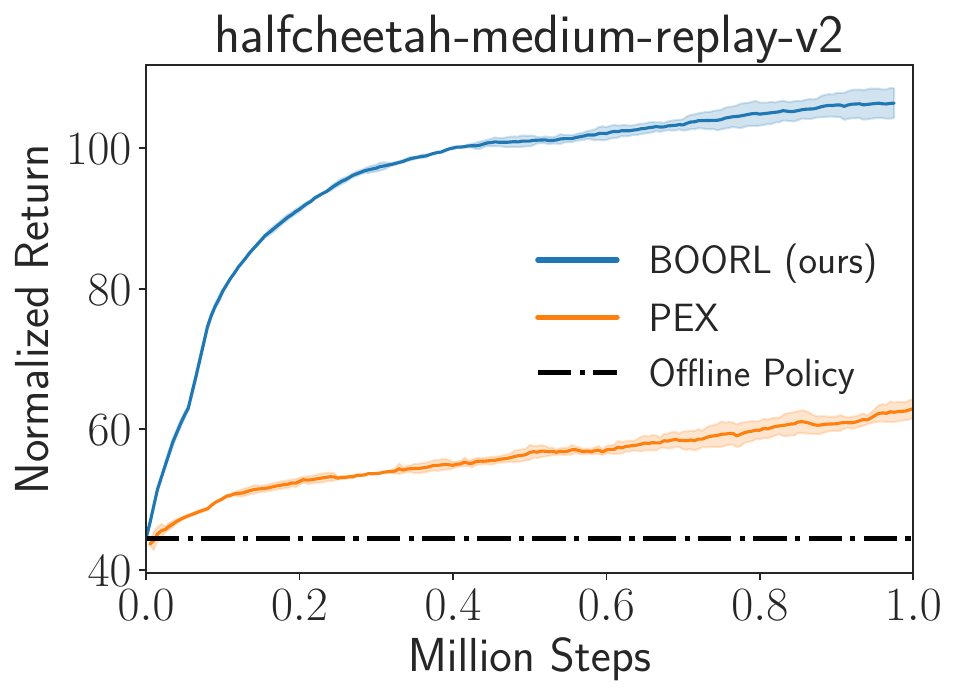}}\subfigure{\includegraphics[scale=0.22]{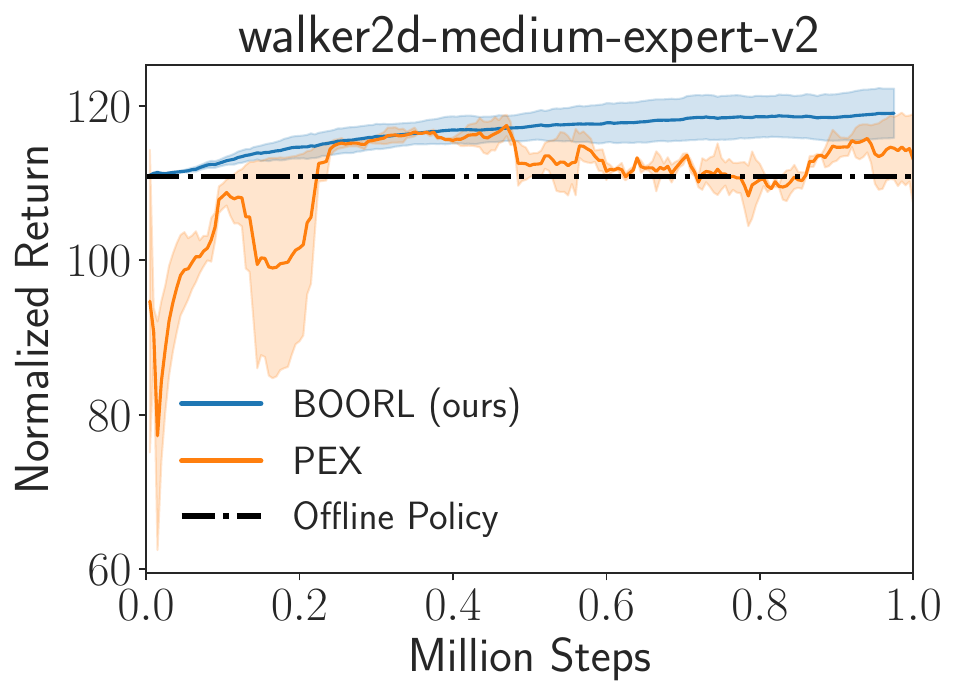}}
    \subfigure{\includegraphics[scale=0.22]{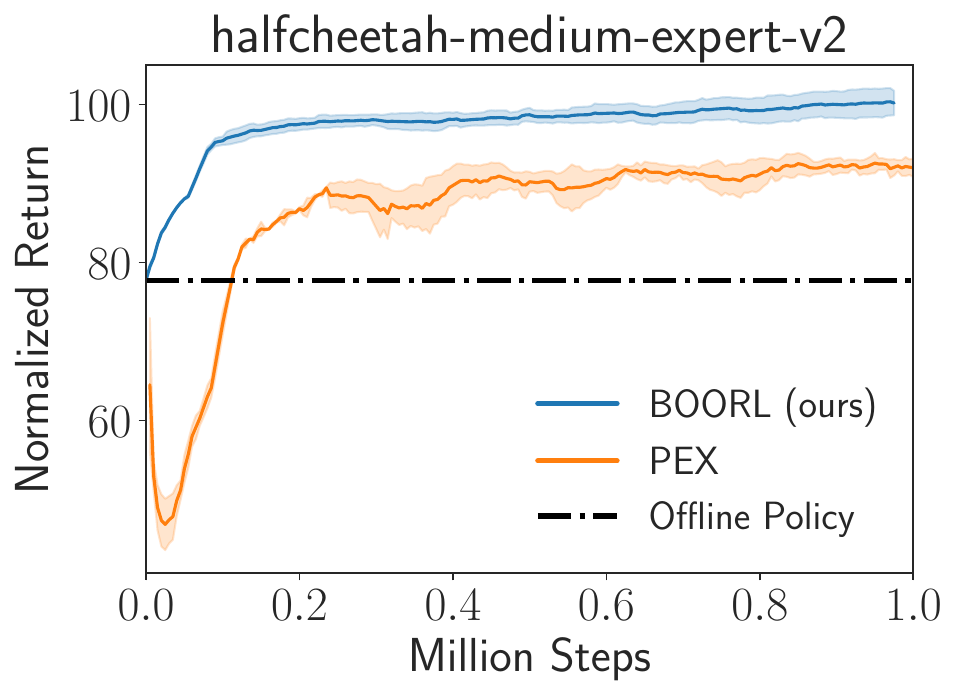}}\subfigure{\includegraphics[scale=0.22]{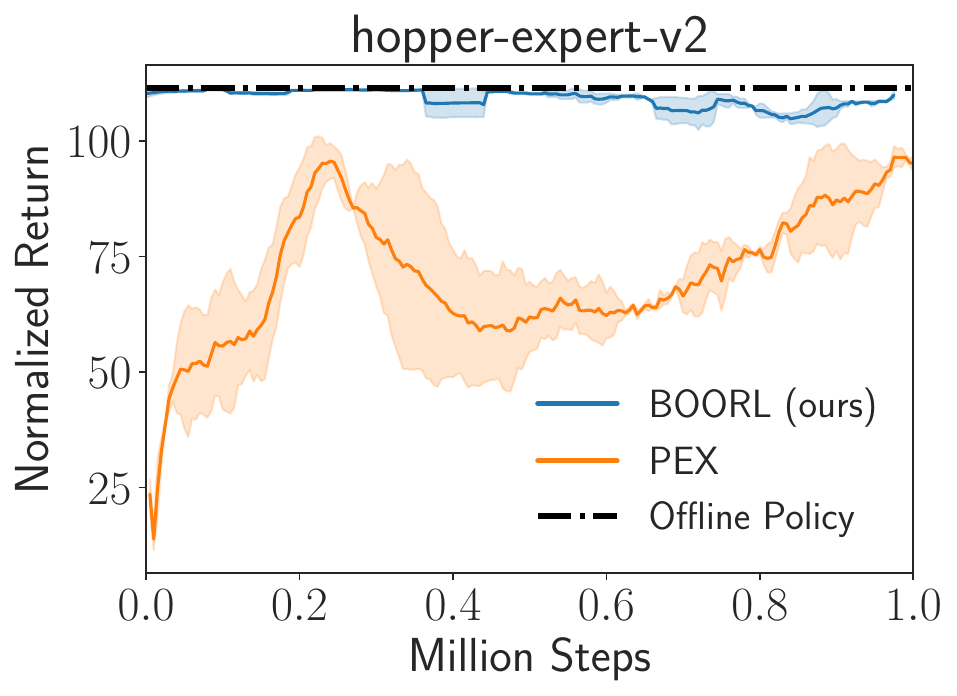}}\subfigure{\includegraphics[scale=0.22]{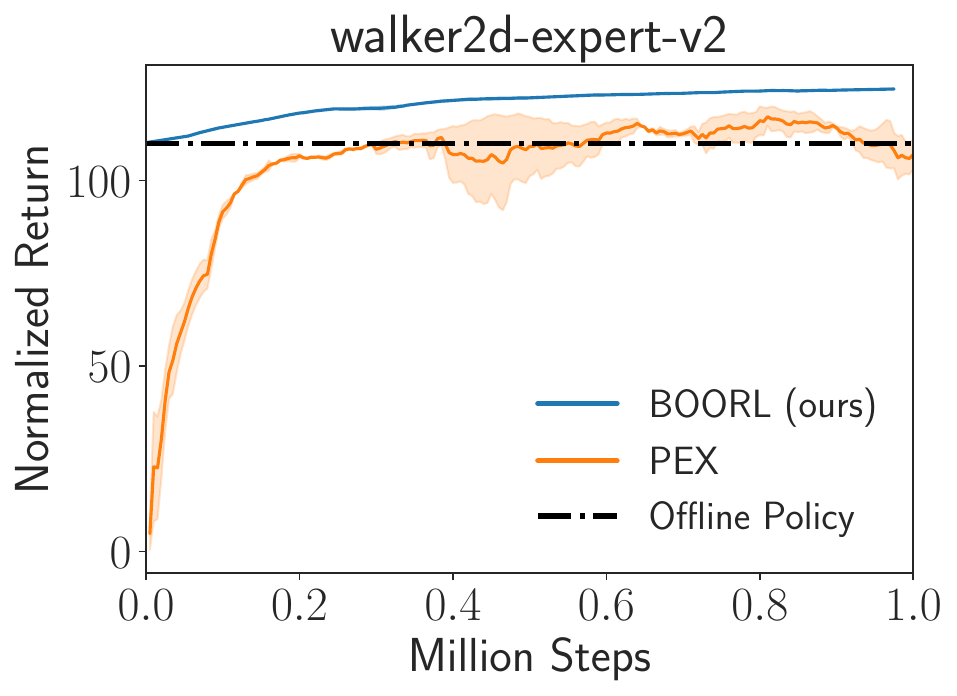}}\subfigure{\includegraphics[scale=0.22]{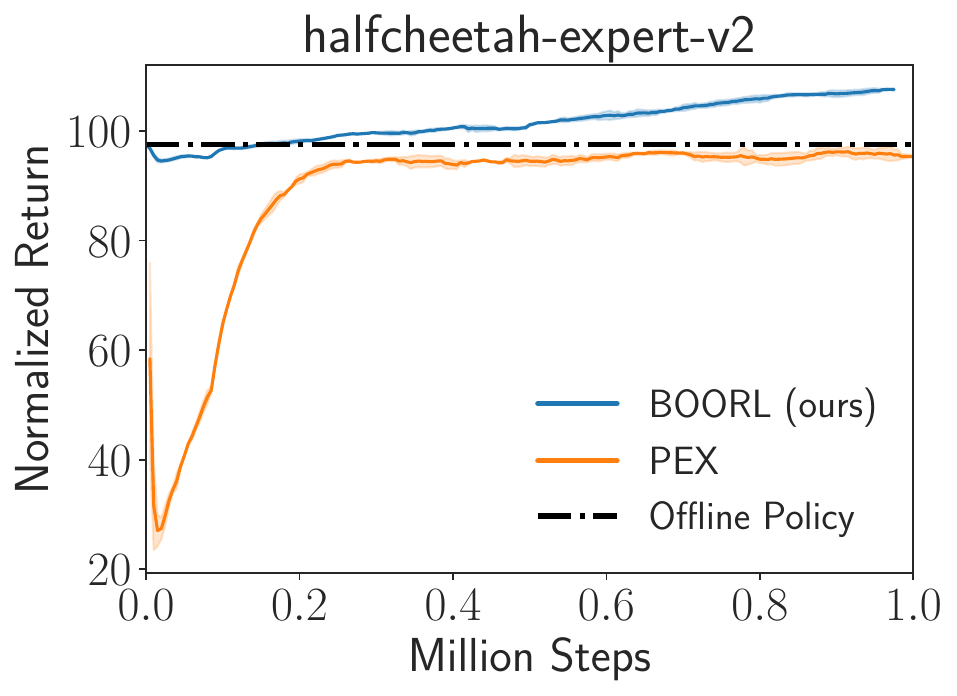}}
    \caption{Comparison between BOORL and PEX in the finetune phase.
    We adopt datasets of various quality for offline training and then load same pre-trained weight for online learning.
    We adopt normalized score metric averaged with five random seeds.    
    }
    \label{appendix fig: pex result}
\end{figure}

\clearpage
\section{Additional Experiments}
\label{appendix: additional exp}

 \subsection{Ablation for BOORL}
 We aim to understand the behavior of BOORL by performing ablation studies. (1) We store the offline and online data into the same buffer for uniform sampling, named BOORL~(Uniform Buffer). (2) We set the Ensemble Number to 1 to investigate the effect of the Thompson Sampling, named BOORL~(Ensemble Num=1).
 The experimental results in Figure~\ref{fig: ablation module}  show that each module is essential to the superior performance of our algorithm. In the ablation studies, we use TD3+BC and TD3 as the backbone of offline and online algorithms.

 \begin{figure}[h]
     \centering\subfigure{\includegraphics[scale=0.22]{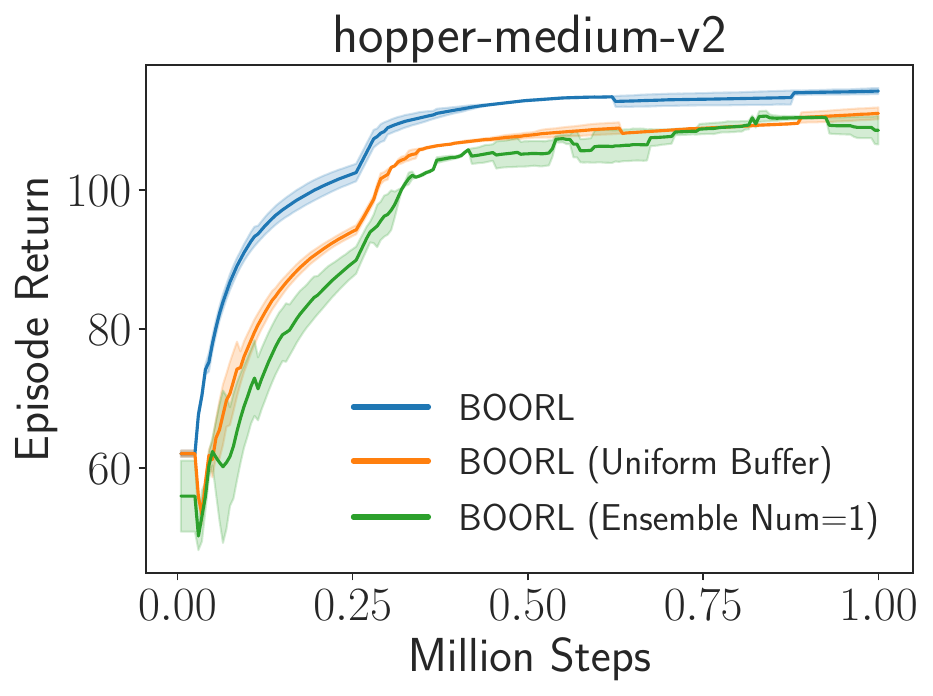}}\subfigure{\includegraphics[scale=0.22]{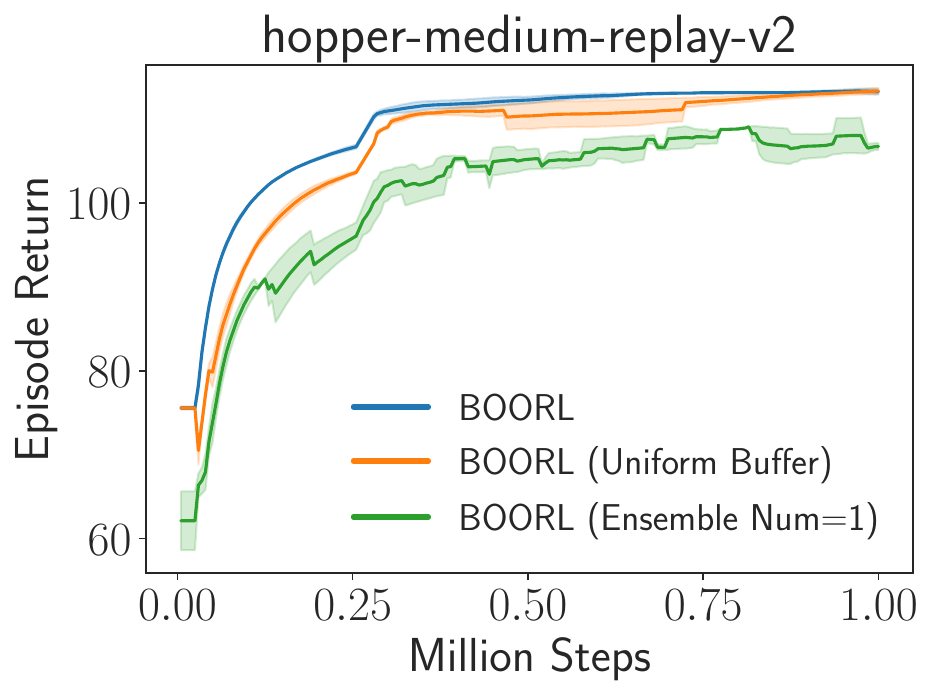}} \subfigure{\includegraphics[scale=0.22]{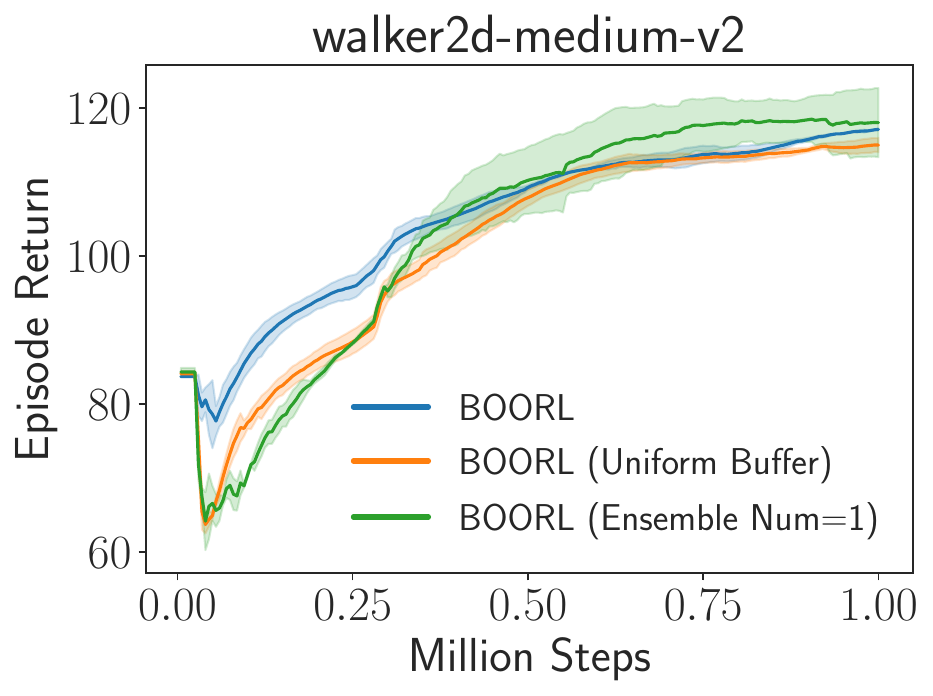}}\subfigure{\includegraphics[scale=0.22]{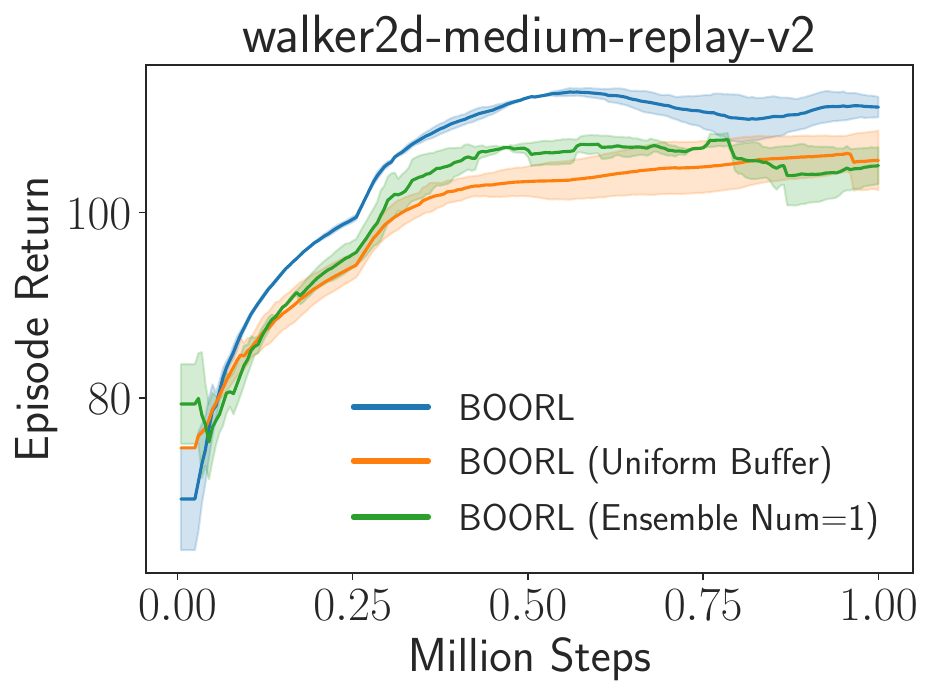}}
    
     \subfigure{\includegraphics[scale=0.21]{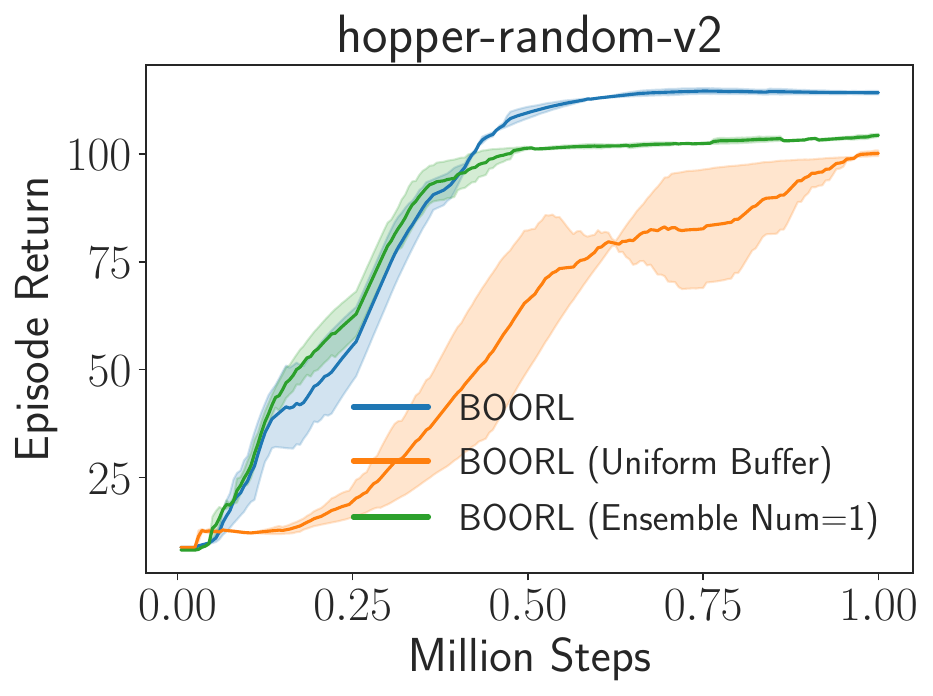}}
     \subfigure{\includegraphics[scale=0.21]{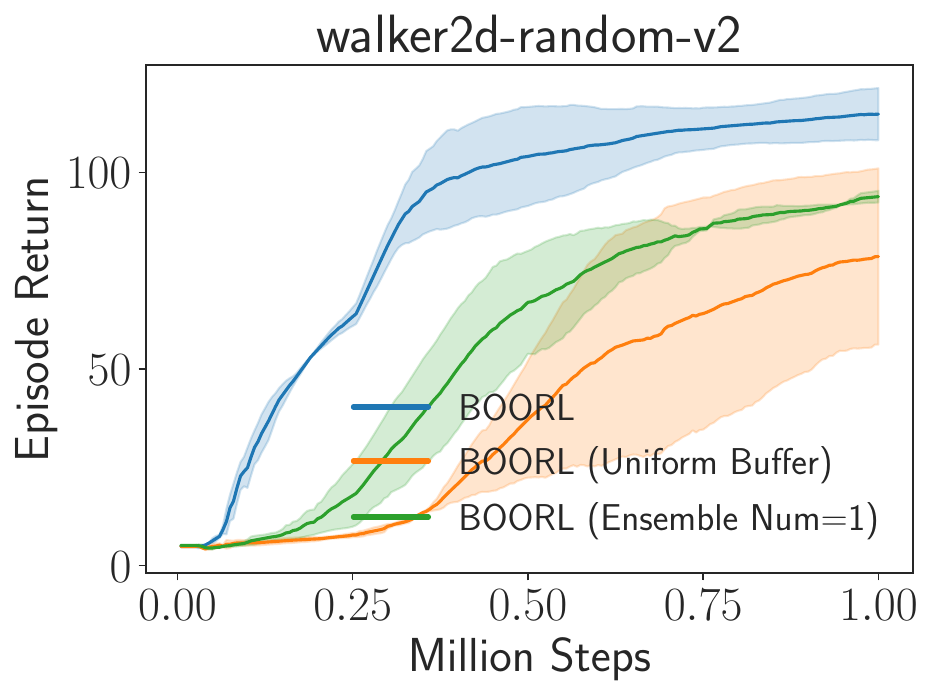}}
     \subfigure{\includegraphics[scale=0.22]{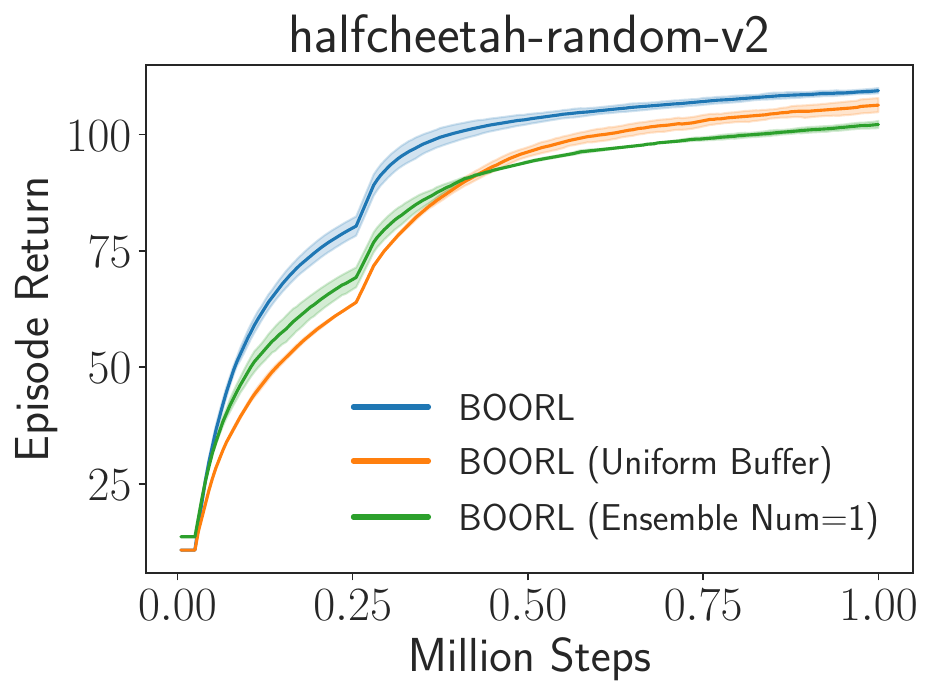}}
     \subfigure{\includegraphics[scale=0.22]{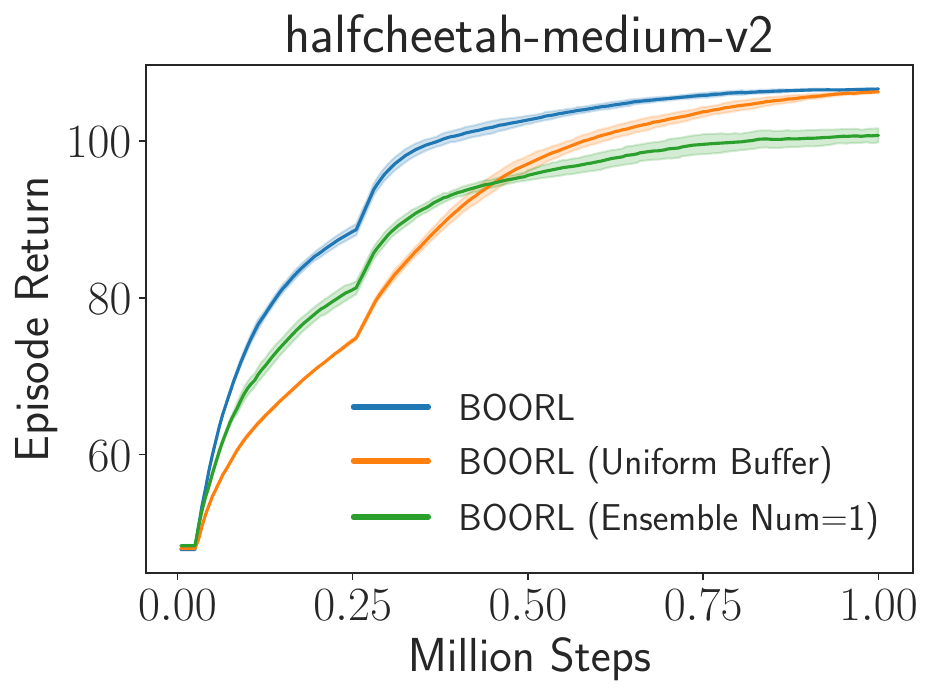}}
     \caption{Module ablation study of BOORL.}
     \label{fig: ablation module}
 \end{figure}

\subsection{Ablation for mask ratio}
In addition, we conduct ablation studies for mask ratio $p$.
The experimental results in Figure~\ref{fig: ablation mask ratio} show that the performance of our algorithm is robust to the changes of $p$.
Similar results are also found in~\cite{osband2016deep}.
Therefore, we select the uniform parameter $p=0.9$ in all experiments.

\begin{figure}[h]
    \centering
    \subfigure{\includegraphics[scale=0.22]{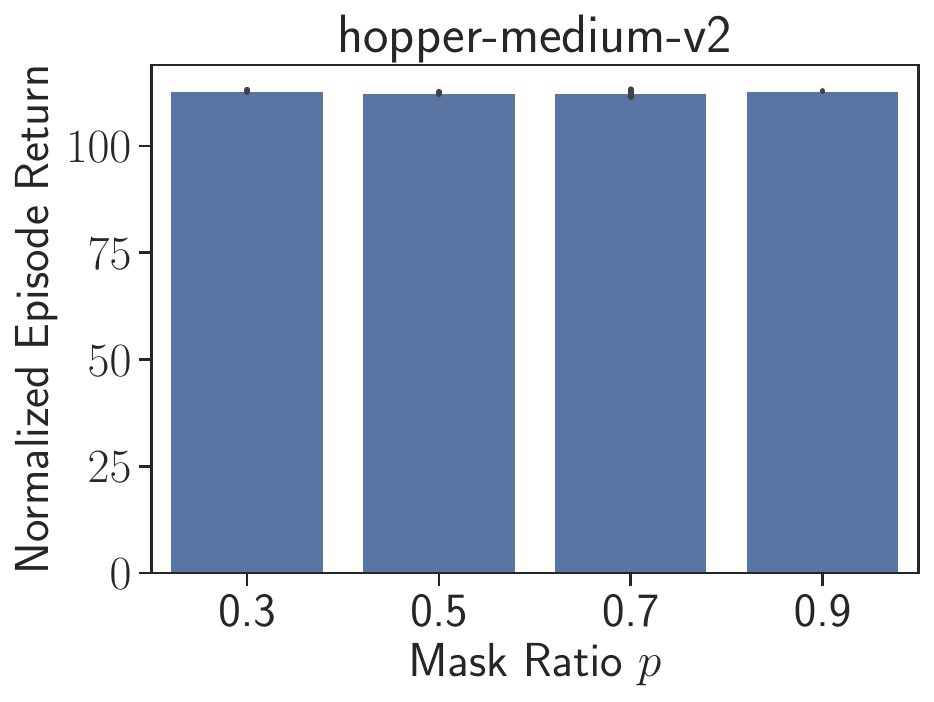}}\subfigure{\includegraphics[scale=0.22]{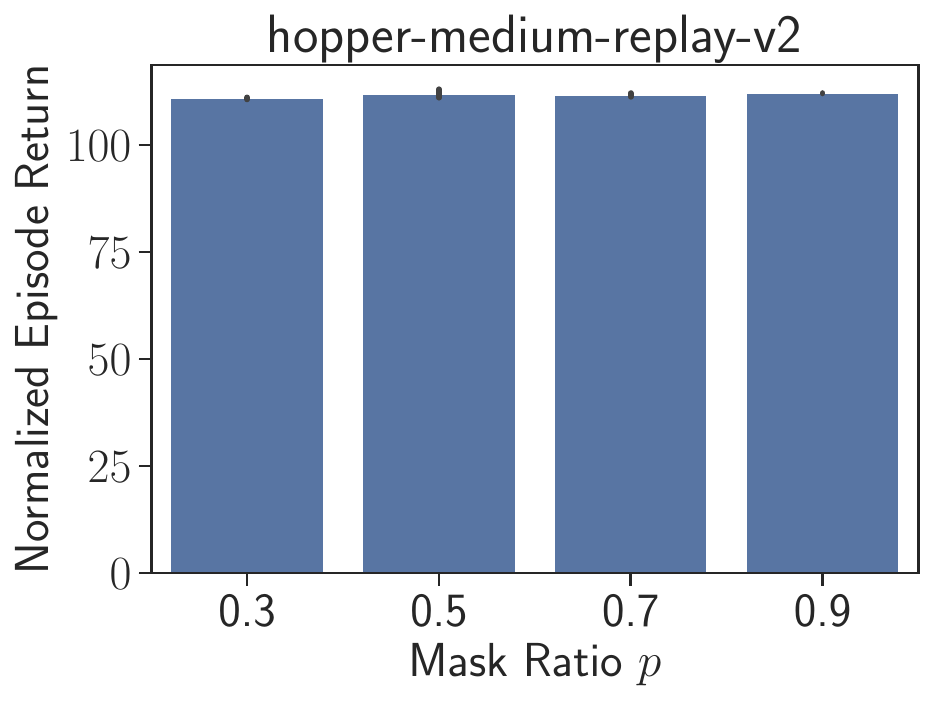}}\subfigure{\includegraphics[scale=0.22]{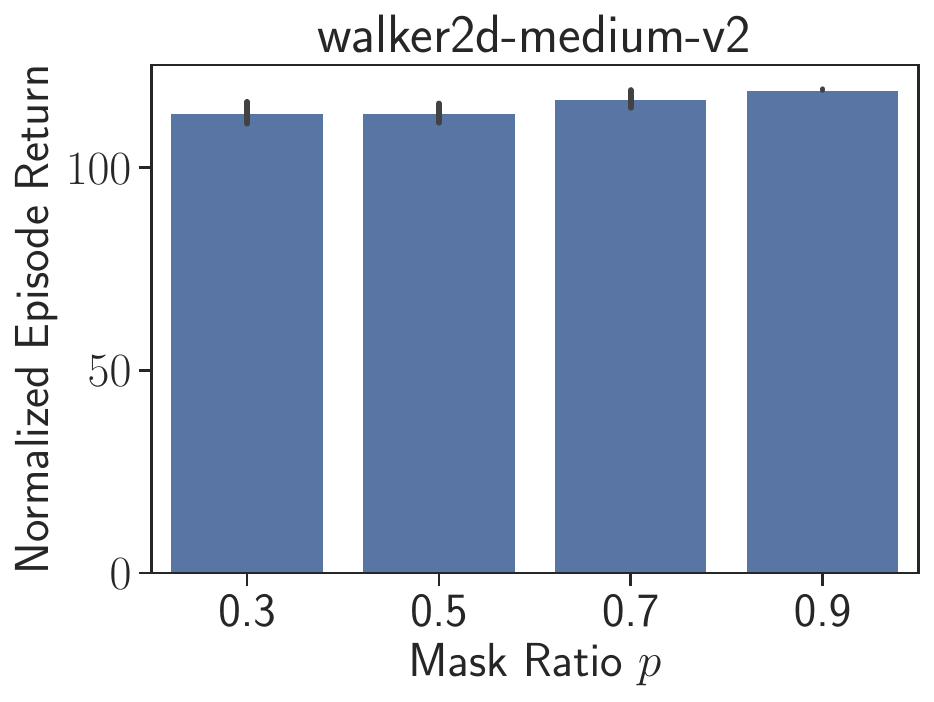}}\subfigure{\includegraphics[scale=0.22]{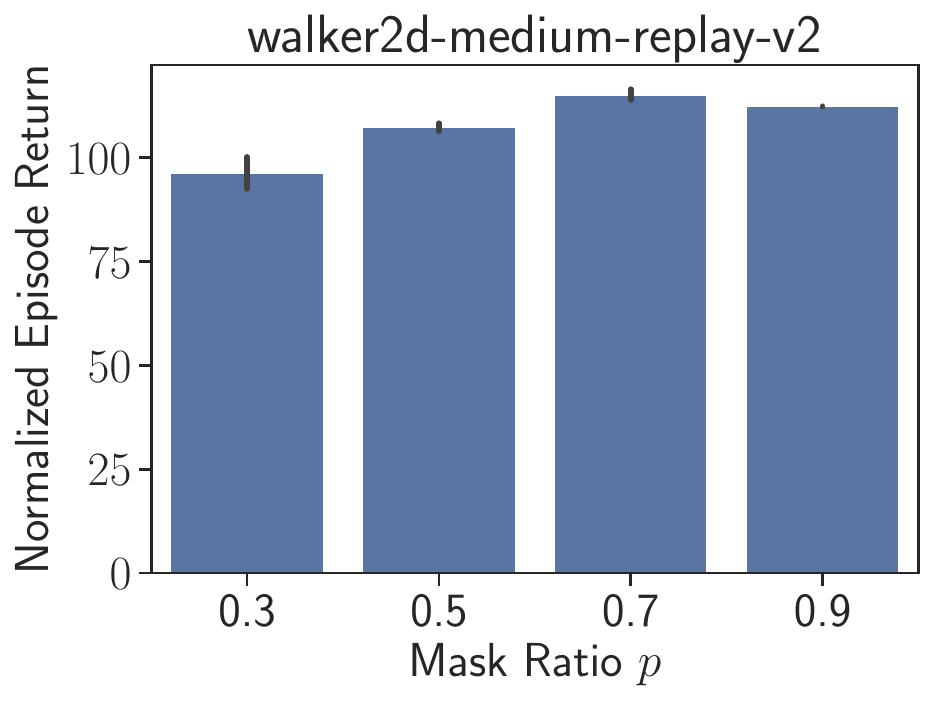}}
    
    \caption{The performance comparison between various mask ratios $p$.}
    \label{fig: ablation mask ratio}
\end{figure}




\subsection{Ablation for computational overhead}
We aim to provide a comparison between the computational overhead of using an ensemble versus not using an ensemble in our approach.
Specifically, we train our method with various ensemble sizes N on the same computational device (GeForce RTX 3090 GPU). 
The time required to complete 1M training is shown in Table~\ref{tab: ablation comput}. 
Since we adopt the multi-head structure and share part of the network, the computational overhead does not increase significantly as the number of N increases.

\begin{table}[H]
    \centering
    \begin{tabular}{c|c|c|c|c}
        \toprule
        Ensemble Size &  1 & 5 & 10 & 20 \\
        \midrule
        Computational Overhead & 2.5 h & 2.7 h & 2.9 h & 3.2 h \\
        \bottomrule
    \end{tabular}
    \caption{Ablation results for computational overhead.}
    \label{tab: ablation comput}
\end{table}

\subsection{Additional substantiation to the drop in performance}
Prior knowledge is encoded in the offline learned value function and the learned policy function. Mismanagement of prior knowledge is mainly influenced by two factors: a sudden change in the replay data distribution and a sudden change in the loss function.

we conducted additional experiments about the performance drop. We first gradually increase the proportion of the offline data in the training dataset from 15\% to 50\% to validate the effect of the sudden change in the replay data distribution. A smaller proportion indicates a more sudden change in the replay data distribution. We report the difference between the initial and minimum performance of the performance drop. Larger values represent a more severe performance drop. The experimental results in Table~\ref{tab: replay data} indicate that a smaller proportion of the offline data leads to a more severe performance drop. This suggests that sudden changes in replay data distribution exacerbate performance drop.

In addition, we use various update intervals of the value function network to show the effect of the sudden change of the value function. Specifically, we vary the update interval of the value function network from 1 to 16. A larger update interval indicates a slower change in the value function. The experimental results in Table~\ref{tab: value function} demonstrate that the slow update of the value function alleviates the performance drop. This suggests that sudden changes in the value function also exacerbate the performance drop.

\begin{table}[H]
    \centering
    \begin{tabular}{c|ccccc}
        \toprule
        Proportion & 15\% & 30\% &  35\% & 45\% & 50\% \\
        \midrule
        Hopper-medium & 20.8$\pm$3.3 & 14.1$\pm$2.7 & 13.8$\pm$2.2 & 7.7$\pm$1.6 & 4.9$\pm$1.3\\
        Walker2d-medium & 40.1$\pm$2.9 & 32.7$\pm$3.1 & 28.7$\pm$3.8 & 26.9$\pm$3.5 & 21.2$\pm$3.7\\
        Halfcheetah-medium & 0.0$\pm$0.0 & 0.0$\pm$0.0 & 0.0$\pm$0.0 & 0.0$\pm$0.0 & 0.0$\pm$0.0\\
        Hopper-medium-replay & 23.7$\pm$1.6 & 15.5$\pm$2.2 & 14.8$\pm$2.5 & 13.9$\pm$1.3 & 13.2$\pm$1.1\\
        Walker2d-medium-replay & 25.5$\pm$3.7 & 20.3$\pm$2.3 & 19.0$\pm$3.5 & 18.2$\pm$2.7 & 16.7$\pm$3.6\\
        Halfcheetah-medium-replay & 0.0$\pm$0.0 & 0.0$\pm$0.0 & 0.0$\pm$0.0 & 0.0$\pm$0.0 & 0.0$\pm$0.0\\
        \bottomrule
    \end{tabular}
    \caption{The effect of the replay data distribution on the performance drop. We adopt the normalized score metric, averaging numbers across five seeds.}
    \label{tab: replay data}
\end{table}

\begin{table}[H]
    \centering
    \begin{tabular}{c|ccccc}
        \toprule
        Update Interval & 1 & 2 & 4 & 8 & 16\\
        \midrule
        Hopper-medium & 4.9$\pm$1.3 & 4.5$\pm$1.0 & 3.8$\pm$0.6 & 3.7$\pm$0.7 & 3.5$\pm$0.4\\
        Walker2d-medium & 21.2$\pm$3.7 & 18.2$\pm$2.5 & 17.9$\pm$3.6 & 16.2$\pm$2.7 & 15.4$\pm$3.2\\
        Halfcheetah-medium & 0.0$\pm$0.0 & 0.0$\pm$0.0 & 0.0$\pm$0.0 & 0.0$\pm$0.0 & 0.0$\pm$0.0\\
        Hopper-medium-replay & 13.2$\pm$1.1 & 12.6$\pm$2.4 & 8.7$\pm$1.6 & 0.0$\pm$0.0 & 0.0$\pm$0.0\\
        Walker2d-medium-replay & 16.7$\pm$3.6 & 16.2$\pm$2.5 & 12.8$\pm$3.3 & 10.1$\pm$2.0 & 6.9$\pm$1.6\\
        Halfcheetah-medium-replay & 0.0$\pm$0.0 & 0.0$\pm$0.0 & 0.0$\pm$0.0 & 0.0$\pm$0.0 & 0.0$\pm$0.0\\
        \bottomrule
    \end{tabular}
    \caption{The effect of the value function on the performance drop. We adopt the normalized score metric, averaging numbers across five seeds.}
    \label{tab: value function}
\end{table}

\clearpage
\section{Experimental Details}
\label{appendix: exp details}

\paragraph{Experimental Setting.}
For TD3+BC~(online) and TD3~(finetune), we first load the offline dataset into the online replay buffer and add the online collected data into the buffer.
Then, we uniformly sample data to train from the online buffer.

\paragraph{Hyper-parameters.}
We adopt the TD3+BC and TD3 as the backbone of offline and online algorithms.
Therefore, we build BOORL based on the code of the TD3+BC.
We outline the hyper-parameters used by BOORL in Table~\ref{tab: parameters}.

\begin{table}[h]
    \centering
    \begin{tabular}{ll}
    \toprule
      Hyperparameter   & Value \\
      \midrule
      \hspace{0.3cm} Optimizer & Adam \\
      \hspace{0.3cm} Critic learning rate & 3e-4 \\
      \hspace{0.3cm} Actor learning rate & 3e-4 \\
      \hspace{0.3cm} Mini-batch size & 256 \\
      \hspace{0.3cm} Discount factor & 0.99 \\
      \hspace{0.3cm} Target update rate & 5e-3 \\
      \hspace{0.3cm} Policy noise & 0.2 \\
      \hspace{0.3cm} Policy noise clipping & (-0.5, 0.5) \\
      \hspace{0.3cm} TD3+BC parameter $\alpha$ & 2.5\\
      \hspace{0.3cm} IQL parameter & 0.9\\
      \midrule
      Architecture & Value \\
      \midrule
      \hspace{0.3cm} Critic hidden dim & 256 \\
      \hspace{0.3cm} Critic hidden layers & 2 \\
      \hspace{0.3cm} Critic activation function & ReLU \\
      \hspace{0.3cm} Actor hidden dim & 256 \\
      \hspace{0.3cm} Actor hidden layers & 2 \\
      \hspace{0.3cm} Actor activation function & ReLU \\
      \midrule
      BOORL Parameters & Value \\
      \midrule
      \hspace{0.3cm} Mask ratio $p$ & 0.9 \\
      \hspace{0.3cm} Ensemble Number & 5 \\
      \hspace{0.3cm} UTD ratio $G$ & 5\\
      \bottomrule
    \end{tabular}
    \caption{Hyper-parameters sheet of BOORL.}
    \label{tab: parameters}
\end{table}

\paragraph{Baselines Implementation.}
We adopt the author-provided implementations from GitHub for TD3~\footnote{\url{https://github.com/sfujim/TD3}}, TD3+BC~\footnote{\url{ https://github.com/sfujim/TD3_BC}}, CQL~\footnote{\url{https://github.com/aviralkumar2907/CQL}}, IQL~\footnote{\url{https://github.com/ikostrikov/implicit_q_learning}}, Off2On~\footnote{\url{https://github.com/shlee94/Off2OnRL}}, ODT~\footnote{\url{https://github.com/facebookresearch/online-dt}}, PEX~\footnote{\url{https://github.com/Haichao-Zhang/PEX}} and Cal-QL~\footnote{\url{https://github.com/nakamotoo/Cal-QL}}.
We use the official implementation in the author-provided code for TD3+BC~(online) and IQL~(online).
All experiments are conducted on the same experimental setup, a single GeForce RTX 3090 GPU and an Intel Core i7-6700k CPU at 4.00GHz.

\end{document}